%% file: main.tex
\documentclass{article}
\usepackage{arxiv,times}
\input{math_commands.tex}

\usepackage[pdfencoding=auto,psdextra]{hyperref}
\usepackage{url}
\usepackage{booktabs}
\usepackage{longtable}
\usepackage{amsmath,amssymb,amsthm}
\newtheorem{proposition}{Proposition}
\newtheorem{lemma}{Lemma}
\newtheorem{remark}{Remark}
\newtheorem{example}{Example}
\usepackage{graphicx}
\usepackage[table]{xcolor}
\definecolor{oursbg}{gray}{0.86}
\usepackage{enumitem}
\usepackage{float}
\usepackage{caption}


\newcommand{\ours}{\textsc{CurveTQ}}

\title{CurveTQ:\\ Rotation-Free Trellis Quantization of LLM Weights via Curvature-Weighted Search}

\author{Guanhua Ding$^{*}$ \quad Zi Wang$^{*}$ \quad Ruichao Li \quad Jack Liu\\
Black Sesame Technologies\\
\texttt{\{gding2793, zwang84.utk, ruichaoli809\}@gmail.com}, \texttt{jack.liu@bst.ai}
}
\begin{document}
\maketitle
{\renewcommand{\thefootnote}{\fnsymbol{footnote}}\footnotetext[1]{Equal contribution.}}

\begin{abstract}
The best two-bit weight quantizers for large language models, such as QTIP and Proteus, share a recipe. Each weight matrix is first rotated by a random orthogonal transform so that its entries look Gaussian, and this rotation has to be undone at every decoding step. The rotated weights are then encoded with a trellis or lattice code under a Euclidean search, and the layer Hessian is used only in the error-feedback step between coding blocks. We show that this recipe leaves part of the Hessian unused. Error feedback turns the loss into a weighted sum of per-coordinate rounding errors, and the weights, the diagonal of the Hessian's LDL factorization, are computed by existing quantizers but never read. We put these weights into the search's distance, the Viterbi branch metric, so that the search follows the curvature within each coding block. This also explains what the rotation is for. The rotation removes this within-block variation, so after it the weight has little left to exploit. Weighting in the native basis and rotating are therefore substitutes. We confirm this on three models: the weighted native search matches the baselines' rotation, a full-dimension randomized Hadamard, to within about one point of downstream accuracy, and weighting after the rotation gains little. Around this search we build \textsc{CurveTQ}, a trellis codec with no rotation. The rotation also flattened the amplitude and the marginal shape of the weights, which we handle directly with a factored scale field and a closed-form quantile table. In either basis, error feedback passes each block's residual on to the next, so we store a start state per coding block that lets the trellis adapt to it. At two bits \textsc{CurveTQ} is $1$--$3$ points higher in mean downstream accuracy than QTIP and Proteus on three 4--8B Instruct models, even after both are given our start state, which by itself lifts either baseline by $1$--$3$ points. It also leads on a 35B mixture of experts, to our knowledge the first trellis-coded result on such a model. With no rotation to undo, our decoder is the fastest of the three at all tested batch sizes and bit widths.
\end{abstract}

\section{Introduction}

When a large language model (LLM) generates text one token at a time, the cost of each step is dominated by reading the model's weights from memory. Storing fewer bits per weight therefore directly translates into faster decoding. At four bits per weight, post-training quantization (PTQ) is by now routine \citep{frantar2023gptq,lin2024awq}. The open frontier is two to three bits, where quantized models still lose a large amount of accuracy, and where the choice of the \emph{code}, how groups of weights are jointly mapped to a small set of stored indices, matters more than the choice of the scalar grid \citep{tseng2024qtip,egiazarian2024aqlm}.

The strongest two-bit trellis and lattice codecs today follow the same three-step recipe, which we call \emph{transform-and-flatten} (Figure~\ref{fig:overview}a). First, each weight matrix is rotated by a random orthogonal transform, so that its entries look like independent Gaussian samples \citep{chee2023quip,tseng2024quipsharp}. Second, the rotated weights are encoded with a high-dimensional code, a trellis or a lattice, using a Euclidean search and a Gaussian-shaped reconstruction. Third, the layer Hessian, which says how sensitive the layer output is to an error in each input coordinate, enters only through an error-feedback step between coding blocks (LDLQ). QTIP \citep{tseng2024qtip} is the canonical instance; Proteus \citep{yang2026proteus} replaces its lookup table by a computed code but keeps the rotation and the Euclidean search. The recipe has a cost at every decoding step: the rotation has to be undone at inference time, either on the activations or on the weight-decoding path.

\paragraph{What error feedback leaves unused.}
It is tempting to think that the feedback step has taken care of the Hessian, so that the search inside a block can safely be Euclidean. In fact, feedback only rearranges the loss of a row into a \emph{weighted} sum of its per-coordinate rounding errors (Eq.~\ref{eq:ldl}). The weights are the diagonal $D$ of the Hessian's LDL factorization, the same quantity GPTQ divides its residuals by \citep{chee2023quip,frantar2023gptq}. In the native basis these weights differ by orders of magnitude between input channels, because a few channels carry outlier activations. Yet, to our knowledge, every trellis and lattice quantizer ignores them and minimizes the unweighted error inside the block; QTIP computes their block form as part of its feedback step and never reads it. Our first step is simply to use them. Placed in the Viterbi branch metric, they change neither the code nor the decoder, and, as we show next, they make the rotation unnecessary.

\paragraph{Why the rotation then becomes unnecessary.}
Once the weight is in the search, the random rotation turns out to act on the same quantity. The rotation is meant to make all coordinates look alike, and for the raw Hessian diagonal it does. For the LDL diagonal it does not: Figure~\ref{fig:gamma}(a) shows that it turns the jagged native weights into a smooth ramp along the coding order, nearly flat within one 64-column block. A weight that is constant within a block does not change which path the block's search selects, so after rotation the weighted and the Euclidean search coincide: \emph{the rotation removes the structure the weight exploits}. Proposition~\ref{prop:main} relates the weighted search's gain over the Euclidean one, in the high-resolution limit, to the within-block variation of these weights, which is large in the native basis and near zero after rotation. Weighting and rotating are therefore substitutes, which Section~\ref{sec:rot} confirms on three models. A rotation also flattens the local amplitude and the marginal shape of the weights; we handle both directly, with a factored group scale and a closed-form quantile table (Section~\ref{sec:codec}), so the decoder needs no transform (Figure~\ref{fig:overview}b).

\paragraph{What else the block boundary changes.}
Taking the feedback block seriously has a second consequence, for the trellis termination. Under feedback, the first columns of each block carry the residual pushed forward from the previous block, so the boundary is where the search needs the most freedom. QTIP's tail-biting trellis ties each tile's first coordinates to its last, constraining exactly this boundary, and stores no start state; we instead let each block start in any trellis state and store the choice (Section~\ref{sec:search}). The $14$ stored bits per block are $0.22$ bpw, our largest single component, and not specific to our codec: given to QTIP and Proteus, the same state is worth $1$--$3$ downstream points at two bits. Removing the state costs $2$--$4$ perplexity and $5$ downstream points, ten times the cost of a group-scale field of the same size. At decode time it is one load per block, $1$--$3\%$ of throughput.

\paragraph{Contributions.}
(1)~We show that the curvature weight and the random rotation are substitutes: the LDL diagonal in the branch metric exploits what a rotation removes (Sections~\ref{sec:search}--\ref{sec:sub}). On three models the weighted native search matches, to within about a point of mean accuracy, the full-dimension rotation QTIP deploys at equal rate (Section~\ref{sec:rot}). The loss identity \citep{chee2023quip} and weighted trellis search \citep{fischer1992ectcq} are both known; what is new is that together they make the rotation unnecessary. (2)~We replace QTIP's tail-biting termination by a block-local trellis with a stored start state per block, our largest single rate component and the one that creates the two-bit downstream margin (Sections~\ref{sec:search}, \ref{sec:ablate}). It is not specific to our codec: transferred unchanged to QTIP and Proteus, it lifts each by $1$--$3$ points (Section~\ref{sec:main}). (3)~Using (1), we build \textsc{CurveTQ}, a rotation-free trellis codec; in the native basis, its per-module quantile table and factored scale field handle the marginal shape and the local amplitude that a rotation would have flattened (Section~\ref{sec:codec}). Against QTIP and Proteus given our start state it is $1$--$3$ points higher downstream at two bits on three Instruct models (Section~\ref{sec:main}). It also leads on a 35B mixture of experts, to our knowledge the first trellis-coded result on one; at 32B the start state alone closes most of the baselines' gap. (4)~With no rotation to undo, our decoder is faster than both baselines' kernels at batch 1 to 128 (Section~\ref{sec:throughput}).

\section{Related work}

\paragraph{Incoherence processing.}
QuIP \citep{chee2023quip} introduced random orthogonal preprocessing that makes weights and Hessians incoherent, together with LDLQ, the error-feedback rounding whose loss identity we use (Eq.~\ref{eq:ldl}); QuIP\# \citep{tseng2024quipsharp} made the transform a randomized Hadamard and added the $E_8$ lattice codebook, block LDLQ and a fine-tuning stage. Later work moves rotations into the activation path or makes them data-aware \citep{ashkboos2024quarot,liu2025spinquant,sun2025flatquant,chen2026wush}. Our view is that a rotation flattens three separate quantities at once, the local amplitude, the marginal shape and the block-local LDL diagonal, and that each can be flattened explicitly in the native basis.

\paragraph{Trellis, lattice and vector codes.}
Two lines of work are closest to ours. The first is trellis-coded quantization. Its weighted and feedback variants \citep{marcellin1990tcq,fischer1992ectcq,atal1979,eyuboglu1988rsse} are classical; QTIP \citep{tseng2024qtip} brought TCQ to LLM weights with the bitshift trellis and two decoders, a lookup code and a computed one. Proteus \citep{yang2026proteus} makes the computed code central and argues from decode bandwidth that a lookup-free trellis is the right target. All of these keep the search Euclidean and confine curvature to the feedback step; Proteus does use a Hessian-weighted objective, but to \emph{select} its code per layer, not inside the search. The second line is lattice and additive vector quantization \citep{egiazarian2024aqlm,vanbaalen2024gptvq,liu2024vptq,glvq2025,vanderouderaa2026llvq}, which trades decode arithmetic for shaping gain (we compare against trellis codecs only, which decode faster: LLVQ reports its single-shell kernel slightly slower than QTIP's, and an independent multi-shell implementation measures QTIP's kernel $2.27\times$ faster \citep{malandrino2026unfolding}). This line does optimize an input-weighted error: AQLM's beam search, and the Hessian-weighted $k$-means of GPTVQ, VPTQ and SqueezeLLM \citep{kim2024squeezellm}. The cost is several dependent codebook reads per weight at decode, where a trellis needs one table read, or none. Our codec is a trellis code that optimizes the weighted objective of vector quantization, at a trellis decoder's cost.

\paragraph{Solvers and side information.}
GPTQ \citep{frantar2023gptq} rounds columns one at a time with Hessian feedback, and its error-normalization constant is the same diagonal $D$ we place in the search. On side information, group scales are universal \citep{lin2024awq,shao2024omniquant}, and storing them on a second, coarser grid goes back to QLoRA and SpQR \citep{dettmers2023qlora,dettmers2024spqr}; our factored scale field is a structured version of this idea. For mixtures of experts, QMoE \citep{frantar2023qmoe} quantizes with one Hessian per expert, which we adopt; no trellis or lattice codec has been reported on such a model. Finally, post-quantization fine-tuning is worth more at two bits than the spread between any two codes \citep{malinovskii2024pvtuning}, and QuIP\#, QTIP, AQLM and Proteus all publish their headline numbers after such a stage; to isolate the code, we run every method without it (Section~\ref{sec:protocol}). Appendix~\ref{app:related} discusses this literature in more detail.

\section{Method}

\begin{figure}[t]
\centering
\includegraphics[width=0.95\linewidth]{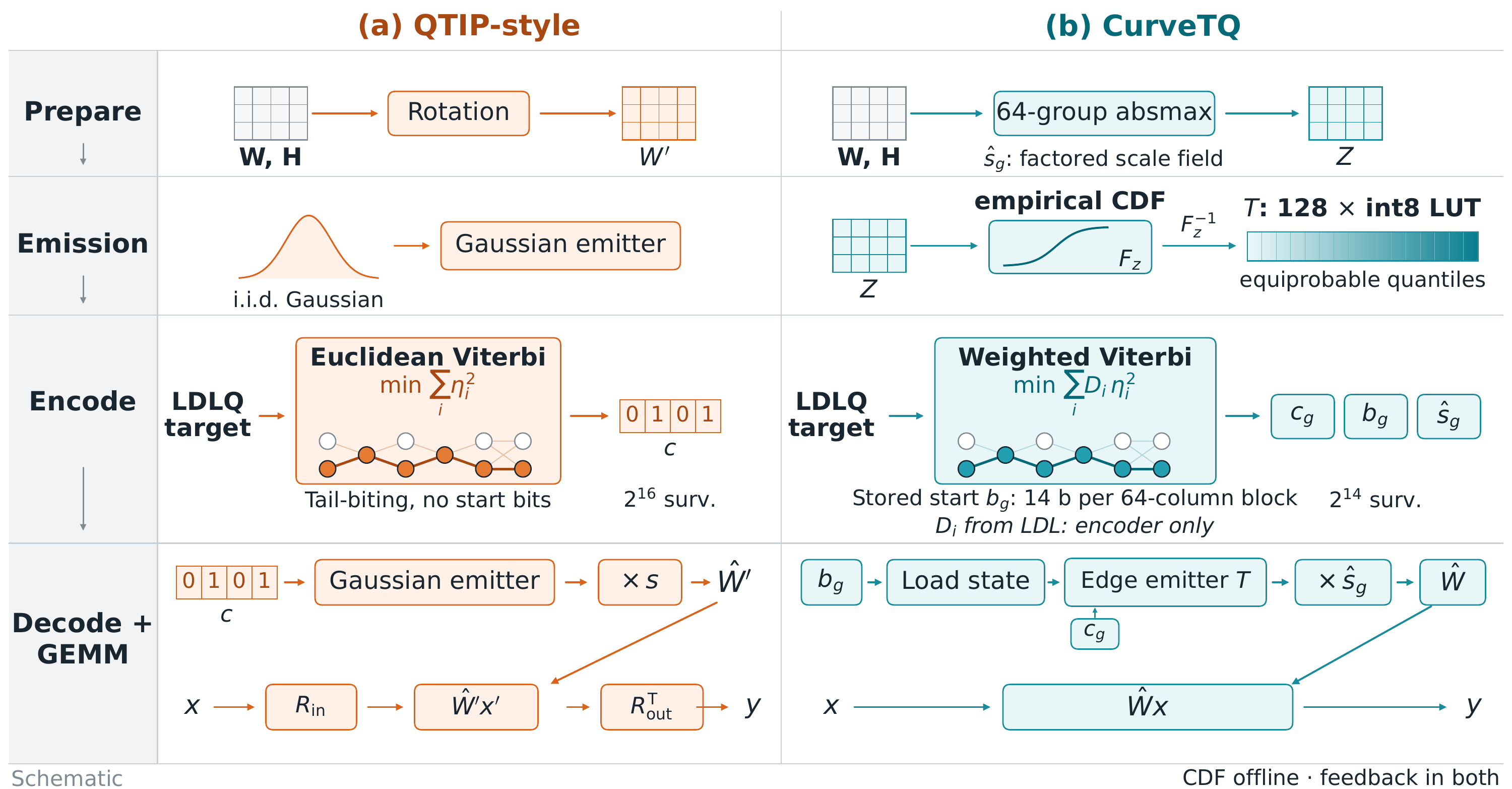}
\caption{Transform-and-flatten (a) versus our recipe (b); schematic, not measured data. The Encode row is the mechanism: the LDL diagonal $D$, which (a) flattens within each block by rotation, becomes in (b) the weight of the Viterbi branch metric, so the search minimizes $\sum_i D_i\eta_i^2$ rather than $\sum_i\eta_i^2$. The rotation's other two effects are obtained by the factored group scale $\hat s_g$ and the quantile table $T$, and the decoder needs no transform.}
\label{fig:overview}
\end{figure}

\label{sec:prelim}\label{sec:ldl}
\paragraph{Preliminaries: error feedback and the bitshift trellis.}
Let $W\in\mathbb R^{R\times C}$ be a linear layer with input Hessian $H=\mathbb E[xx^\top]$, damped as $H+\delta I$; all identities below hold for the damped loss. Weight-only PTQ chooses $\hat W$ to minimize $\sum_r e_r^\top He_r$, $e_r=w_r-\hat w_r$. Fix a coding order over the $C$ input coordinates and write $H=(I+\bar U)D(I+\bar U)^\top$. Column-wise feedback quantizes coordinate $j$ against $w_j+\sum_{i<j}\bar U_{ij}e_i$ and defines $\eta_j$ as that target minus $\hat w_j$; then $\eta=(I+\bar U)^\top e$ and
\begin{equation}
e^\top He=\sum_{i=1}^{C} D_i\,\eta_i^2
\label{eq:ldl}
\end{equation}
holds for any quantizer and any code (Appendix~\ref{app:proofs}). Under \emph{block} feedback with blocks of $G=64$ columns (QuIP\#'s block LDLQ at our block size), the exact form for the block factor is $e^\top He=\sum_b \eta_b^\top D_b\,\eta_b$ with $D_b$ the Schur complement of the trailing submatrix of $H$ onto block $b$ (our default column factor approximates it; Appendix~\ref{app:ablate}): feedback handles the triangular factor, $\mathrm{diag}(D_b)$ is additive and is what we place in the search, and the off-diagonal part is omitted.

\paragraph{The code.}
We use QTIP's bitshift trellis: an $L$-bit state $s_t$, a $k$-bit code $a_t$ per weight ($k\in\{2,3,4\}$), the transition $s_{t+1}=((s_t\ll k)\,|\,a_t)\bmod 2^L$, and a reconstruction from the emission window (Section~\ref{sec:codec}); decoding walks the state machine, encoding is a Viterbi search over $2^L$ states.

\subsection{Curvature-weighted trellis search}
\label{sec:search}

Rows are coded in blocks of $G=64$ columns, left to right; the residual of a coded block is propagated onto the not-yet-coded columns through the triangular factor of $H$. Each block is coded by a Viterbi search that minimizes
\begin{equation}
P^\star=\arg\min_P\ \sum_{i\in b}\omega_i\big[\tilde z_i-T(k_i(P))\big]^2,\qquad
\omega_i=\frac{1}{d_i^2},
\label{eq:metric}
\end{equation}
where $\tilde z$ is the feedback-corrected, scale-normalized target of the block, $k_i(P)$ the code index of coordinate $i$ on path $P$, $T$ the reconstruction table, and $d_i$ the $i$-th diagonal entry of the Cholesky factor of $H^{-1}$ (GPTQ's error-normalization constant). Because Eq.~\ref{eq:metric} is additive over coordinates, Viterbi minimizes it exactly at unchanged complexity, and the packed format is that of the Euclidean search ($\omega\equiv1$): the weight changes which path is chosen, not how it is stored.

We make two further approximations. First, Eq.~\ref{eq:metric} keeps only the diagonal of $D_b$ and omits the in-block coupling. Second, $1/d_i^2$ conditions on \emph{all} columns after $i$, including in-block columns that block feedback never corrects. Appendix~\ref{app:ramp} compares $1/d_i^2$, our default, with the feedback-free $H_{ii}$.

\paragraph{Block-local trellis with an explicit start state.}
A trellis code must decide how one coding block's path relates to the next. We let each block's Viterbi start in any of the $2^L$ states and store the chosen state, the \emph{boundary-state field} $b_g$ (the discrete counterpart of the amplitude field $s_g$ of Section~\ref{sec:sub}), as $L$ bits per (row, block). This \emph{boundary adaptation rate}, $L/G=14/64=0.219$ bpw, is counted in every reported rate (Appendix~\ref{app:rate}). The decoder loads the state and walks the block; nothing else changes. QTIP instead codes 256-weight tiles as \emph{tail-biting} sequences: the first state is a function of the tile's own last bits, so no start state is stored and the rate is exactly $k$ \citep{tseng2024qtip}. That construction is right for a long tile over a rotated, near-i.i.d.\ source. Under error feedback, however, the block boundary is where the inherited residual lands, and tying the first $\lceil L/k\rceil$ coordinates of a 64-column block to its own last ones is the wrong constraint. Section~\ref{sec:ablate} compares this choice with tail-biting and chained terminations; Section~\ref{sec:main} also equips the baselines with it.

\subsection{Why the rotation becomes unnecessary}
\label{sec:sub}

Factor each weight as $w_i=s_{g(i)}z_i$, with $s_g$ the scale of its 64-group and $z_i\in[-1,1]$. The block objective Eq.~\ref{eq:metric} then has three factors at three granularities: the amplitude field $s_g$, the marginal law of $z$, and the curvature weight $D_i$. A random orthogonal transform flattens all three: block energies concentrate, so one scale per row suffices; the marginal tends to a Gaussian, so a universal emitter fits; and the raw diagonal concentrates, $H'_{ii}\to\mathrm{tr}(H)/C$. The LDL diagonal, however, does not flatten: it becomes a smooth ramp along the coding order (Appendix~\ref{app:ramp}). Define the anisotropy of $D$ and its block-local version, in bits per weight,
\begin{equation}
\Gamma(H)=\tfrac12\log_2\frac{\mathrm{AM}(D)}{\mathrm{GM}(D)},\qquad
\Gamma_b(H)=\tfrac12\log_2\frac{\sum_b\mathrm{tr}(D_b)}{\sum_b G\,\mathrm{GM}(\mathrm{diag}\,D_b)}\ \ge 0,
\label{eq:gamma}
\end{equation}
with equality in $\Gamma_b$ iff $\mathrm{diag}\,D_b$ is constant within every block.

\begin{proposition}[weights and floors]
\label{prop:main}
With the feedback rule of Eq.~\ref{eq:ldl}: (i) $e^\top He=\sum_i D_i\eta_i^2$ for any quantizer. (ii) If the residuals have equal second moments $\mathbb E[\eta_i^2]=\sigma^2$---the high-resolution behaviour of a Euclidean search with one alphabet for all coordinates---then $\mathbb E[e^\top He]=\sigma^2\,\mathrm{tr}(D)\ \ge\ C\sigma^2\,\mathrm{GM}(D)=C\sigma^2\det(H)^{1/C}$, with equality iff all $D_i$ are equal. (iii) The lower bound in (ii) is attained, under the high-resolution rate model, by a code whose per-coordinate step is proportional to $D_i^{-1/2}$ \citep{huang1963block,zador1982}.
\end{proposition}

Proofs are in Appendix~\ref{app:proofs}. In words: within the high-resolution relaxation of (ii)--(iii), weighting can gain at most $\Gamma(H)$ over a Euclidean search in a given basis, and a block-local code at most $\Gamma_b(H)$. Since $\det(RHR^\top)=\det H$, the floor is the same in every basis; within this relaxation only the Euclidean loss $\sigma^2\mathrm{tr}(D)$ is not, and that is what a rotation changes.

\paragraph{The substitution hypothesis.}
A Haar or sign-randomized Hadamard rotation turns $D'_i$ into a smooth ramp from the harmonic to the arithmetic mean of the damped spectrum (Appendix~\ref{app:ramp}; checked numerically). Where the ramp is flat across a block, the search weight $D'_i$ is nearly constant within it, so the two searches nearly coincide: \emph{after rotation, weighting is redundant}. The converse, that native weighting \emph{matches} rotation, is not implied by Proposition~\ref{prop:main} (Appendix~\ref{app:proofs}, Proposition~\ref{prop:noorder}); it is a hypothesis about trained LLM Hessians, whose outlier input channels put large block-local anisotropy into the native $D$, and it predicts that the weighted metric's gain tracks $\Gamma_b$ natively and vanishes after rotation. Section~\ref{sec:rot} tests both on three models. We therefore keep the basis and the weight, and flatten the amplitude and the marginal shape explicitly (Figure~\ref{fig:overview}b).

\subsection{\textsc{CurveTQ}: the rotation-free codec}
\label{sec:codec}

\paragraph{Group scale and quantile table.}
The scale of a 64-group is its absmax $s_g$, stored on the factored grid described below as $\hat s_g\ge s_g$; the coder operates on $z_i=w_i/\hat s_g$, so encoder and decoder share the normalization. Each coded linear layer has its own alphabet of $N=128$ int8 values at the equiprobable quantiles of a sample of its normalized weights, indexed by a 7-bit hash of the emission window, with no interpolation at decode. The sample is a deterministic stride over the layer's 64-groups, so the table is closed-form and the build is a deterministic function of the calibration set. A quantile grid is a heuristic rather than the rate--distortion-optimal reproduction law, but it matches an EM fit of the same 128 entries (Section~\ref{sec:ablate}).

\paragraph{Emission window.}
QTIP emits from a 16-bit state and searches $2^{16}$ survivors; we instead emit from the \emph{edge} $(s_t,a_t)$ of a 14-bit state (the same code family, and at $k=2$ the same 16-bit window) and search $2^{14}$. This is the Mealy form (outputs on transitions rather than on states) of QTIP's node-emitting trellis; an exact Viterbi run confirms identical costs and reconstructions on every tested row, with a quarter of the survivors.

\paragraph{Table-free emission.}
The 128-entry table can be removed altogether: in the table-free variant the reconstruction is \emph{computed} from the emission window, like QTIP's 1MAD code, with one global scale in place of the table (details in Appendix~\ref{app:impl}). Nothing else changes, so the difference between the two variants isolates the table (Appendix~\ref{app:full}).

\paragraph{Factored scale field.}
A fp16 scale per 64 weights costs $0.25$ bpw. Second-level quantization stores first-level scales on a low-bit grid with a fp16 constant per block \citep{dettmers2023qlora,dettmers2024spqr}, treating them as a flat sequence. We treat them instead as a two-dimensional field over rows and column groups, whose row and column-group effects carry much of the variance (about half over the linear layers of Llama-2-7B, by a two-way decomposition $\log_2 s_{r,g}=\mu+\alpha_r+\beta_g+\epsilon_{r,g}$). We store $\hat s_{r,g}=A_rB_g\hat R_{r,g}$ with fp16 row and column-group factors (geometric means along each axis) and only the residual on a 4-bit linear grid with a fp16 absmax per 256, both rounded up against the stored fp16 factors, so that $\hat s\ge s$ holds for every group. The cost is $0.068$ bpw on Qwen3-4B and $0.067$ on Llama-3 (per-term arithmetic in Appendix~\ref{app:rate}), against $0.25$ for fp16 scales and $0.125$ for int8 second-level quantization. The fp16-scale variant of every build is reported in Appendix~\ref{app:full}; every rate we report counts all of this side information, the emission table and the start state (Appendix~\ref{app:rate}).

\paragraph{Mixture-of-experts models.}
The codec applies to a mixture of experts unchanged; only the Hessians differ. Following QMoE \citep{frantar2023qmoe}, expert $e$ receives its own $H^{(e)}=\sum_{t\to e}x_tx_t^\top$ over the tokens the router sends it. Under top-8-of-256 routing that is a few thousand tokens for a 2048-dimensional input, so $H^{(e)}$ is often rank-deficient and its LDL diagonal damping-dominated, which leaves the weighted metric close to Euclidean.

\section{Experiments}
\label{sec:protocol}

\paragraph{Setup.}
We quantize six public checkpoints: Llama-2-7B, Llama-3-8B- and Llama-3.1-8B-Instruct \citep{llama3_2024}, Qwen3-VL-4B- and -32B-Instruct \citep{qwen3vl2025} (Qwen3-4B, Qwen3-32B below) and the Qwen3.6-35B-A3B mixture of experts, all with the same 128 calibration segments of 2048 tokens from C4 \citep{raffel2020c4}; our build is deterministic given the calibration set. QTIP \citep{tseng2024qtip} and Proteus \citep{yang2026proteus} run from one trellis/LDLQ implementation at $L=16$, each at its published rate and in two strengthened forms, with an added 64-group fp16 scale ($x.25$ bpw) and with our per-block stored start state (Section~\ref{sec:search}; $2.19$--$2.22$ bpw at two bits, $3.16$--$3.20$ at three), so that the second form carries our start-state side information. Appendix~\ref{app:full} also gives them both at once, at $x.31$--$x.34$ bpw, above ours. HIGGS \citep{malinovskii2025higgs} and 64-group GPTQ \citep{frantar2023gptq} run at $x.25$ bpw. Table~\ref{tab:models} gives the per-model protocol. We report WikiText-2 perplexity \citep{merity2017wikitext} at context 4096 and mean accuracy over a fixed suite (4 to 34 tasks per model; Appendix~\ref{app:tasks}) run through one harness \citep{gao2024lmeval,zhang2024lmmseval}; main comparisons are paired differences to each baseline over tasks, with bootstrap 95\% intervals (Appendix~\ref{app:ci}).

\subsection{Main results}
\label{sec:main}

\begin{table}[t]
\centering
\footnotesize
\setlength{\tabcolsep}{3.0pt}
\caption{Perplexity and mean downstream accuracy at two and three bits (four bits in Appendix~\ref{app:full}); protocol in Section~\ref{sec:protocol}. $^{+}$: published rate plus $0.003$ bpw of sign bits; ours include the per-block start state (Appendix~\ref{app:rate}). $_{+\mathrm{start}}$: the baseline with our stored start state, charged net (Appendix~\ref{app:rate}). $^{\ast}$Llama-2-7B: ARC-c/ARC-e/PIQA/WinoGrande, 0-shot, the suite of the QTIP and Proteus papers. No method is fine-tuned; Proteus is its discrete solver, without learned scales or distillation. Bold: best in tier. ``--'': not run; ``$\times$'': no valid outputs.}
\label{tab:main}
\input{tab_main}
\end{table}

Table~\ref{tab:main} reports the trellis methods at two and three bits and, at two bits, HIGGS and GPTQ, which fall far behind the trellis codecs (other rows and baselines in Appendix~\ref{app:full}; confidence intervals for every difference below in Appendix~\ref{app:ci}). GPTQ is the only other rotation-free row: its scalar grid collapses at two bits, and its gap to our build is what the code is worth without a rotation.

\paragraph{Two bits.}
At $2.29$ bpw \textsc{CurveTQ} has lower perplexity than QTIP and Proteus at $2.00$ and $2.25$ bpw on all four Instruct models. Downstream it leads them by $+2.5$ to $+5.2$ points on the three 4- and 8B Instruct models, all twelve contrasts excluding zero, and by $+0.5$ to $+1.3$ on Qwen3-32B (Table~\ref{tab:ci}). The $_{+\mathrm{start}}$ rows add our per-block start state to the baselines' own trellis. It is worth $+1.0$ to $+3.3$ points to them, more than the group scale, and it closes the perplexity gap on the Llamas and Qwen3-32B. The downstream margin survives: $+0.9$ to $+3.1$ points against QTIP$_{+\mathrm{start}}$ and Proteus$_{+\mathrm{start}}$ on the three models, five of six contrasts excluding zero. It is concentrated on generative tasks (GSM8K, IFEval), while multiple-choice tasks move by about a point. A 128-group scale on top of the start state changes little: at a rate above ours ($2.31$--$2.34$ bpw), the margin is $+0.4$ to $+2.7$ points, five of six contrasts excluding zero, and the scale adds little over the start state alone (Appendix~\ref{app:full}). On Llama-2-7B and Qwen3-32B, where generative tasks lose less, the start state is most of the difference: without it the baselines trail by $1$--$2$ points, with it they come within $0.2$ points on both. Our Llama-2-7B QTIP reproduction is within $0.05$ perplexity of its paper.

\paragraph{Three and four bits.}
At three bits \textsc{CurveTQ} ($3.29$ bpw) leads the $3.00$-bpw rows downstream on every model. Once the baselines carry our start state ($3.16$--$3.20$ bpw) they draw level: \textsc{CurveTQ} is ahead of or level with the baseline in twelve of the sixteen (model, metric, baseline) comparisons on the four models, and all eight paired contrasts lie within $\pm0.7$ points (Table~\ref{tab:ci}). The per-group-scale rows are in Appendix~\ref{app:full}. At four bits all trellis codecs are within $0.35$ perplexity and one point (Appendix~\ref{app:full}).

\begin{table}[t]
\centering
\small
\setlength{\tabcolsep}{8pt}
\caption{Qwen3.6-35B-A3B (MoE) at two bits: PPL and mean accuracy over 30 tasks (per task: Appendix~\ref{app:tasks}).}
\label{tab:moe}
\begin{tabular}{@{}l c r r@{}}
\toprule
Method & bpw & ppl $\downarrow$ & Mean $\uparrow$ \\
\midrule
fp16 & 16 & 6.12 & 71.3 \\
\midrule
\rowcolor{oursbg}\textbf{CurveTQ} & 2.29 & \textbf{6.76} & \textbf{68.2} \\
Proteus & 2.00$^{+}$ & 6.99 & 66.0 \\
QTIP & 2.00$^{+}$ & 7.10 & 65.6 \\
QTIP$_{+\mathrm{g64}}$ & 2.25$^{+}$ & 6.99 & 65.3 \\
Proteus$_{+\mathrm{g64}}$ & 2.25$^{+}$ & 7.08 & 65.1 \\
QTIP$_{+\mathrm{start}}$ & 2.19$^{+}$ & 6.81 & 67.6 \\
Proteus$_{+\mathrm{start}}$ & 2.22$^{+}$ & 6.78 & 66.6 \\
\bottomrule
\end{tabular}
\end{table}
\paragraph{Mixture of experts.}
Table~\ref{tab:moe} reports Qwen3.6-35B-A3B (40 MoE layers, 256 experts, top-8 routing) quantized with one Hessian per expert by the same code. At $2.29$ bpw, the exact element-weighted rate, our build has lower perplexity and higher mean accuracy than Proteus and QTIP at $2.00$ bpw, by $+2.1$ [$+1.4$, $+3.0$] and $+2.5$ [$+1.9$, $+3.3$] paired points over the 30 tasks; the 64-group scale gains the baselines nothing downstream, our start state $+0.6$ to $+2.0$ points. At equal side information \textsc{CurveTQ} is still ahead of both, by $+0.6$ [$-0.1$, $+1.3$] points over QTIP$_{+\mathrm{start}}$ and $+1.5$ [$+0.5$, $+2.8$] over Proteus$_{+\mathrm{start}}$, winning 18 and 21 of the 30 tasks (Table~\ref{tab:ci}). Expert Hessians are estimated from few tokens, so their LDL diagonals are damping-dominated and the weight is nearly inert here (Section~\ref{sec:codec}).

\subsection{Rotation and curvature}
\label{sec:rot}

\begin{figure}[t]
\centering
\includegraphics[width=\linewidth]{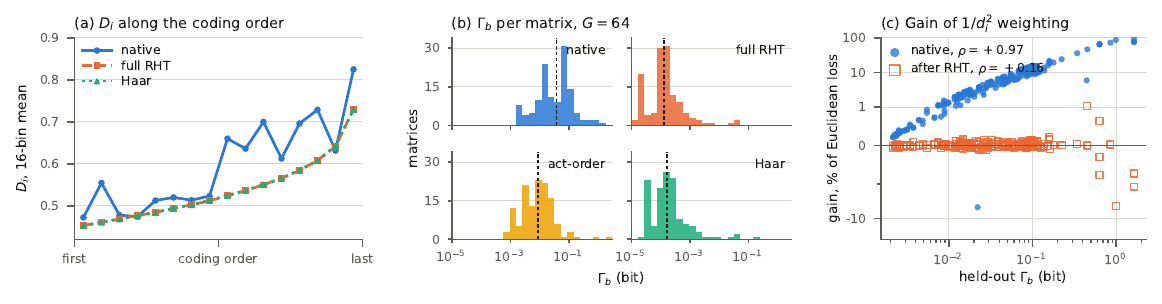}
\caption{Within-block curvature anisotropy $\Gamma_b$ (Eq.~\ref{eq:gamma}, $G=64$), and what a rotation does to it; all panels Llama-2-7B. (a) The LDL diagonal along the coding order for one matrix (layer 15, q/k/v input, 16-bin means): native is jagged (whole-layer $\max/\min$ $2503\times$), the full-dimension randomized Hadamard and the dense Haar rotation both produce the smooth ramp of Appendix~\ref{app:ramp} ($2.7\times$, $3.9\times$). (b) $\Gamma_b$ over 128 public Llama-2-7B Hessians, one histogram per transform; medians (dashed) $0.038$ native, $0.0089$ act-order, $0.00013$ full RHT, $0.00016$ Haar. (c) Realized gain of $1/d_i^2$ weighting on a held-out layer objective, as a fraction of the Euclidean loss ($1-L_W/L_E$), against $\Gamma_b$ computed from the full held-out Hessian, 224 matrices: the gain follows $\Gamma_b$ natively (Spearman $+0.97$, filled) and not after the RHT ($+0.16$, hollow); the other two models in Figure~\ref{fig:gain3}. Symlog vertical axis; details in Appendix~\ref{app:ramp}.}
\label{fig:gamma}
\end{figure}

\begin{table}[t]
\centering
\small
\setlength{\tabcolsep}{10pt}
\caption{Rotation (full RHT) $\times$ metric at $2.29$ bpw: mean accuracy over each model's suite of Table~\ref{tab:main} (34, 4 and 14 tasks), paired difference over $n=4$ slices $\pm$ s.d. P1/S1: weighted $-$ Euclidean, native/rotated; R-e/R-w: rotated $-$ native, Euclidean/weighted; SUB: rotated Euclidean $-$ native weighted. ``--'': the native Euclidean search collapses on Llama-3-8B, so contrasts against it are not reported.}
\label{tab:rot}
\begin{tabular}{@{}l rrr@{}}
\toprule
& Qwen3-4B & Llama-2-7B & Llama-3-8B \\
\midrule
P1 & $+2.44\pm0.57$ & $+4.29\pm1.04$ & -- \\
S1 & $+0.36\pm0.51$ & $+0.08\pm0.53$ & $-0.13\pm0.19$ \\
R-e & $+3.54\pm0.71$ & $+4.68\pm1.40$ & -- \\
R-w & $+1.46\pm0.42$ & $+0.48\pm0.14$ & $-0.06\pm0.67$ \\
SUB & $+1.10\pm0.63$ & $+0.39\pm0.66$ & $+0.07\pm0.84$ \\
\bottomrule
\end{tabular}
\end{table}
Section~\ref{sec:sub} predicts that after a mixing rotation the curvature weight has little left to exploit, and that in the native basis it can substitute for the rotation. We test both with a $2\times2$ design, \{rotation, none\} $\times$ \{Euclidean, $1/d_i^2$\}, in the \textsc{CurveTQ} format of Table~\ref{tab:main} ($k=2$, factored int4 scales, per-module table, $2.29$ bpw) so that only rotation and metric vary. The rotation is the full-dimension randomized Hadamard $H_{2^m}\otimes H_K$ that QuIP\# and QTIP apply to the whole input dimension (dense, but with an $O(C\log C)$ transform), drawn per matrix; table, factored scales and weight are recomputed from $RHR^\top$ in each cell; each cell is built on $n=4$ disjoint calibration slices and scored on the model's full task suite of Table~\ref{tab:main}, reported as paired difference $\pm$ s.d. The models are Qwen3-4B, Llama-3-8B-Instruct and Llama-2-7B; the last is included because public 6144-sequence Hessians exist for it \citep{tseng2024quipsharp}, which make Figure~\ref{fig:gamma}(b) exact.

\paragraph{Results.}
Figure~\ref{fig:gamma}(a,b) shows what the rotation does to $D$: it does not flatten it, but within a 64-column block the ramp it produces is flat. The within-block anisotropy $\Gamma_b$ of Eq.~\ref{eq:gamma} falls on every one of the 128 matrices, by a median of two and a half orders of magnitude ($0.038\to0.00013$ bit). Table~\ref{tab:rot} reports the realized contrasts, and three things hold on the models where they can be measured. First, in the native basis the weight gains (P1): $+2.44\pm0.57$ points on Qwen3-4B and $+4.29\pm1.04$ on Llama-2-7B, positive on all four slices of both. On Llama-3-8B the native Euclidean search collapses across the suite (MMLU at chance, GSM8K under $3\%$), an instance of the two-bit fragility of Llama-3 that \citet{huang2024llama3} document, so we report only the contrasts that do not involve that cell. Second, after the rotation the weight adds little (S1): $+0.36$, $+0.08$ and $-0.13$ points, each inside its slice spread. Third, the substitution: the rotated Euclidean search and the weighted native search are within about a point of each other (SUB), $+1.10\pm0.63$, $+0.39\pm0.66$ and $+0.07\pm0.84$. On the two Llamas it is smaller than its slice s.d. On Qwen3-4B the rotated search is ahead on 25 of the 34 tasks, by $1.1$ points on average, while WikiText-2 perplexity favours the native weighted search by $0.36$. A dense Haar rotation gives the same picture (Appendix~\ref{app:ramp}). Figure~\ref{fig:gamma}(c) adds the per-matrix view on a held-out layer objective: in the native basis the weighted search's gain tracks $\Gamma_b$ matrix by matrix (Spearman $+0.97$ on Llama-2-7B, $+0.67$ and $+0.99$ on the other two), and after the rotation it does not ($+0.16$, $-0.11$, $-0.01$). The weighted native search therefore matches the rotation QTIP deploys to within about a point at equal rate with no transform at all. The weight's value is not a lower loss than the rotation's but the native basis it lets us keep. In that basis a per-module quantile table and a factored scale field, components a rotated, Gaussian-coded design has no use for, are what separate \textsc{CurveTQ} from the $_{+\mathrm{start}}$ baselines of Table~\ref{tab:main}; consistently, the table-free build ties Proteus$_{+\mathrm{start}}$ (Appendix~\ref{app:full}). The rotation lets the baselines search without a weight; Section~\ref{sec:throughput} measures what it costs at decode.

\subsection{Ablations}
\label{sec:ablate}
\label{sec:ablations}

\begin{table}[t]
\centering
\small
\setlength{\tabcolsep}{8pt}
\caption{Component ablation, Qwen3-4B at two bits: PPL and mean accuracy over 34 tasks. Blocks: components added one per row up to the full recipe; one removed; search ablations from the quantile-fit row; block terminations. Paired differences with 95\% CIs: Appendix~\ref{app:ablate}.}
\label{tab:abl}
\input{tab_abl}
\end{table}
\paragraph{Building up the codec.}
Table~\ref{tab:abl} ablates the codec on Qwen3-4B at two bits (all models and rates in Appendix~\ref{app:ablate}). Read with the bpw column, the ladder has two kinds of step. Equal-rate steps change only accuracy: the closed-form quantile fit matches an EM fit of the same 128 entries ($0.28$ perplexity, within twice the seed noise of Appendix~\ref{app:ablate}) and makes the build deterministic, and the edge-indexed emission is the largest single step, $1.27$ perplexity and about one downstream point at equal bits. Rate-buying steps trade accuracy for bits: int8 second-level scales save $0.124$ bpw for $0.41$ perplexity, and moving them to a flat int4 grid a further $0.062$ bpw for $0.58$ perplexity and $0.8$ points. The last step recovers what the flat grid gave up---the factored int4 scale field with a per-module table is $1.26$ perplexity and $0.8$ points better at $0.006$ bpw more---so the ladder is not monotone in accuracy, and its intermediate rate-buying rows are cheaper but worse than their neighbours. End to end the recipe saves $0.18$ bpw, $1.82$ perplexity and $1.4$ [$+0.5$, $+2.4$] downstream points.

\paragraph{Removing one component.}
fp16 scales in place of the factored field cost $0.18$ bpw for $0.35$ perplexity worse and $0.4$ [$-0.3$, $+1.1$] points better: the field saves its bits at no downstream cost. Removing the emission table costs $2.13$ perplexity and $1.2$ [$0.2$, $2.2$] points at two bits and is within noise at four ($0.11$ perplexity, $0.2$ points): the table matters only at the lowest rate.

\paragraph{Ablating the search.}
We remove, one at a time, the two components Section~\ref{sec:search} argues for, starting from the quantile-fit configuration (fp16 scales). Removing error feedback costs $0.34$ perplexity and $-2.5$ [$-4.3$, $-1.4$] downstream points, concentrated on GSM8K ($-27$) and IFEval ($-6$). Replacing the $1/d_i^2$ branch metric by the Euclidean one, feedback kept, costs $-3.3$ [$-4.5$, $-2.3$] points while \emph{lowering} WikiText-2 perplexity by $0.24$. Perplexity and downstream accuracy thus disagree on what the weight is worth; we report both.

\paragraph{Ablating the block termination.}
The last block of Table~\ref{tab:abl} replaces our block-local trellis by the two constructions that store no per-block state, the recipe otherwise identical: a row-long \emph{chained} trellis and a per-block \emph{tail-biting} one (QTIP's construction on a 64-weight block). Both save $\approx0.22$ bpw and cost $2.09$ and $4.10$ perplexity and $-5.5$ [$-7.0$, $-4.2$] and $-5.1$ [$-6.9$, $-3.5$] downstream points, against the $0.4$ points the factored scale field moves for a comparable $0.18$ bpw two rows above. Lengthening the tail-biting sequence to QTIP's own tile length recovers $2.22$ of the $4.10$ perplexity but is still $-4.2$ [$-6.4$, $-2.5$] points behind.

\subsection{Decode throughput}
\label{sec:throughput}

\begin{table}[t]
\centering
\small
\setlength{\tabcolsep}{8pt}
\caption{Batch-1 decode throughput (tokens/s), one RTX~3090. Llama-2-7B QTIP and fp16 are QTIP's published numbers; the rest, and all batch 64/128 numbers, are ours. Batch 64/128 and the $+$rotation / $-$start-state arms are in Appendix~\ref{app:impl}. \textsc{CTQ}: \textsc{CurveTQ}.}
\label{tab:tput}
\input{tab_tput}
\end{table}
We implemented the decoder of Section~\ref{sec:codec} (row Decode $+$ GEMM of Figure~\ref{fig:overview}b) as a fused dequantize-and-multiply CUDA kernel and measured decoding on an RTX~3090 in QTIP's own benchmark setting, at three bit widths and three batch sizes (Table~\ref{tab:tput}; full table and setup in Appendix~\ref{app:impl}, Table~\ref{tab:tput_full}). Our harness reproduces QTIP's published fp16 figure on Llama-2-7B ($52.25$ against $52.5$ tokens/s), so at batch~1 we quote their published QTIP numbers there; QTIP publishes no batch 64/128 figures, so those, and all of Qwen3-4B, we run with its released kernels; Proteus, which has no public kernel, is our implementation throughout. Three things follow. First, our decoder is the fastest of the three in every cell: at batch~1 it is $1.05$--$1.09\times$ QTIP and $1.12$--$1.36\times$ Proteus, and at batch 64 and 128 it stays ahead by $1$--$13\%$. Second, the rotation is expensive at batch~1. Adding QTIP's activation-side Hadamard to our own decoder, so that only the transform changes, costs $5$--$9$ tokens/s on Llama-2-7B and $13$--$17$ on Qwen3-4B, $4$--$11\%$ of throughput. That is most of, or more than, our whole margin over QTIP: with the rotation added our decoder is about as fast as QTIP's ($0.93$--$1.02\times$), so at batch~1 most of the lead is the transform a rotation-free decoder does not pay for. Third, the stored start state is cheap to decode: removing it gains only $1$--$5$ tokens/s, $1$--$3\%$, so the $0.22$ bpw it costs returns $1$--$3$ downstream points (Section~\ref{sec:main}) at almost no cost in speed.

\section{Conclusion}

We showed that the curvature weight and the random rotation are substitutes: the LDL diagonal that error feedback already computes, placed in the Viterbi branch metric, exploits the within-block variation a rotation removes, and on three models the weighted native search matches the baselines' rotation to within about a point. On this basis we built \textsc{CurveTQ}, a trellis codec with no rotation to undo at inference, whose factored group scale and quantile table handle the amplitude and the marginal shape directly. Its block-local trellis with a stored start state also lifts QTIP and Proteus by $1$--$3$ points; against them so strengthened, \textsc{CurveTQ} is still $1$--$3$ points higher downstream at two bits, leads on a 35B mixture of experts, and decodes faster at every batch size we tested.

\bibliography{refs}
\bibliographystyle{arxiv}

\clearpage
\appendix
\renewcommand{\thetable}{\thesection.\arabic{table}}
\setcounter{table}{0}
\makeatletter\@addtoreset{table}{section}\makeatother
\section{Proofs and examples for Proposition~\ref{prop:main}}
\label{app:proofs}

Throughout, $H\in\mathbb R^{C\times C}$ is the damped input Hessian of one layer ($\delta=10^{-2}$ of the mean diagonal in every build), $H=(I+\bar U)D(I+\bar U)^\top$ its LDL factorization in a fixed coding order with $\bar U$ strictly upper triangular and $D=\mathrm{diag}(D_1,\dots,D_C)\succ0$, $w\in\mathbb R^C$ one output row, $\hat w$ its quantization and $e=w-\hat w$.

\subsection{The loss identity}

\begin{lemma}[identity (i)]
\label{lem:ldl}
Let column-wise error feedback quantize coordinate $j$ against the target $\tau_j=w_j+\sum_{i<j}\bar U_{ij}e_i$ and define $\eta_j=\tau_j-\hat w_j$. Then $\eta=(I+\bar U)^\top e$ and
\[
e^\top He=\sum_{i=1}^{C}D_i\,\eta_i^2
\]
for every $\hat w\in\mathbb R^C$. Under block feedback with blocks $b$ of consecutive coordinates, the same argument gives $e^\top He=\sum_b\eta_b^\top D_b\eta_b$ with $D_b$ the Schur complement of the trailing submatrix of $H$ onto block $b$.
\end{lemma}

\begin{proof}
By definition $\eta_j=w_j-\hat w_j+\sum_{i<j}\bar U_{ij}e_i=e_j+(\bar U^\top e)_j$, i.e.\ $\eta=(I+\bar U)^\top e$. Hence $e^\top He=e^\top(I+\bar U)D(I+\bar U)^\top e=\eta^\top D\eta$. For block feedback replace the scalar LDL factorization by the block one, $H=(I+\bar U_B)D_B(I+\bar U_B)^\top$ with $D_B=\mathrm{blockdiag}(D_b)$; the same substitution gives the block form. No property of $\hat w$ is used.
\end{proof}

The per-step quantity $D_i\eta_i^2$ is the optimal-brain-surgeon loss of rounding coordinate $i$ given all later coordinates \citep{hassibi1993obs}, and $D_i=1/[H_{F_i}^{-1}]_{ii}$ is the reciprocal of the quantity by which GPTQ normalizes its residuals \citep{frantar2023gptq}; Lemma~\ref{lem:ldl} states that these per-step losses add exactly to the layer loss.

\subsection{Proof of Proposition~\ref{prop:main}, parts (ii) and (iii)}

\begin{proof}[Proof of (ii)]
By Lemma~\ref{lem:ldl} the loss has no cross terms in $\eta$, so $\mathbb E[e^\top He]=\sum_iD_i\,\mathbb E[\eta_i^2]$. Under the hypothesis $\mathbb E[\eta_i^2]=\sigma^2$ for all $i$ this equals $\sigma^2\,\mathrm{tr}(D)$. Since $\det(I+\bar U)=1$, $\prod_iD_i=\det H$, and the inequality of arithmetic and geometric means gives $\mathrm{tr}(D)\ge C\,(\prod_iD_i)^{1/C}=C\det(H)^{1/C}$, with equality iff all $D_i$ are equal.
\end{proof}

\begin{proof}[Proof of (iii)]
Under the high-resolution rate model $R_i=h_i-\tfrac12\log_2(c\,\sigma_i^2)$, fixing $\sum_iR_i$ fixes $\prod_i\sigma_i^2=:\sigma^{2C}$. Minimizing $\sum_iD_i\sigma_i^2$ subject to $\prod_i\sigma_i^2=\sigma^{2C}$ by a Lagrange multiplier gives $D_i\sigma_i^2=\lambda$ for all $i$; the constraint then yields $\lambda=\sigma^2\,(\prod_iD_i)^{1/C}=\sigma^2\,\mathrm{GM}(D)$ and the minimum $C\sigma^2\,\mathrm{GM}(D)$. This is the transform-coding-gain argument of \citet{huang1963block}. For the second statement, the change of variables $y=D^{1/2}x$ maps the weighted norm $\|D^{1/2}\cdot\|$ to the Euclidean norm and scales the density by $\det(D)^{-1/2}$; Zador's high-resolution formula for fixed-rate vector quantization \citep{zador1982} depends on the density only through a factor that scales by $\det(D)^{1/C}$, which gives the same floor $C\sigma^2\,\mathrm{GM}(D)$ for a code shaped to the weighted norm.
\end{proof}

\begin{remark}[on the hypothesis of (ii)]
Equal residual second moments across coordinates is the high-resolution behaviour of a single alphabet applied to every coordinate. At two bits, with per-group absmax normalization and overload, it holds only approximately, so (ii) is an idealization of the Euclidean search rather than a bound on it: with $\mathbb E[\eta_i^2]=\sigma^2+\epsilon_i$ the loss is $\sigma^2\,\mathrm{tr}(D)+\sum_iD_i\epsilon_i$, whose second term has no fixed sign. The realized gain of the weighted search is measured directly in Section~\ref{sec:rot}.
\end{remark}

\subsection{A rotation flattens $D$ only partially; a data-aware one flattens it exactly}

\begin{example}
\label{ex:2x2}
Let $H=\mathrm{diag}(10,1)$. In the native basis $D=(10,1)$, $\mathrm{AM}(D)/\mathrm{GM}(D)=1.74$, i.e.\ $\Gamma(H)=0.40$ bit: feedback has nothing to propagate and the anisotropy is entirely available to a weighted search. A rotation by $45^\circ$ gives $H'=\bigl(\begin{smallmatrix}5.5&-4.5\\-4.5&5.5\end{smallmatrix}\bigr)$, whose raw diagonal is flat but whose LDL diagonal in coding order is $D'=(1.82,\,5.5)$: $\mathrm{AM}/\mathrm{GM}=1.16$, $\Gamma(H')=0.10$ bit, and $\prod_iD'_i=\det H=10$ as Proposition~\ref{prop:main}(ii) requires. With $H'$ as the native Hessian, a raw-diagonal weight $H'_{ii}$ sees no anisotropy, the LDL weight sees the residual $0.10$ bit, and feedback recovers the rest: the Euclidean loss is $11\sigma^2$ without feedback, $\sigma^2\,\mathrm{tr}(D')=7.32\sigma^2$ with it, against the floor $2\sigma^2\sqrt{10}=6.32\sigma^2$.
\end{example}

The rotation that equalizes the LDL diagonal exactly is not the $45^\circ$ one but the geometric-mean decomposition of $H^{1/2}$ \citep{jiang2005gmd}, here the angle $\theta$ with $10\sin^2\theta+\cos^2\theta=\sqrt{10}$, i.e.\ $\theta\approx29.3^\circ$; after it a Euclidean search with feedback attains the floor. That rotation depends on $H$. A random rotation does not, and the example shows that it need not flatten $D$; Appendix~\ref{app:ramp} describes what it does instead.

\subsection{The substitution hypothesis is not a theorem}

Write, for a fixed block size $G$ and coding order, $\Lambda_{\mathrm{nat,W}}=\sum_bG\,\mathrm{GM}(\mathrm{diag}\,D_b)$ for the high-resolution proxy loss (per unit residual variance) of the weighted search in the native basis, and $\Lambda_{\mathrm{rot,E}}=\mathbb E_R\sum_b\mathrm{tr}(D'_b)$ for that of the Euclidean search after a random rotation $R$, both under block feedback with in-block residuals of second-moment matrix $\sigma^2I$ (Appendix~\ref{app:ramp}).

\begin{proposition}
\label{prop:noorder}
Neither $\Lambda_{\mathrm{nat,W}}\le\Lambda_{\mathrm{rot,E}}$ nor $\Lambda_{\mathrm{nat,W}}\ge\Lambda_{\mathrm{rot,E}}$ holds for all positive definite $H$.
\end{proposition}

\begin{proof}
Fix $C=4$, $G=2$, $H_A=\mathrm{diag}(10,1,10,1)$ and $H_B=\mathrm{diag}(10,10,1,1)$. Both have spectrum $\{10,10,1,1\}$, so for Haar-distributed $R$ the law of $RHR^\top$ is the same for both and $\Lambda_{\mathrm{rot,E}}(H_A)=\Lambda_{\mathrm{rot,E}}(H_B)=:\Phi$. For every orthogonal $R$, with $H'=RHR^\top$, two bounds hold. Each $D'_b$ is a Schur complement of $H'$, so $(D'_b)_{ii}\ge D'_i$, the scalar LDL diagonal of $H'$; since $\prod_iD'_i=\det H$, the inequality of arithmetic and geometric means gives $\sum_b\mathrm{tr}(D'_b)\ge\sum_iD'_i\ge C\det(H)^{1/C}$. A Schur complement is dominated by its block, so $\sum_b\mathrm{tr}(D'_b)\le\sum_b\mathrm{tr}(H'_{bb})=\mathrm{tr}\,H$. The first bound is attained only if all $D'_i$ coincide and the second only if $H'$ is block diagonal, events of Haar measure zero for $H\ne cI$; hence $C\det(H)^{1/C}<\Phi<\mathrm{tr}\,H$, i.e.\ $4\sqrt{10}<\Phi<22$ (numerically $\Phi=17.73\pm0.01$ over $2\cdot10^4$ draws). In the native basis $D_b=\mathrm{diag}(H_b)$, so $\Lambda_{\mathrm{nat,W}}(H_A)=2\cdot2\sqrt{10}=4\sqrt{10}<\Phi$ and $\Lambda_{\mathrm{nat,W}}(H_B)=2\cdot10+2\cdot1=22>\Phi$.
\end{proof}

The two matrices differ only in where the anisotropy of $D$ sits. In $H_A$ it lies inside every coding block, the weighted search attains the floor $C\det(H)^{1/C}$ and a random rotation can only lose; in $H_B$ it lies between blocks, the weighted search has nothing to exploit and the rotation's mixing lowers the Euclidean loss. Interleaving or grouping the eigenvalues extends the construction to any $C$ and $G$ with $2\le G\le C/2$; the boundary cases are one-directional (with a single block the rotated Euclidean proxy equals $\mathrm{tr}\,H$ for every $R$, and with $G=1$ both native proxies equal $\mathrm{tr}\,H$). Whether native weighting matches rotation is therefore a property of where the Hessians of trained language models put their anisotropy, not of positive definite matrices in general. Their native $D$ is dominated by outlier input channels that place large anisotropy inside coding blocks (Figure~\ref{fig:gamma}a,b), the regime of $H_A$; Section~\ref{sec:rot} tests the hypothesis on that population.

\subsection{When the two curvature weights coincide}

\begin{remark}
\label{rem:weights}
Within a block, multiplying all weights $\omega_i$ by a common factor leaves the Viterbi argmin unchanged. Hence the LDL weight $\omega_i=1/d_i^2=D_i$ and the raw-diagonal weight $\omega_i=H_{ii}$ select the same path for every target if $D_i/H_{ii}$ is constant within the block. Since $D_i$ is the residual second moment of $x_i$ after linear prediction from $x_{>i}$ and $H_{ii}=\mathbb E[x_i^2]$, this ratio is $1-R_i^2$, with $R_i^2$ the uncentred squared multiple correlation of channel $i$ with the channels coded after it. The two weights therefore agree exactly on blocks whose channels are equally predictable from later ones, and differ most where a few channels are strongly correlated with the rest. This is the sense in which the near-tie between $H_{ii}$ and $1/d_i^2$ observed on the feedback axis (Appendix~\ref{app:ramp}) is expected.
\end{remark}

\section{Rotation and curvature: the ramp, the proxies, and extended analysis}
\label{app:ramp}

\begin{table}[htbp]
\centering
\footnotesize
\setlength{\tabcolsep}{5pt}
\caption{The three heterogeneities that a rotation homogenizes, and how the rotation-free recipe handles each.}
\label{tab:mechanisms}
\begin{tabular}{@{}p{0.23\linewidth}p{0.47\linewidth}p{0.24\linewidth}@{}}
\toprule
Heterogeneity & Mechanism & Inference cost \\
\midrule
local amplitude $s_g$ & per-(row, 64) group scale with factored second-level quantization (\S\ref{sec:codec}) & one multiply per weight; $0.07$ bpw \\
normalized shape of $z$ & closed-form quantile table, or computed code (\S\ref{sec:codec}) & 128-entry int8 table, or none \\
curvature $D_i$ & weighted branch metric~Eq.~\ref{eq:metric} & none (encode side) \\
coupling $\bar U$ & inter-block: error feedback; intra-block: omitted & none (encode side) \\
\bottomrule
\end{tabular}
\end{table}

\paragraph{Ramp property.}
After a Haar or sign-randomized Hadamard rotation, $D'_i$ follows a smooth ramp from the harmonic mean of the damped spectrum to its arithmetic mean along the coding order (checked numerically, not proved). Its two ends are exact: $D'_C=u^\top Hu$ and $D'_1=1/(v^\top H^{-1}v)$, with $u$ and $v$ the last and first columns of $R$, so the two ends fluctuate on different scales, set by $H$ and by $H^{-1}$. What the argument needs is only that the ramp is locally flat on the scale of a block, which we measure directly: the within-block s.d.\ of $\ln D'_i$ has median $0.013$ on the 128 public Llama-2-7B Hessians ($0.116$ natively). Where it is flat, $\Gamma'_b\approx0$ and the possible gain of weighting is small: \emph{after rotation, weighting is redundant} (prediction S1).

\paragraph{The substitution hypothesis.}
Whether native weighting \emph{matches} rotation is not implied by (i)--(iii); both sides are computable from $H$ alone. Under block feedback with in-block residuals of second-moment matrix $\mathbb E[\eta_b\eta_b^\top]=\sigma^2I$ (equal and uncorrelated; equal diagonal alone does not suffice, since $\mathbb E[\eta_b^\top D_b\eta_b]=\mathrm{tr}(D_b\,\mathbb E[\eta_b\eta_b^\top])$), the high-resolution proxy losses of the four conditions in Section~\ref{sec:rot} are $\Lambda_{\mathrm{nat,E}}=\sum_b\mathrm{tr}(D_b)$, $\Lambda_{\mathrm{nat,W}}=\sum_b G\,\mathrm{GM}(\mathrm{diag}\,D_b)$, $\Lambda_{\mathrm{rot,E}}=\mathbb E_R\sum_b\mathrm{tr}(D'_b)$ and $\Lambda_{\mathrm{rot,W}}\approx\Lambda_{\mathrm{rot,E}}$; in this notation $\Gamma_b$ of Eq.~\ref{eq:gamma} is $\tfrac12\log_2(\Lambda_{\mathrm{nat,E}}/\Lambda_{\mathrm{nat,W}})$, the block-local gain of Section~\ref{sec:sub}. The hypothesis is $\Lambda_{\mathrm{nat,W}}\approx\Lambda_{\mathrm{rot,E}}$ per layer. Neither ordering holds for general matrices (Appendix~\ref{app:proofs}), so this is a statement about the Hessians of trained LLMs, where outlier input channels put large, block-local anisotropy into the native $D$. It predicts: (P1) in the native basis, the realized gain of the weighted metric tracks $\Gamma_b$; (P2) after rotation it vanishes wherever $\Gamma'_b\approx0$.

\paragraph{The $H$-only proxies on the public Llama-2-7B Hessians.}
With the 128 public Hessians (6144 sequences, damping $10^{-2}$ of the mean diagonal, $G=64$): after the full-dimension RHT or a Haar rotation, $\Gamma'_b$ has median $0.0001$ and maximum $0.04$ (RHT) or $0.20$ (Haar) bit, and $\Lambda_{\mathrm{rot,W}}$ is within $2\%$ of $\Lambda_{\mathrm{rot,E}}$ in 125 of 128 matrices, so S1 holds on $H$ alone. In the native basis $\Gamma_b$ has median $0.038$ bit (mean $0.085$) and is large only where the spectrum collapses: $1.69$ bit for the first layer's attention input ($\mathrm{erank}=4$), $0.97$ for the second layer's down-projection ($\mathrm{erank}=1$), $0.80$ for the last layer's. The rotation's own proxy gain for a Euclidean search, $\Lambda_{\mathrm{nat,E}}/\Lambda_{\mathrm{rot,E}}$, has median $1.105$ ($0.072$ bit) and reaches $2.2$ ($0.58$ bit) in those layers. The substitution ratio $\Lambda_{\mathrm{nat,W}}/\Lambda_{\mathrm{rot,E}}$ has median $0.998$ (interquartile $0.996$--$1.096$) and is $\le1$ in 96 of 128 matrices; native weighting wins by a wide margin exactly where the rotation gains most, and loses on the attention output projections (median $1.077$, ten of 32 above $1.1$). These proxies assume equal per-coordinate residual variance and bound the anisotropy mechanism only.

\begin{figure}[htbp]
\centering
\includegraphics[width=\linewidth]{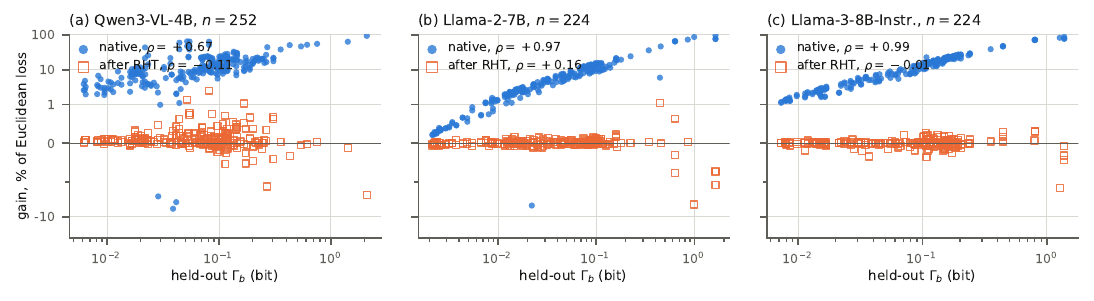}
\caption{The plane of Figure~\ref{fig:gamma}(c) on all three models: realized per-matrix gain of $1/d_i^2$ weighting on the held-out layer objective, as a fraction of the Euclidean loss ($1-L_W/L_E$), against $\Gamma_b$ of Eq.~\ref{eq:gamma} computed from the full held-out Hessian (64 sequences no cell calibrated on), native (filled circles) and after the full-dimension randomized Hadamard (hollow squares). The native gain tracks $\Gamma_b$ on every model (Spearman $+0.67$, $+0.97$, $+0.99$) and the rotated one does not ($-0.11$, $+0.16$, $-0.01$); the rotation lowers the median $\Gamma_b$ from $0.075$, $0.032$, $0.103$ to $0.0027$, $0.0003$, $0.0005$ bit, and the median native gains follow the native medians ($11.7\%$, $4.6\%$, $13.2\%$; per-matrix medians, whereas the model-level percentages below sum over matrices). Figure~\ref{fig:gamma}(c) plots panel (b) alone. Symlog vertical axis ($|{\cdot}|<1\%$ linear).}
\label{fig:gain3}
\end{figure}

\paragraph{Held-out layer objective.}
Scoring every cell on $\mathrm{tr}(EHE^\top)/\mathrm{tr}(WHW^\top)$ with $H$ from 64 sequences no cell calibrated on (summed over coded matrices; rotation-invariant, so rotated and native cells compare directly) separates what the metric does from what the rotation does, and the two do not agree with the tasks in the same way. The metric's effect tracks tasks: P1 removes $17.2\pm0.3\%$ of the Euclidean loss on Qwen3-4B, $10.3\pm0.2\%$ on Llama-3-8B and $5.8\pm0.03\%$ on Llama-2-7B, on all four slices of each; S1 is $-0.14\%$, $-0.01\%$ and $-0.01\%$ (the two rotated searches agree to $2\times10^{-3}$ percentage points in weight-space error on every model); and matrix by matrix the native gain follows $\Gamma_b$ at Spearman $+0.67$, $+0.97$ and $+0.99$ (Figure~\ref{fig:gain3}) while after the RHT it does not ($-0.11$, $+0.16$, $-0.01$). The rotation's effect does not track tasks: on Qwen3-4B the full-dimension RHT lowers the held-out objective of a Euclidean search by $31.9\pm0.8\%$, twice the metric's gain, and the rotated Euclidean cell sits $17.7\pm0.7\%$ below the weighted native one on this objective on all four slices, yet the two are within about a point on the 34-task mean (SUB $+1.10\pm0.63$) and WikiText-2 perplexity leans to the native weighted one ($0.36\pm0.29$, all four slices); on Llama-3-8B and Llama-2-7B the same two cells differ by $0.33\pm0.02\%$ and $0.32\pm0.03\%$ on the layer objective, fifty times less, and by no more on tasks ($+0.07\pm0.84$ and $+0.39\pm0.66$). Qwen3-4B is the one model of three where this objective sees the rotation at all. A per-layer quadratic loss is the wrong yardstick for a rotation, which changes the error's direction in every layer at once; we report it for the metric, where it agrees with every other measure, and read the rotation on end-to-end quantities only. End-to-end perplexity on Qwen3-4B (WikiText-2, context 4096, $n=4$) tells the same story for the metric: P1 $-1.19\pm0.73$, S1 $-0.27\pm0.19$, SUB $+0.36\pm0.29$, R-w $+0.09\pm0.17$; P1, S1 and SUB each carry their sign on all four slices, so on this yardstick the weight retains a small residual value after the rotation and the native weighted search is $0.36$ perplexity ahead of the rotated Euclidean one.

\paragraph{Weighting and feedback.}
The branch weight and the LDLQ feedback address the same quantity from two sides, so we separate them on Qwen3-4B in the format of Table~\ref{tab:main}, scored on the same 34-task suite as Table~\ref{tab:rot}. Switching the feedback off (round-to-nearest within the trellis search) and the weight off costs $10.3$ points; restoring either one alone recovers most of it: the feedback alone is worth $+7.85\pm0.47$ points and the $H_{ii}$ branch weight alone, with no feedback, $+8.73\pm0.29$, each positive on all four calibration slices. Using both buys $+10.29\pm0.34$ ($1/d_i^2$) or $+10.47\pm0.35$ ($H_{ii}$), not $+16.6$: the interaction is $-6.3$ points, negative on all four slices, so the two mechanisms are near-substitutes rather than complements, and once the feedback is on the weight is worth $+2.44\pm0.57$ ($1/d_i^2$, the P1 of Table~\ref{tab:rot}) or $+2.62\pm0.56$ ($H_{ii}$). Weight-only exceeds feedback-only by $+0.88\pm0.27$ points, and the two curvature weights under feedback differ by less than a fifth of a point ($-0.18\pm0.12$), consistent with Section~\ref{sec:search}. Every contrast is paired over the same four calibration slices; the arm with neither mechanism uses no Hessian and is a single build. Two format caveats apply to the feedback-free arms, and both are small relative to the effect: they code each row as a single trellis sequence and so pay one start state per row rather than one per 64-column block ($2.068$ against $2.288$ bpw, in their favour), and the factored scale field is defined only on the block-feedback path, so they use the flat int4 grid (worth about $-0.8$ points, against them, Section~\ref{sec:ablate}). The missing start state is itself worth $4$--$5$ points (Table~\ref{tab:abl}), so charging it in full still leaves both single-mechanism effects above $+2$ points.

\paragraph{Which rotation.}
The rotation of Table~\ref{tab:rot} is the one the baselines deploy: a sign-randomized Hadamard on the full input dimension, $H_{2^m}\otimes H_K$ with a Sloane Hadamard matrix $H_K$ for non-power-of-two widths ($K=20$ for the 2560-wide and $K=172$ for the 11008-wide inputs, $K=28$ for 14336, $K=1$ for 4096, as in QTIP's tables; $K=76$, a Paley-II matrix QTIP never needed, for Qwen3-4B's 9728-wide down-projection input), the construction of QuIP\# and QTIP. A blocked Hadamard on 256 columns, the cheaper stand-in we used in an earlier version of this experiment, leaves the rotated Euclidean search $1.3$--$2.8$ points behind the weighted native search on all three models; the full-dimension transform closes that gap entirely (Table~\ref{tab:rot}), as does a dense per-matrix Haar rotation on Qwen3-4B, on the same 34-task suite ($+1.32\pm0.70$ against native weighted, as against $+1.10$ for the RHT; $+0.22\pm0.34$ against the RHT; weighting after it $-0.27\pm0.87$). The two dense rotations are indistinguishable from each other and within about a point of the weighted native search; the block Hadamard's shortfall was a mixing effect, not evidence about the weight.

\paragraph{What the native basis carries.}
The native basis is heterogeneous in three ways, and the codec meets each where it lives: the amplitude field $s_g$ carries local magnitude, the branch weight $D_i$ carries local sensitivity, and the boundary-state field $b_g$ of Section~\ref{sec:search} carries the local non-stationarity that feedback itself creates at a block boundary. The last is not something a rotation removes: it comes from the feedback, which the baselines use as well, and the start state lifts them in the rotated basis too (Section~\ref{sec:main}).

\section{Extended related work}
\label{app:related}

\paragraph{What a rotation repairs.}
QuIP's Table~2 shows that for a \emph{scalar} grid at two bits, removing incoherence is a cliff: outlier input channels overwhelm a grid that has no local scale. That cliff is the amplitude and shape job of the rotation, which our group scales and quantile table do in the native basis; it is not evidence about the curvature job. Leech-lattice VQ \citep{vanderouderaa2026llvq} reports the same from the other side: with a norm-preserving correction in place of plain GPTQ, dropping the rotation moves Llama-2-7B at two bits from $191.9$ to $6.9$ perplexity, leaving the Hadamard worth about $0.2$: part of what a rotation buys is compensation for a mismatch between solver and code, not a property of the weights.

\paragraph{Trellis successors.}
Proteus \citep{yang2026proteus} keeps QTIP's rotation, block LDLQ and Euclidean Viterbi but replaces the lookup table with a computed Gaussianizing code, chosen per layer from a small candidate pool by a Hessian-weighted \emph{selection} objective (curvature enters the choice of code, not the branch cost), and then learns per-channel scales and distills per-layer affines. Its throughput advantage at large batch is attributed to codebook shared-memory traffic rather than to the rotation. Q-Palette \citep{lee2025qpalette} derives the optimal bit allocation for Gaussianized weights and provides fractional-bit trellis, vector and scalar quantizers; BCJR-QAT \citep{bcjrqat2026} replaces the Viterbi argmax by forward--backward marginals so that the trellis can be trained through, and puts the curvature in the training objective.

\paragraph{Fine-tuning, and what a published ``two bits'' contains.}
PV-Tuning \citep{malinovskii2024pvtuning} holds the code fixed and varies only the fine-tuner: AQLM at $2.01$ bpw on Llama-2-7B moves from $7.38$ WikiText-2 perplexity with calibration alone to $6.69$, $6.41$ and $5.90$ as the fine-tuner strengthens. QTIP's own end-to-end stage is worth about one perplexity point at 7B and about $0.2$ at 70B. Converted through the no-fine-tuning rate slope this is $0.7$--$1.1$ effective bits at 7B, more than the spread between any two codes we compare. Even tuning per-column scales alone moves QuIP\# and QTIP by more than a perplexity point in a common pipeline \citep{vanderouderaa2026llvq}. QuIP\#, QTIP, AQLM and Proteus all publish their headline low-bit numbers after such a stage, and the rates they publish differ in what they charge for its parameters and for metadata. We therefore fix the fine-tuning protocol at none for every method we run and compare at matched rate; how the margins change under a fine-tuning stage is left to future work.

\paragraph{Other relevant lines.}
GPTQ's column-wise feedback has recently been identified with Babai's nearest-plane algorithm on the Hessian lattice \citep{chen2026babai}, whose error bound is a quadratic form in the same diagonal $D$ that we place in the search. YAQA \citep{tseng2025yaqa} improves which Hessian is used and measures quantized models by their KL divergence to the original. ICQuant \citep{li2025icquant} codes outlier positions by their index gaps; HIGGS \citep{malinovskii2025higgs} derives MSE-optimal grids for Hadamard-rotated weights; NF4 is the Gaussian-quantile construction that our empirical quantile table generalizes. For mixtures of experts, MoEQuant \citep{hu2025moequant} studies calibration imbalance across experts and recent work allocates bits across experts. Calibration-set sensitivity is documented by \citet{williams2024calibration}; for a deterministic codec its consequence is sharper, since the calibration slice is the only replicate axis. Quantization is known to hurt multi-step reasoning first \citep{li2024evaluating,liu2025reasoning}, which is where our two-bit gains concentrate.

\section{Full perplexity and accuracy tables}
\label{app:full}

Tables~\ref{tab:ppl} and~\ref{tab:down} extend Table~\ref{tab:main} with the NF baseline (a Hessian-aware quantile grid with a Hadamard transform).

The $x.25^{+}$ rows are not a handicap: a 64-group fp16 scale is the cheapest way either baseline can spend comparable extra bits, and in nine of ten (model, codec) cells it gives that baseline its own best two-bit perplexity, so the comparison is against a strengthened baseline. Both rate rows are kept because the exchange rate is not transferable: the same scale is worth $-1.64$ perplexity to QTIP on Qwen3-4B and only $-0.22$ to Proteus, so a baseline cannot be interpolated to our rate. The $_{+\mathrm{start}}$ rows spend comparable side information on our stored start state instead (Appendix~\ref{app:ablate}), and Section~\ref{sec:ablate} bounds the question from our side with a variant of our own recipe at $2.07$ bpw.

\paragraph{Both side channels at once.} The $_{+\mathrm{start+g128}}$ rows give each baseline our stored start state \emph{and} a 128-group fp16 scale, $2.31$ bpw for QTIP and $2.34$ for Proteus, so for the first time the baseline rate lies above ours. On the four dense models the ordering does not move: at two bits our margin is $+1.0$ to $+2.7$ downstream points on five of six contrasts (Table~\ref{tab:ci}) and $1.5$--$2.8$ perplexity on Qwen3-4B, with the Llamas within $0.25$; above two bits every contrast lies within $\pm1.1$ points and $0.31$ perplexity. The second channel adds little to the first---$-0.7$ to $+0.8$ points over the $_{+\mathrm{start}}$ rows, and a higher perplexity in five of the eighteen cells on the three 4--8B models (e.g.\ QTIP on Qwen3-4B at two bits, $13.99\to14.31$; Proteus on Qwen3-4B at three and four bits)---so the two kinds of side information are not additive, and neither one interpolated to our rate would close the gap. On Qwen3-32B the same holds for perplexity ($6.99$ and $6.89$ against our $6.90$). We did not run these arms on the mixture of experts.

\paragraph{Removing the table.}
The table-free variant of Section~\ref{sec:codec} (\ours$_{\mathrm{no\,LUT}}$ in Tables~\ref{tab:ppl} and~\ref{tab:down}) is within noise of the table variant at three and four bits: over the six model $\times$ tier cells above two bits, four paired differences include zero, the others are $-0.5$ [$-0.8$, $-0.2$] and $+0.7$ [$+0.1$, $+1.6$] (Table~\ref{tab:nl_ci}), and perplexity moves by at most $0.11$. At two bits the table matters on two of the four models: removing it costs $2.1$ perplexity and $1.2$ points on Qwen3-4B, and $2.6$ points on Llama-3.1-8B, where perplexity changes by only $0.06$; once more, perplexity understates a downstream change. On the other two it is close to free ($0.5$ points on Llama-3-8B) or slightly negative (Llama-2-7B).

\begin{table}[htbp]
\centering
\footnotesize
\setlength{\tabcolsep}{3.5pt}
\caption{WikiText-2 perplexity (context 4096; lower is better), all methods. Baselines: QTIP~\citep{tseng2024qtip}, Proteus~\citep{yang2026proteus}, HIGGS~\citep{malinovskii2025higgs}, GPTQ~\citep{frantar2023gptq}; GPTQ (g64) is a symmetric int-$k$ grid with one fp16 scale per 64 weights and no zero point, $k+0.25$ bpw. Bold: best in tier; the fp16-scale variant of our build ($_{\mathrm{fp16}}$, $x.47$ bpw) is listed for reference and not ranked.}
\label{tab:ppl}
\input{tab_ppl}
\end{table}

\begin{table}[htbp]
\centering
\footnotesize
\setlength{\tabcolsep}{3.5pt}
\caption{Mean downstream accuracy (\%) over the task suite of Section~\ref{sec:protocol} ($n=34$ tasks for the Qwen3 models, 14 for each Llama model), all methods. ``$\times$'': the model fails to produce valid outputs on most tasks. Bold: best in tier; the fp16-scale variant ($_{\mathrm{fp16}}$) is listed for reference and not ranked.}
\label{tab:down}
\input{tab_down}
\end{table}

\section{Paired comparisons}
\label{app:ci}

Table~\ref{tab:ci} lists, for every comparison quoted in Section~\ref{sec:main}, the mean paired difference in accuracy over tasks between our build and a trellis baseline at the neighbouring rate, with a bootstrap 95\% confidence interval (4000 resamples of the task set) and the number of tasks on which ours scores higher. The interval answers whether the margin depends on the task mix; it does not cover calibration-set variability, which Section~\ref{sec:rot} measures on three models at about $0.3$--$0.6$ points s.d.\ per build and which is of the same order as the smaller margins in this table. A win counts the same whether the margin is 20 points or 0.5, so the last column shows how broad each gain is: at two bits our build is higher on 82\% of all task comparisons pooled over the four dense models and the four trellis baselines (313 of 384), 69\% at three bits (266 of 384) and 65\% at four bits (161 of 248).

The margin shrinks with model size, as quantization loss does in general \citep{li2024evaluating}, and is largest on generative tasks (GSM8K $+14$ to $+21$ points against $+2$ to $+3$ on multiple choice), but survives their removal: without GSM8K the margin over Proteus is $+2.8$ instead of $+3.6$ points and every contrast but one remains significant.

\paragraph{The table-free variant across cells.}
Across the nine model $\times$ tier cells of Table~\ref{tab:nl_ci}, the table-free build stays ahead of QTIP in all nine and is ahead of Proteus or within $0.2$ points of it in all nine.

{\footnotesize\setlength{\tabcolsep}{4pt}
\input{tab_ci}}

\begin{table}[htbp]
\centering
\footnotesize
\setlength{\tabcolsep}{3pt}
\caption{The table-free variant: paired differences in downstream accuracy (points) of \textsc{CurveTQ}$_{\mathrm{no\,LUT}}$ against the table variant and against QTIP and Proteus at the neighbouring rate, over tasks, with bootstrap 95\% CIs.}
\label{tab:nl_ci}
\input{tab_nl_ci}
\end{table}

\section{Weight-space error}
\label{app:werr}
Table~\ref{tab:werr} reports the relative weight-space error under the Hessian norm for every build of Table~\ref{tab:ppl}.
Our build has the lowest relative weight-space error in every cell but one (Qwen3-4B at two bits, where Proteus$_{+\mathrm{start+g128}}$ is lower by $0.01$). Weight-space error is a build-time diagnostic; the ordering of the other methods on it does not carry over to perplexity or tasks.

\begin{table}[htbp]
\centering
\footnotesize
\setlength{\tabcolsep}{3.5pt}
\caption{Relative weight-space error $\|W-\hat W\|_H/\|W\|_H$ in percent (lower is better), same builds as Table~\ref{tab:ppl}. Bold: best in tier; the fp16-scale variant ($_{\mathrm{fp16}}$) is listed for reference and not ranked.}
\label{tab:werr}
\input{tab_werr}
\end{table}

\section{Task suite and per-task results}
\label{app:tasks}

\begin{table}[htbp]
\centering
\footnotesize
\setlength{\tabcolsep}{2.8pt}
\caption{Per-model protocol. ``LM linears'': only the language-model linear layers are quantized; the vision tower stays in bf16 and is outside the reported rate, which covers the coded linear layers. Llama-2-7B is the model on which QTIP and Proteus report their headline numbers, so we follow their context length and their four zero-shot tasks and our reproductions can be checked against the published values.}
\label{tab:models}
\begin{tabular}{llclll}
\toprule
Checkpoint & Short name & Ctx. & Quantized & Tasks & Methods run \\
\midrule
Llama-2-7B (base)     & Llama-2-7B   & 4096 & all linears & 4 zero-shot (QTIP's) & all, $+$ no-LUT \\
Llama-3-8B-Instruct   & Llama-3-8B   & 4096 & all linears & 14 text & all, $+$ no-LUT \\
Llama-3.1-8B-Instruct & Llama-3.1-8B & 4096 & all linears & 14 text & all, $+$ no-LUT \\
Qwen3-VL-4B-Instruct  & Qwen3-4B     & 4096 & LM linears  & 34 (14 $+$ 4 LiveBench $+$ 16 MM) & all, $+$ no-LUT \\
Qwen3-VL-32B-Instruct & Qwen3-32B    & 4096 & LM linears  & 34 & QTIP, Proteus, ours \\
Qwen3.6-35B-A3B (MoE) & MoE          & 4096 & expert $+$ attn.\ linears & 30 & QTIP, Proteus, ours \\
\bottomrule
\end{tabular}
\end{table}

\paragraph{Coverage of the table-free variant.}
The table-free variant of Section~\ref{sec:codec} is run on the four dense models up to 8B.

\paragraph{Tasks and settings.}
Text tasks other than LiveBench are run through lm-evaluation-harness \citep{gao2024lmeval} with a vLLM backend, LiveBench (release 2024-11-25; reasoning, mathematics, data analysis and language) through its official framework, and all multimodal tasks through lmms-eval \citep{zhang2024lmmseval}, with one configuration per task for every method. Few-shot settings follow the Open LLM Leaderboard v1 convention: MMLU 5-shot, GSM8K 5-shot (exact match, flexible answer extraction), ARC-challenge 25-shot, HellaSwag 10-shot (normalized accuracy), WinoGrande 5-shot; MMLU-Pro and MMLU-ProX (English) use the harness's 5-shot generative configuration; ARC-easy, PIQA, BoolQ, TruthfulQA (MC2), GPQA-Diamond, IFEval (prompt-level strict accuracy), INCLUDE (43 languages, averaged) and LiveBench are zero-shot. Generation is greedy (temperature 0, one sample per prompt), with a limit of 256 new tokens for GSM8K, 1280 for IFEval, 2048 for MMLU-Pro and MMLU-ProX, and 4096 for LiveBench and for all multimodal tasks, whose task-level limits are raised to the harness's model-level default of 4096. Text tasks use the harness's plain completion prompts, without the chat template, for the Qwen3-VL and Llama models; this lowers absolute IFEval scores relative to chat-formatted evaluations but is the same for every method. Qwen3.6-35B-A3B, whose template enables a reasoning mode, is run with its chat template and reasoning disabled, again for every method. LiveBench and all multimodal tasks always apply the model's chat template, and images are resized to at most $401{,}408$ pixels. Multimodal tasks use the default lmms-eval metric of each benchmark. The tables below list every task behind the means of Table~\ref{tab:down}; MMLU 0-shot was also run and is omitted from the means to avoid counting MMLU twice.

\paragraph{Reading the tables.}
Rows are methods grouped by bit tier, with the bits per weight after the method name ($^{+}$: published rate plus $0.003$ bpw of sign bits); columns are tasks. Within a tier, \textsc{CurveTQ} is the factored-scale build ($x.29$ bpw) and \textsc{CurveTQ}$_{\mathrm{fp16}}$ the fp16-scale build ($x.47$ bpw), listed for reference and not ranked. Bold marks the best quantized entry per task within a tier. ``--'' marks a configuration that was not run. The weaker baselines produce near-zero scores on isolated tasks, e.g.\ HIGGS on MMLU-Pro at three bits ($0.5$) and GPTQ on several generative and long-context tasks at two bits; we report these runs as measured and do not re-run or drop them. Likewise, NF (Hessian) has an outlying perplexity on Llama-3.1-8B at four bits ($12.77$) while its downstream mean is normal. The last column is the unweighted mean over the tasks shown; for Qwen3-4B the means of Table~\ref{tab:down} are over the union of the text and multimodal tables.

\begin{table}[p]
\centering
\scriptsize
\setlength{\tabcolsep}{3.4pt}
\caption{Llama-3-8B-Instruct: per-task accuracy (\%) for every method and bit tier.}
\label{tab:full_l8}
\input{tab_full_l8}
\end{table}

\begin{table}[p]
\centering
\scriptsize
\setlength{\tabcolsep}{3.4pt}
\caption{Llama-3.1-8B-Instruct: per-task accuracy (\%) for every method and bit tier.}
\label{tab:full_l31}
\input{tab_full_l31}
\end{table}

\begin{table}[p]
\centering
\scriptsize
\setlength{\tabcolsep}{2.4pt}
\caption{Qwen3-VL-4B-Instruct, text tasks: per-task accuracy (\%) for every method and bit tier.}
\label{tab:full_q4_text}
\resizebox{\linewidth}{!}{\input{tab_full_q4_text}}
\end{table}

\begin{table}[p]
\centering
\scriptsize
\setlength{\tabcolsep}{2.4pt}
\caption{Qwen3-VL-4B-Instruct, multimodal tasks: per-task accuracy (\%) for every method and bit tier. The vision tower is kept in bf16 for all methods.}
\label{tab:full_q4_vl}
\input{tab_full_q4_vl}
\end{table}

\begin{table}[t]
\centering
\scriptsize
\setlength{\tabcolsep}{2.4pt}
\caption{Qwen3-VL-32B-Instruct at the 2- and 3-bit tiers: per-task accuracy (\%) on multimodal tasks (top) and text tasks (bottom); the same 34-task suite as Qwen3-4B.}
\label{tab:full_q32}
\input{tab_full_q32_vl}
\vspace{8pt}
\resizebox{\linewidth}{!}{\input{tab_full_q32_text}}
\end{table}

\begin{table}[t]
\centering
\scriptsize
\setlength{\tabcolsep}{2.4pt}
\caption{Qwen3.6-35B-A3B (mixture of experts) at the 2-bit tier: per-task accuracy (\%) on multimodal tasks (top) and text tasks (bottom); the mean is over 30 tasks. $^{+}$: published rate plus $0.003$ bpw of sign bits.}
\label{tab:full_q35}
\input{tab_full_q35_vl}
\vspace{8pt}
\input{tab_full_q35_text}
\end{table}

\section{Ablations: full perplexity ladder}
\label{app:ablate}

\paragraph{The approximations behind the branch metric.}
Three approximations are made. First, the block form $\sum_b\eta_b^\top D_b\eta_b$ of Lemma~\ref{lem:ldl} is exact for the residual of the block factor $U_{bb}^{-1}$; by default we propagate each block's residual through GPTQ's column factor (divided by $\mathrm{diag}(U_{bb})$, $U$ the upper Cholesky factor of $H^{-1}$), so the block form evaluated on the residual the search sees differs from the true loss by a correction of no fixed sign; in our tests the block factor gave a slightly higher weight-space error. Second, Eq.~\ref{eq:metric} keeps only the diagonal of $D_b$; the dropped in-block coupling is corrected neither by the search nor by the between-block feedback. The baselines' Euclidean search omits both parts; we keep the diagonal. Third, $1/d_i^2$ conditions on \emph{all} columns after $i$, including in-block columns that block feedback never corrects, so it under-weights early in-block columns relative to $\mathrm{diag}(D_b)$. The raw diagonal $\omega_i=H_{ii}$ is the weight of the objective \emph{without} feedback and is kept as a control. Appendix~\ref{app:ramp} compares $H_{ii}$ with $1/d_i^2$ on the 34-task suite; both Hessian-derived weights repair the few matrices on which the Euclidean search fails outright, and $1/d_i^2$ is our default.

\paragraph{The two yardsticks.}
The search ablations of Section~\ref{sec:ablate} separate perplexity from downstream accuracy. The curvature weight protects the coordinates with large $D_i$, the outlier input channels, and the two yardsticks disagree on what that protection is worth: a $0.24$ improvement in perplexity accompanies a $3.3$-point loss downstream, far outside the $0.4$-point spread we see on re-evaluating one model. The same decoupling appears in the table-free variant at two bits (Section~\ref{sec:main}), where a $0.06$ change in perplexity accompanied a $2.6$-point downstream change on Llama-3.1-8B, and it is why we report both yardsticks throughout.

\begin{table}[htbp]
\centering
\footnotesize
\setlength{\tabcolsep}{4pt}
\caption{Table~\ref{tab:abl} in full: rate deltas, and paired differences to the reference row with 95\% CIs over the 34 tasks.}
\label{tab:abl_ci}
\input{tab_abl_full_ci}
\end{table}

Table~\ref{tab:abl_ci} repeats the ablation of Table~\ref{tab:abl} with the rate deltas and the paired downstream differences with bootstrap confidence intervals; the intermediate tail-biting length ($T{=}128$) is included there. Table~\ref{tab:abl_full} gives the component ladder of Table~\ref{tab:abl} on all three dense models and at all three rates, on WikiText-2 perplexity (context 4096). All rows are single calibration-slice builds. The fp16-scale build, at $0.18$ bpw more, is not uniformly better than the factored int4 build: it is $0.36$ worse on Qwen3-4B at two bits ($11.86$ against $11.50$) and within $\pm0.13$ of it, in either direction, on the other eight cells. The removal row also splits the last step of the ladder, which adds the per-module table and the factored field together: of its $1.27$ on Qwen3-4B, the factored field over fp16 scales accounts for $0.36$ and the per-module table for the rest. The seed noise of the 256-entry EM base, measured on Qwen3-4B with $n=5$ seeds, is $0.19$ at two bits, $0.09$ at three and $0.04$ at four; differences below about twice these values should be read as noise. Two regularities hold across models and rates. The edge-indexed row widens the emission window from the 14-bit state of the row above to 16 bits at the same $2^{14}$ survivors; it is not a difference from QTIP's node-emitting form, which emits from the same 16-bit window (Section~\ref{sec:codec}). Every component's effect shrinks with rate: the edge-indexed emission is worth $1.27$ perplexity on Qwen3-4B at two bits, $0.17$ at three and $0.01$ at four, and the emission table $2.13$, $0.06$ and $0.11$. And the effects are larger on Qwen3-4B than on the two Llama models, whose two-bit steps are all within $0.3$ perplexity of zero; the bit savings of the rate-buying steps are the same on all models. The search ablations were run at two bits only.

\begin{table}[htbp]
\centering
\footnotesize
\setlength{\tabcolsep}{4pt}
\caption{WikiText-2 perplexity (context 4096) of every configuration of Table~\ref{tab:abl}, three models $\times$ three rates. bpw as in Table~\ref{tab:abl} (the $x$ in $x.469$ etc.\ is the code rate). fp16: Qwen3-4B $9.72$, Llama-3-8B $7.68$, Llama-3.1-8B $6.67$.}
\label{tab:abl_full}
\begin{tabular}{@{}l l rrr rrr rrr@{}}
\toprule
& & \multicolumn{3}{c}{Qwen3-4B} & \multicolumn{3}{c}{Llama-3-8B} & \multicolumn{3}{c}{Llama-3.1-8B} \\
\cmidrule(lr){3-5}\cmidrule(lr){6-8}\cmidrule(lr){9-11}
Configuration & bpw & 2-bit & 3-bit & 4-bit & 2-bit & 3-bit & 4-bit & 2-bit & 3-bit & 4-bit \\
\midrule
128-entry table, EM fit & $x$.469 & 13.32 & 10.46 & 9.81 & 10.10 & 8.17 & 7.80 & 9.35 & 7.09 & 6.79 \\
+ quantile fit & $x$.469 & 13.05 & 10.24 & 9.80 & 10.12 & 8.12 & 7.73 & 9.15 & 7.12 & 6.77 \\
+ int8 scales & $x$.345 & 13.46 & 10.15 & 9.78 & 10.14 & 8.14 & 7.78 & 9.16 & 7.16 & 6.78 \\
+ edge emission & $x$.345 & 12.19 & 10.31 & 9.79 & 9.93 & 8.03 & 7.82 & 9.18 & 7.09 & 6.76 \\
+ int4 scales (flat) & $x$.282 & 12.77 & 10.37 & 9.95 & 9.78 & 8.11 & 7.78 & 9.12 & 7.12 & 6.77 \\
\rowcolor{oursbg}+ factored int4 scales (\textsc{CurveTQ}) & $x$.289 & 11.50 & 10.28 & 9.68 & 9.91 & 8.16 & 7.75 & 9.16 & 7.09 & 6.77 \\
\midrule
\textsc{CurveTQ} $-$ factored scales (fp16) & $x$.469 & 11.86 & 10.18 & 9.83 & 9.85 & 8.06 & 7.76 & 9.29 & 7.08 & 6.75 \\
\textsc{CurveTQ} $-$ table & $x$.289 & 13.63 & 10.34 & 9.79 & 10.01 & 8.06 & 7.78 & 9.23 & 7.11 & 6.79 \\
\midrule
quantile fit $-$ error feedback & $x$.469 & 13.38 & -- & -- & -- & -- & -- & -- & -- & -- \\
quantile fit, Euclidean metric & $x$.469 & 12.81 & -- & -- & -- & -- & -- & -- & -- & -- \\
\bottomrule
\end{tabular}
\end{table}

\paragraph{Ablating the block termination: the full comparison.}
The last block of Table~\ref{tab:abl} replaces our block-local trellis by the two constructions that store no per-block state, at an otherwise identical recipe. A row-long \emph{chained} trellis, in which each block starts in the state the previous block ended in, saves $0.216$ bpw and costs $2.09$ perplexity and $-5.5$ [$-7.0$, $-4.2$] downstream points; a per-block \emph{tail-biting} trellis, QTIP's construction applied to the 64-weight block, saves $0.220$ bpw and costs $4.10$ perplexity and $-5.1$ [$-6.9$, $-3.5$] points, although its weight-space error is the lower of the two. The chained variant loses on all 34 tasks and the tail-biting one on 32, most on the generative ones (GSM8K $-19$ and $-26$, IFEval $-14$ and $-12$). Read against Table~\ref{tab:main}, the chained variant at $2.07$ bpw still edges QTIP and Proteus with per-group scales on perplexity ($13.59$ against $13.72$ and $13.99$ at $2.25^{+}$) but not downstream ($46.4$ against $47.9$ and $47.3$), and the tail-biting variant trails on both ($15.60$ against QTIP's $15.36$ at $2.00^{+}$). Lengthening the tail-biting sequence by stacking rows, at the same $2.07$ bpw, recovers $0.98$ and $2.22$ of the $4.10$ perplexity at $T=128$ and $T=256$; the longer one passes the chained variant ($13.38$ against $13.59$) but still trails the stored start state by $1.84$ (Appendix~\ref{app:rate}). Downstream the recovery is smaller and slower than perplexity suggests: the first doubling returns nothing measurable ($+0.1$ [$-1.1$, $+1.2$] over the 64-length block for $0.98$ perplexity) and the second returns $+0.8$ [$-0.8$, $+2.2$], so at QTIP's own tile length the construction is still $-4.2$ [$-6.4$, $-2.5$] points behind the stored start state and loses on 31 of 34 tasks. For comparison, the factored scale field saves $0.18$ bpw at a cost of $0.4$ [$-0.3$, $+1.1$] downstream points two rows above: the start state is the side channel on which a bit buys the most, which is why we keep it and charge it.

\paragraph{How much of the margin is the extra rate?} The $T{=}256$ variant also bounds how much of our margin over the baselines is bought with rate rather than designed, without any assumption about how perplexity trades against bits. It is our recipe with QTIP's termination and QTIP's tile length, storing no start bits and carrying a $14$-bit state against QTIP's $16$, and it sits at $2.0685$ bpw, $0.066$ above the baselines' published $2.003$ and $0.219$ below ours. At that rate it reaches $13.38$ perplexity against QTIP's $15.36$ and Proteus's $14.21$. Our perplexity margin over QTIP on this model therefore splits into $1.98$ from the weighted search, the quantile table and the factored scales, spent at $0.066$ bpw, and $1.84$ from the block-local start state, spent at $0.219$ bpw (both halves measured against the same rebuild; $0.04$ of the $3.86$ in Table~\ref{tab:main} is the rebuild's GPU difference). The two halves are comparable, so on perplexity the stored state is not what creates the gap---it is what doubles it. Downstream the split is one-sided in the other direction: at $47.7$ points the variant is $+0.8$ [$-0.9$, $+2.2$] over QTIP and $+0.4$ [$-1.4$, $+2.0$] over Proteus with per-group scales, neither measurable, while the start state is worth $+4.2$ [$+2.5$, $+6.4$]. Without a stored start state the other components recover perplexity at this rate but not accuracy; with one they deliver both. We report the variant as a control, not a configuration to use (it is $1.84$ perplexity and $4.2$ points worse than the build we ship), and it is run on Qwen3-4B only.

\paragraph{The start state transfers.} The $_{+\mathrm{start}}$ rows of Tables~\ref{tab:main}, \ref{tab:ppl} and~\ref{tab:down} are QTIP from its released code and Proteus from our implementation, each with one change: every 256-weight tile's trellis is split into four 64-weight segments, each coded by one Viterbi pass whose start and end states are free, i.e.\ our block-local construction on their trellis and code. Nothing else differs from the published baseline (same rotation, LDLQ, scale and code). At two bits the state is worth $0.7$--$1.4$ perplexity and $+1.0$ to $+3.3$ downstream points to the baselines on the three Instruct models, more than the 64-group scale buys them. The gain shrinks with rate: at three bits it is $-0.2$ to $+0.7$ points and at four bits within $\pm0.5$. Against these strengthened baselines our two-bit margin is $+0.9$ to $+3.1$ points (five of six contrasts exclude zero; Table~\ref{tab:ci}), $+1.2$ where Proteus$_{+\mathrm{start}}$ has the lower perplexity (Llama-3-8B, $9.83$ against $9.91$). As in Appendix~\ref{app:ci}, the margin is largest on generative tasks (GSM8K, IFEval and the LiveBench categories: $+3.6$ to $+8.4$ points on average) but does not depend on them: over the multiple-choice tasks alone it is $+1.5$, $+1.2$ and $+2.2$ against QTIP$_{+\mathrm{start}}$ on the three models, all excluding zero, and $+0.2$, $+0.8$ and $+0.8$ against Proteus$_{+\mathrm{start}}$, within noise; at three and four bits the twelve contrasts lie within $\pm0.7$ and two are marginally significant. On Llama-2-7B and Qwen3-32B the start state alone brings the baselines level with us ($62.9$--$63.1$, $65.2$--$66.4$ and $66.1$--$66.9$ points on Llama-2-7B at two, three and four bits; $66.0$--$66.1$ on Qwen3-32B at two bits, from $65.0$--$65.8$ without it).

\section{Implementation and throughput measurements}
\label{app:impl}

\begin{table}[htbp]
\centering
\footnotesize
\setlength{\tabcolsep}{3.5pt}
\caption{Decode throughput (tokens/s) on one RTX~3090 at batch 1, 64 and 128 and at 2, 3 and 4 bits. Llama-2-7B QTIP at batch 1 and fp16 are the published numbers \citep{tseng2024qtip}, which our harness reproduces (fp16: $52.25$); QTIP publishes batch 1 only, so its batch 64/128 rows, Qwen3-4B fp16 and QTIP, and all Proteus rows are our measurements. ``$+$rot.'': our decoder with QTIP's activation-side Hadamard added; ``$-$start'': without the stored start state. Bold: fastest of the three codecs.}
\label{tab:tput_full}
\input{tab_tput_full}
\end{table}

\paragraph{Setup.}
All kernels are measured on one RTX~3090 in the benchmark setting of QTIP \citep{tseng2024qtip}: end-to-end decode steps at batch 1, 64 and 128 and at 2, 3 and 4 bits, captured in a CUDA graph; the reported figure is the median over steps. Our harness reproduces QTIP's published fp16 throughput on Llama-2-7B ($52.25$ against their $52.5$ tokens/s), so for that model we quote QTIP's published batch-1 numbers at 2, 3 and 4 bits rather than re-measure them. QTIP reports batch~1 only, so its batch 64/128 rows, and all of its Qwen3-4B rows (a model QTIP does not report), are our runs of its released kernels. Proteus has no public kernel, so all of its rows use a kernel written from its description, and its published throughput is measured on an H800, which we do not have. Its rows here should therefore be read as what that description yields in our harness on this GPU, not as a reproduction of Proteus's published results; the anchored comparison in this table is the one against QTIP, whose released kernels we run and whose published fp16 figure our harness reproduces. QTIP runs its released HYB code ($L=16$, $V=2$), the code of its published results and kernels, for accuracy and throughput alike. QTIP and Proteus include the Hadamard transform on the activations that incoherence processing requires at inference; our decoder does not need one. Our decoder is a fused dequantize-and-multiply kernel: one warp per output row loads the block's start state, walks the trellis, reconstructs each weight by one table read from the emission window (in the table-free variant, the $(14{+}k)$-bit window is mixed by one multiply and one xor-shift and the four bytes of the result are summed, centred and divided by the exact standard deviation $\sqrt{4(256^2-1)/12}$, giving a unit-variance code value; one global scale is then chosen by a nine-point line search on 16 calibration rows, so the stored table of $0.0001$ bpw disappears), applies the group scale reconstructed from the factored field, and accumulates against the activation. The ``$+$rot.'' arm of Table~\ref{tab:tput_full} runs the same decoder with an activation-side Hadamard inserted before each quantized layer, following QTIP's transform, so that its cost measures what incoherence processing adds to a decoder that would otherwise not need it; the ``$-$start'' arm runs it with every block started from a fixed state, so that its gain measures the decode-time cost of the stored start state. Both arms are decode-only measurements: the resulting models are not the ones evaluated for accuracy.

\section{Rate accounting and stream audit}
\label{app:rate}

\paragraph{Rate accounting.}
Every bits-per-weight figure in this paper counts the $k$ code bits, all scale side information, the emission table and the $L$-bit start state of every block; for QTIP and Proteus we add their reported $0.003$ bpw of sign bits, marked $^{+}$. The per-tensor arithmetic and a byte-level audit of the serialized stream follow.

\paragraph{Why the start state and not a tail-biting block.}
QTIP instead codes 256-weight tiles as \emph{tail-biting} sequences, whose first state is a function of the tile's own last bits, so no start state is stored and the rate is exactly $k$ \citep{tseng2024qtip}. The construction is right for its setting: a long tile over a rotated, near-i.i.d.\ source, coded without weights, where the tile boundary falls on nothing in particular and the constraint removes only $L/256$ of the path freedom. Under error feedback it has a cost that the i.i.d.\ setting hides. The block is the 64-column feedback block, so the constraint removes $L/64$ of the freedom; and the block is not stationary: its target $\tilde z$ carries the residual pushed forward from the previous block, concentrated on the first coordinates, and $\omega_i$ varies across it. Tail-biting ties the first $\lceil L/k\rceil$ coordinates of a block to its last ones, so the coordinates that carry the inherited residual are the ones the search may not choose freely, and their error is passed on again. We measure this directly: with a tail-biting block the reconstruction error at the first block position is $2.6$--$2.9\times$ the within-block mean, and lengthening the sequence by stacking rows improves it monotonically but has not reached the stored-state result by QTIP's own tile length ($15.60 \to 14.62 \to 13.38$ at $T=64,128,256$, against $11.54$; Table~\ref{tab:auditfull}): block length is part of the problem, but where the constraint falls matters as much as how much freedom it removes. The freedom is spent where the feedback puts the difficulty: on the built models the error at the first block position is $0.55\times$ the within-block mean and recovers to the mean by position eight.

\paragraph{What is counted.}
For the $_{+\mathrm{start}}$ baselines the stored stream of a 64-weight segment is the $L$-bit start state followed by the code bits; a node-emitting trellis emits from the state itself, so its last $KV$ code bits are never emitted and the segment costs $L-KV$ bits over tail-biting ($12$ per block for QTIP, $V=2$; $14$ for Proteus, $V=1$; $0.19$ and $0.22$ bpw at two bits), whereas our edge-emitting trellis uses every code bit and is charged the full $L=14$.
The reported rate splits into a payload term and the side information the decoder needs, $R=R_{\mathrm{payload}}+R_{\mathrm{scale}}+R_{\mathrm{state}}+R_{\mathrm{table}}$, with the boundary adaptation rate $R_{\mathrm{state}}=L/G$ (Section~\ref{sec:search}). Term by term: for a tensor of $R$ rows and $C$ columns coded at $k$ bits with blocks of $G=64$, the reported rate is the sum of: $k$ bits per weight for the trellis path; $L/G=14/64$ bits per weight for the $L$-bit start state of every (row, block) pair (Section~\ref{sec:search}); $4/G$ bits for the int4 residual of the factored scale field, $16/(G\cdot 256)$ for its fp16 block absmax, $16/C$ for the fp16 row factor and $16\lceil C/G\rceil/(RC)$ for the fp16 column factor; and $128\cdot 8+32$ bits per tensor for the emission table. Both levels of the scale field are rounded up against the stored fp16 factors, so $\hat s\ge s$ holds exactly for every group; the builder checks that no normalized weight leaves $[-1,1]$. The fp16-scale variant replaces the four scale terms by $16/G$. Summed over the coded tensors of each model this gives the ``Reported'' column of Table~\ref{tab:rate}; the start-state term is $0.219$ bpw for every model, the rest varies with tensor shapes in the fourth decimal. Embeddings, output head, norms, router and vision tower stay in bf16 and are outside the reported rate, which covers the coded linear layers. For reference the table also lists the rate a row-long trellis with one stored state per row would have; that construction is evaluated in Section~\ref{sec:ablate}. The start state is stored in full: a variant that keeps the top $L-b$ bits of the previous block's end state and spends only the low $b=8$ freely, at $0.125$ instead of $0.219$ bpw, saves $0.094$ bpw and costs $0.51$, $0.43$ and $0.44$ perplexity at two bits on Qwen3-4B, Llama-3-8B and Llama-3.1-8B respectively. Every reported build stores all $L$ bits.

\begin{table}[htbp]
\centering\footnotesize
\caption{Stored rate of every built model. ``Code, scales, table'' excludes start states; ``one state per row'' is the hypothetical row-long trellis of Section~\ref{sec:ablate}; ``one state per block'' is our coder and the rate reported throughout.}
\label{tab:rate}
\input{tab_rate}
\end{table}

\paragraph{Stream audit.}
To check the arithmetic against bytes, we instrumented the builder to serialize, for every tensor, exactly what a decoder needs: the $k$-bit symbol stream, the int4 residual codes packed two per byte, the fp16 block absmax, row and column factors, the 128 int8 knots with their fp32 step, and the start states. A separate routine reads those bytes back, regenerates the emission index through the mixer of Section~\ref{sec:codec}, walks the trellis and rescales, and compares the result with the tensor the coder wrote into the checkpoint. Table~\ref{tab:stream} shows the first decoder layer of Qwen3-4B at $k=2$: every component matches its formula to the byte, the full model is recovered bit-exactly from the stream, and the two decoding rows show that the per-block state is what a decoder needs: given one state per row instead, $10.2\%$ of the weights differ. The same check was run on the whole model and on the chained and tail-biting variants of Section~\ref{sec:ablate} (Table~\ref{tab:auditfull}); all decode bit-exactly at their stated rates. Two further variants lengthen the tail-biting sequence by coding two and four consecutive rows of one 64-column block as a single sequence, leaving feedback and scales per row; the second is QTIP's length of 256. Both decode bit-exactly with zero stored start bits, their first and last states agree in all 252 modules, and both pack to the same $2.0685$ bpw as the 64-length block. Perplexity falls monotonically with sequence length ($15.60$, $14.62$, $13.38$ at $T=64$, $128$, $256$), so roughly half of the tail-biting penalty is the shortness of the sequence. The remainder is not: at $T=256$ (QTIP's own tile length, and the length at which its rate argument is stated) the construction still trails a stored start state by $1.84$ ppl. The two doublings return $-0.98$ and $-1.24$, so we cannot claim the curve has saturated, only that it has not reached the stored-state result by the length the comparison method uses. With a 64-length tail-biting block the error at the first position of a block is $2.6$--$2.9\times$ the within-block mean (above); lengthening the sequence dilutes that boundary term across more coordinates without removing it. Downstream the ladder does not halve: the two doublings return $+0.1$ [$-1.1$, $+1.2$] and $+0.8$ [$-0.8$, $+2.2$] suite-mean points, neither measurable, against the start state's $+4.2$ [$+2.5$, $+6.4$]; at $T=256$ the construction is still $-4.2$ points behind and loses on 31 of 34 tasks. The perplexity recovery is therefore not evidence that sequence length can substitute for the stored state, only that perplexity is the less demanding of the two measures here (Section~\ref{sec:ablate}).

\begin{table}[htbp]
\centering\footnotesize
\caption{Serialized stream of the first decoder layer of Qwen3-4B at $k=2$, and whether the checkpoint tensors are recovered from the stream alone.}
\label{tab:stream}
\input{tab_stream}
\end{table}

\begin{table}[htbp]
\centering\footnotesize
\caption{Qwen3-4B at $k=2$ with the block terminations of Section~\ref{sec:ablate} plus a $T{=}256$ tail-biting variant, same recipe otherwise. ``Packed'' is the rate a decoder needs, verified from the serialized stream. The rebuild of our coder was made on a different GPU type from the checkpoint of Table~\ref{tab:main} and lands $0.04$ perplexity from it ($11.54$ against $11.50$); same-GPU rebuilds are bit-identical, and the checkpoints we release are those of Table~\ref{tab:main}.}
\label{tab:auditfull}
\input{tab_audit_full}
\end{table}

\end{document}

%% file: math_commands.tex
\usepackage{amsmath,amsfonts,bm}

\def\eqref#1{equation~\ref{#1}}

\def\1{\bm{1}}

\DeclareMathAlphabet{\mathsfit}{\encodingdefault}{\sfdefault}{m}{sl}
\SetMathAlphabet{\mathsfit}{bold}{\encodingdefault}{\sfdefault}{bx}{n}



%% file: tab_main.tex
\begin{tabular}{@{}l@{\hspace{4pt}}c@{\hspace{6pt}}rrrrr@{\hspace{8pt}}crrrr@{}}
\toprule
& & \multicolumn{5}{c}{WikiText-2 perplexity $\downarrow$} & \multicolumn{5}{c}{Downstream accuracy (\%) $\uparrow$} \\
\cmidrule(lr){3-7}\cmidrule(lr){8-12}
Method & bpw & L2-7B & Q3-4B & L3-8B & L3.1-8B & Q3-32B & L2-7B$^{\ast}$ & Q3-4B & L3-8B & L3.1-8B & Q3-32B \\
\midrule
fp16 & 16 & 5.21 & 9.72 & 7.68 & 6.67 & 6.17 & 66.6 (43.1/76.3/77.7/69.2) & 62.0 & 61.8 & 63.2 & 70.2 \\
\midrule
QTIP & 2.00$^{+}$ & 6.86 & 15.36 & 11.53 & 10.31 & 7.21 & 61.7 (36.0/70.0/75.0/65.7) & 46.8 & 48.8 & 49.8 & 65.4 \\
Proteus & 2.00$^{+}$ & 6.73 & 14.21 & 10.99 & 10.02 & 7.09 & 61.8 (35.5/70.0/74.5/67.0) & 48.0 & 48.8 & 51.5 & 65.0 \\
\rowcolor{oursbg}\textbf{CurveTQ} & 2.29 & 6.44 & \textbf{11.50} & 9.91 & \textbf{9.16} & 6.90 & 62.9 (38.6/71.6/74.3/67.0) & \textbf{51.8} & \textbf{53.2} & \textbf{54.0} & \textbf{66.3} \\
QTIP$_{+\mathrm{g64}}$ & 2.25$^{+}$ & 6.72 & 13.72 & 11.19 & 10.03 & 7.07 & 61.6 (35.8/68.9/74.6/67.1) & 47.9 & 48.4 & 48.8 & 65.0 \\
Proteus$_{+\mathrm{g64}}$ & 2.25$^{+}$ & 6.70 & 13.99 & 10.81 & 10.05 & 6.97 & 61.6 (36.0/68.7/74.9/66.6) & 47.3 & 48.3 & 51.0 & 65.8 \\
QTIP$_{+\mathrm{start}}$ & 2.19$^{+}$ & 6.49 & 13.99 & 10.22 & 9.42 & 6.93 & \textbf{63.1} (36.7/72.0/75.8/67.9) & 49.8 & 51.1 & 50.9 & 66.1 \\
Proteus$_{+\mathrm{start}}$ & 2.22$^{+}$ & \textbf{6.31} & 13.15 & \textbf{9.83} & 9.30 & \textbf{6.88} & 63.0 (37.3/70.8/75.7/68.2) & 51.0 & 52.1 & 52.5 & 66.1 \\
HIGGS & 2.25 & 10.95 & 21.23 & 19.09 & 18.40 & -- & 55.8 (28.9/60.7/70.1/63.6) & 36.8 & 40.1 & 39.8 & -- \\
GPTQ (g64) & 2.25 & 1567 & 40.75 & 536 & 530 & -- & 37.3 (21.5/26.1/52.4/49.3) & 17.1 & $\times$ & $\times$ & -- \\
\midrule
QTIP & 3.00$^{+}$ & 5.50 & 10.68 & 8.30 & 7.32 & 6.38 & 65.7 (42.1/75.1/77.0/68.4) & 58.9 & 58.6 & 60.2 & 69.2 \\
Proteus & 3.00$^{+}$ & 5.47 & 10.73 & 8.15 & 7.23 & 6.37 & 65.6 (42.4/75.2/76.6/68.3) & 59.0 & 59.2 & 60.0 & 69.3 \\
\rowcolor{oursbg}\textbf{CurveTQ} & 3.29 & 5.42 & \textbf{10.28} & 8.16 & \textbf{7.09} & 6.32 & 65.9 (42.1/75.8/77.5/68.1) & \textbf{60.0} & \textbf{59.9} & 60.9 & \textbf{69.7} \\
QTIP$_{+\mathrm{start}}$ & 3.16$^{+}$ & 5.48 & 10.62 & \textbf{8.12} & 7.27 & 6.37 & 65.1 (41.1/74.5/77.7/67.2) & 59.6 & 59.2 & \textbf{61.2} & 69.3 \\
Proteus$_{+\mathrm{start}}$ & 3.20$^{+}$ & \textbf{5.41} & 10.43 & 8.16 & 7.12 & \textbf{6.31} & \textbf{66.4} (41.9/76.1/77.3/70.3) & \textbf{60.0} & 59.7 & \textbf{61.2} & 69.4 \\
\bottomrule
\end{tabular}

%% file: tab_abl.tex
\begin{tabular}{@{}l r r r@{}}
\toprule
Configuration & bpw & ppl & acc. \\
\midrule
\multicolumn{4}{@{}l}{\emph{Building up the codec (one component per row)}} \\
128-entry table, EM fit (start) & 2.469 & 13.32 & 50.4 \\
+ closed-form quantile fit & 2.469 & 13.05 & 50.8 \\
+ int8 second-level scales & 2.345 & 13.46 & 50.9 \\
+ edge-indexed emission & 2.345 & 12.19 & 51.9 \\
+ int4 scales (flat grid) & 2.282 & 12.77 & 51.0 \\
\rowcolor{oursbg}\textbf{+ factored int4 scales, per-module table} & \textbf{2.289} & 11.50 & 51.8 \\
\midrule
\multicolumn{4}{@{}l}{\emph{Removing one component from \textsc{CurveTQ}}} \\
$-$ factored scales (fp16 scales) & 2.469 & 11.86 & 52.2 \\
$-$ emission table (computed code) & 2.289 & 13.63 & 50.7 \\
\midrule
\multicolumn{4}{@{}l}{\emph{Ablating the search (from the quantile-fit row)}} \\
$-$ error feedback & 2.469 & 13.38 & 48.2 \\
Euclidean metric instead of $1/d_i^2$ & 2.469 & 12.81 & 47.4 \\
\midrule
\multicolumn{4}{@{}l}{\emph{Block termination (from \textsc{CurveTQ})}} \\
chained trellis, one state per row & 2.073 & 13.59 & 46.4 \\
tail-biting block, no start state & 2.069 & 15.60 & 46.8 \\
\quad same, $T{=}256$ (QTIP's tile) & 2.069 & 13.38 & 47.7 \\
\bottomrule
\end{tabular}

%% file: tab_tput.tex
\begin{tabular}{@{}l c r r r >{\columncolor{oursbg}}r@{}}
\toprule
Model & bits & fp16 & QTIP & Prot. & \textsc{CTQ} \\
\midrule
L2-7B & 2 & 52.5 & 127.0 & 105.9 & \textbf{139.0} \\
 & 3 & & 119.0 & 103.5 & \textbf{125.7} \\
 & 4 & & 109.0 & 102.6 & \textbf{114.9} \\
\midrule
Q3-4B & 2 & 74.3 & 153.6 & 117.9 & \textbf{160.8} \\
 & 3 & & 144.2 & 115.4 & \textbf{151.7} \\
 & 4 & & 137.2 & 113.2 & \textbf{146.2} \\
\bottomrule
\end{tabular}

%% file: tab_ppl.tex
\begin{tabular}{@{}l c c rrrr@{}}
\toprule
Method & bpw & Rot. & Qwen3-4B & Llama-3-8B & Llama-3.1-8B & Qwen3-32B \\
\midrule
fp16 & 16 & none & 9.72 & 7.68 & 6.67 & 6.17 \\
\midrule
QTIP & 2.00$^{+}$ & RHT & 15.36 & 11.53 & 10.31 & 7.21 \\
Proteus & 2.00$^{+}$ & RHT & 14.21 & 10.99 & 10.02 & 7.09 \\
\textbf{CurveTQ} & 2.29 & none & \textbf{11.50} & 9.91 & 9.16 & 6.90 \\
\textbf{CurveTQ}$_{\mathrm{no\,LUT}}$ & 2.29 & none & 13.63 & 10.01 & 9.23 & -- \\
QTIP$_{+\mathrm{g64}}$ & 2.25$^{+}$ & RHT & 13.72 & 11.19 & 10.03 & 7.07 \\
Proteus$_{+\mathrm{g64}}$ & 2.25$^{+}$ & RHT & 13.99 & 10.81 & 10.05 & 6.97 \\
QTIP$_{+\mathrm{start}}$ & 2.19$^{+}$ & RHT & 13.99 & 10.22 & 9.42 & 6.93 \\
Proteus$_{+\mathrm{start}}$ & 2.22$^{+}$ & RHT & 13.15 & 9.83 & 9.30 & \textbf{6.88} \\
QTIP$_{+\mathrm{start+g128}}$ & 2.31$^{+}$ & RHT & 14.31 & 10.17 & 9.35 & 6.99 \\
Proteus$_{+\mathrm{start+g128}}$ & 2.34$^{+}$ & RHT & 13.02 & \textbf{9.82} & \textbf{9.03} & 6.89 \\
HIGGS & 2.25 & Had. & 21.23 & 19.09 & 18.40 & -- \\
GPTQ (g64) & 2.25 & none & 40.75 & 536 & 530 & -- \\
\textbf{CurveTQ}$_{\mathrm{fp16}}$ & 2.47 & none & 11.86 & 9.85 & 9.29 & -- \\
\midrule
QTIP & 3.00$^{+}$ & RHT & 10.68 & 8.30 & 7.32 & 6.38 \\
Proteus & 3.00$^{+}$ & RHT & 10.73 & 8.15 & 7.23 & 6.37 \\
\textbf{CurveTQ} & 3.29 & none & \textbf{10.28} & 8.16 & \textbf{7.09} & 6.32 \\
\textbf{CurveTQ}$_{\mathrm{no\,LUT}}$ & 3.29 & none & 10.34 & \textbf{8.06} & 7.11 & -- \\
QTIP$_{+\mathrm{g64}}$ & 3.25$^{+}$ & RHT & 10.51 & 8.30 & 7.28 & 6.36 \\
Proteus$_{+\mathrm{g64}}$ & 3.25$^{+}$ & RHT & 10.75 & 8.20 & 7.23 & 6.35 \\
QTIP$_{+\mathrm{start}}$ & 3.16$^{+}$ & RHT & 10.62 & 8.12 & 7.27 & 6.37 \\
Proteus$_{+\mathrm{start}}$ & 3.20$^{+}$ & RHT & 10.43 & 8.16 & 7.12 & \textbf{6.31} \\
QTIP$_{+\mathrm{start+g128}}$ & 3.31$^{+}$ & RHT & 10.29 & 8.17 & 7.21 & -- \\
Proteus$_{+\mathrm{start+g128}}$ & 3.34$^{+}$ & RHT & 10.56 & 8.14 & 7.22 & -- \\
HIGGS & 3.25 & Had. & 11.06 & 8.57 & 7.79 & -- \\
NF (Hessian) & 3.25 & Had. & 11.03 & 234 & 108 & -- \\
GPTQ (g64) & 3.25 & none & 11.85 & 9.72 & 8.93 & -- \\
\textbf{CurveTQ}$_{\mathrm{fp16}}$ & 3.47 & none & 10.18 & 8.06 & 7.08 & -- \\
\midrule
QTIP & 4.00$^{+}$ & RHT & 9.99 & 7.81 & 6.86 & -- \\
Proteus & 4.00$^{+}$ & RHT & 9.83 & 7.77 & 6.80 & -- \\
\textbf{CurveTQ} & 4.29 & none & \textbf{9.68} & \textbf{7.75} & \textbf{6.77} & -- \\
\textbf{CurveTQ}$_{\mathrm{no\,LUT}}$ & 4.29 & none & 9.79 & 7.78 & 6.79 & -- \\
QTIP$_{+\mathrm{g64}}$ & 4.25$^{+}$ & RHT & 10.02 & 7.80 & 6.84 & -- \\
Proteus$_{+\mathrm{g64}}$ & 4.25$^{+}$ & RHT & 9.88 & 7.82 & 6.80 & -- \\
QTIP$_{+\mathrm{start}}$ & 4.13$^{+}$ & RHT & 10.11 & 7.82 & 6.83 & -- \\
Proteus$_{+\mathrm{start}}$ & 4.19$^{+}$ & RHT & 9.92 & 7.77 & 6.80 & -- \\
QTIP$_{+\mathrm{start+g128}}$ & 4.31$^{+}$ & RHT & 9.79 & 7.81 & 6.82 & -- \\
Proteus$_{+\mathrm{start+g128}}$ & 4.34$^{+}$ & RHT & 9.99 & 7.76 & \textbf{6.77} & -- \\
HIGGS & 4.25 & Had. & 10.01 & 7.87 & 6.88 & -- \\
NF (Hessian) & 4.25 & Had. & 9.79 & 7.96 & 12.77 & -- \\
GPTQ (g64) & 4.25 & none & 10.33 & 8.15 & 7.18 & -- \\
\textbf{CurveTQ}$_{\mathrm{fp16}}$ & 4.47 & none & 9.83 & 7.76 & 6.75 & -- \\
\bottomrule
\end{tabular}

%% file: tab_down.tex
\begin{tabular}{@{}l c rrrr@{}}
\toprule
Method & bpw & Qwen3-4B & Llama-3-8B & Llama-3.1-8B & Qwen3-32B \\
\midrule
fp16 & 16 & 62.0 & 61.8 & 63.2 & 70.2 \\
\midrule
QTIP & 2.00$^{+}$ & 46.8 & 48.8 & 49.8 & 65.4 \\
Proteus & 2.00$^{+}$ & 48.0 & 48.8 & 51.5 & 65.0 \\
\textbf{CurveTQ} & 2.29 & \textbf{51.8} & \textbf{53.2} & \textbf{54.0} & \textbf{66.3} \\
\textbf{CurveTQ}$_{\mathrm{no\,LUT}}$ & 2.29 & 50.7 & 52.7 & 51.4 & -- \\
QTIP$_{+\mathrm{g64}}$ & 2.25$^{+}$ & 47.9 & 48.4 & 48.8 & 65.0 \\
Proteus$_{+\mathrm{g64}}$ & 2.25$^{+}$ & 47.3 & 48.3 & 51.0 & 65.8 \\
QTIP$_{+\mathrm{start}}$ & 2.19$^{+}$ & 49.8 & 51.1 & 50.9 & 66.1 \\
Proteus$_{+\mathrm{start}}$ & 2.22$^{+}$ & 51.0 & 52.1 & 52.5 & 66.1 \\
QTIP$_{+\mathrm{start+g128}}$ & 2.31$^{+}$ & 49.1 & 50.5 & 51.3 & -- \\
Proteus$_{+\mathrm{start+g128}}$ & 2.34$^{+}$ & 50.6 & 52.9 & 53.0 & -- \\
HIGGS & 2.25 & 36.8 & 40.1 & 39.8 & -- \\
GPTQ (g64) & 2.25 & 17.1 & $\times$ & $\times$ & -- \\
\textbf{CurveTQ}$_{\mathrm{fp16}}$ & 2.47 & 52.2 & 52.5 & 53.5 & -- \\
\midrule
QTIP & 3.00$^{+}$ & 58.9 & 58.6 & 60.2 & 69.2 \\
Proteus & 3.00$^{+}$ & 59.0 & 59.2 & 60.0 & 69.3 \\
\textbf{CurveTQ} & 3.29 & 60.0 & \textbf{59.9} & 60.9 & \textbf{69.7} \\
\textbf{CurveTQ}$_{\mathrm{no\,LUT}}$ & 3.29 & \textbf{60.2} & 59.4 & \textbf{61.7} & -- \\
QTIP$_{+\mathrm{g64}}$ & 3.25$^{+}$ & 58.8 & 59.5 & 60.9 & 69.4 \\
Proteus$_{+\mathrm{g64}}$ & 3.25$^{+}$ & 59.8 & 59.1 & 61.3 & 69.4 \\
QTIP$_{+\mathrm{start}}$ & 3.16$^{+}$ & 59.6 & 59.2 & 61.2 & 69.3 \\
Proteus$_{+\mathrm{start}}$ & 3.20$^{+}$ & 60.0 & 59.7 & 61.2 & 69.4 \\
QTIP$_{+\mathrm{start+g128}}$ & 3.31$^{+}$ & 59.3 & 58.8 & 61.3 & -- \\
Proteus$_{+\mathrm{start+g128}}$ & 3.34$^{+}$ & 59.6 & \textbf{59.9} & 60.6 & -- \\
HIGGS & 3.25 & 55.2 & 57.6 & 58.7 & -- \\
NF (Hessian) & 3.25 & 55.2 & -- & -- & -- \\
GPTQ (g64) & 3.25 & 50.8 & 52.3 & 52.0 & -- \\
\textbf{CurveTQ}$_{\mathrm{fp16}}$ & 3.47 & 60.2 & 60.1 & 61.5 & -- \\
\midrule
QTIP & 4.00$^{+}$ & 61.5 & 61.1 & 62.1 & -- \\
Proteus & 4.00$^{+}$ & 61.7 & 61.2 & 62.7 & -- \\
\textbf{CurveTQ} & 4.29 & 61.8 & \textbf{61.5} & \textbf{63.0} & -- \\
\textbf{CurveTQ}$_{\mathrm{no\,LUT}}$ & 4.29 & 61.6 & 61.2 & 62.6 & -- \\
QTIP$_{+\mathrm{g64}}$ & 4.25$^{+}$ & 61.6 & 60.9 & 62.7 & -- \\
Proteus$_{+\mathrm{g64}}$ & 4.25$^{+}$ & \textbf{61.9} & 61.0 & 62.7 & -- \\
QTIP$_{+\mathrm{start}}$ & 4.13$^{+}$ & 61.5 & 61.3 & 62.5 & -- \\
Proteus$_{+\mathrm{start}}$ & 4.19$^{+}$ & 61.8 & 61.4 & \textbf{63.0} & -- \\
QTIP$_{+\mathrm{start+g128}}$ & 4.31$^{+}$ & 61.6 & 61.1 & 62.4 & -- \\
Proteus$_{+\mathrm{start+g128}}$ & 4.34$^{+}$ & 61.5 & 61.0 & 62.9 & -- \\
HIGGS & 4.25 & 61.2 & 60.6 & 62.7 & -- \\
NF (Hessian) & 4.25 & 61.4 & 59.9 & 62.1 & -- \\
GPTQ (g64) & 4.25 & 59.5 & 59.3 & 61.5 & -- \\
\textbf{CurveTQ}$_{\mathrm{fp16}}$ & 4.47 & 61.7 & 61.5 & 62.5 & -- \\
\bottomrule
\end{tabular}

%% file: tab_ci.tex
\begin{longtable}{@{}l l l r r r@{}}
\caption{Paired differences in downstream accuracy (points) between our build and each trellis baseline, over tasks.}
\label{tab:ci}\\
\toprule
Model & Tier & Contrast & $\Delta$ & 95\% CI & wins/$n$ \\
\midrule
\endfirsthead
\toprule
\multicolumn{6}{@{}l}{\itshape Table~\ref{tab:ci} continued}\\
\midrule
Model & Tier & Contrast & $\Delta$ & 95\% CI & wins/$n$ \\
\midrule
\endhead
Qwen3-4B & 2-bit & CurveTQ $-$ Proteus & $+3.89$ & $[+2.69,\,+5.13]$ & 29/34 \\
Qwen3-4B & 2-bit & CurveTQ $-$ QTIP & $+5.00$ & $[+3.39,\,+7.03]$ & 31/34 \\
Qwen3-4B & 2-bit & CurveTQ $-$ Proteus$_{+\mathrm{g64}}$ & $+4.54$ & $[+2.83,\,+6.98]$ & 30/34 \\
Qwen3-4B & 2-bit & CurveTQ $-$ QTIP$_{+\mathrm{g64}}$ & $+3.97$ & $[+2.53,\,+5.62]$ & 29/34 \\
Qwen3-4B & 2-bit & CurveTQ $-$ Proteus$_{+\mathrm{start}}$ & $+0.86$ & $[-0.21,\,+1.88]$ & 24/34 \\
Qwen3-4B & 2-bit & CurveTQ $-$ QTIP$_{+\mathrm{start}}$ & $+2.03$ & $[+1.14,\,+2.97]$ & 27/34 \\
Qwen3-4B & 2-bit & CurveTQ $-$ Proteus$_{+\mathrm{start+g128}}$ & $+1.26$ & $[+0.29,\,+2.22]$ & 24/34 \\
Qwen3-4B & 2-bit & CurveTQ $-$ QTIP$_{+\mathrm{start+g128}}$ & $+2.69$ & $[+1.64,\,+3.88]$ & 28/34 \\
Qwen3-4B & 3-bit & CurveTQ $-$ Proteus & $+0.93$ & $[+0.40,\,+1.48]$ & 26/34 \\
Qwen3-4B & 3-bit & CurveTQ $-$ QTIP & $+1.06$ & $[+0.28,\,+1.90]$ & 25/34 \\
Qwen3-4B & 3-bit & CurveTQ $-$ Proteus$_{+\mathrm{g64}}$ & $+0.21$ & $[-0.58,\,+0.94]$ & 22/34 \\
Qwen3-4B & 3-bit & CurveTQ $-$ QTIP$_{+\mathrm{g64}}$ & $+1.18$ & $[+0.26,\,+2.09]$ & 26/34 \\
Qwen3-4B & 3-bit & CurveTQ $-$ QTIP$_{+\mathrm{start}}$ & $+0.35$ & $[-0.29,\,+0.95]$ & 22/34 \\
Qwen3-4B & 3-bit & CurveTQ $-$ Proteus$_{+\mathrm{start}}$ & $-0.05$ & $[-0.79,\,+0.65]$ & 18/34 \\
Qwen3-4B & 4-bit & CurveTQ $-$ Proteus & $+0.03$ & $[-0.43,\,+0.51]$ & 20/34 \\
Qwen3-4B & 4-bit & CurveTQ $-$ QTIP & $+0.25$ & $[-0.29,\,+0.88]$ & 23/34 \\
Qwen3-4B & 4-bit & CurveTQ $-$ Proteus$_{+\mathrm{g64}}$ & $-0.24$ & $[-0.75,\,+0.19]$ & 19/34 \\
Qwen3-4B & 4-bit & CurveTQ $-$ QTIP$_{+\mathrm{g64}}$ & $+0.06$ & $[-0.63,\,+0.71]$ & 19/34 \\
\midrule
Llama-3-8B & 2-bit & CurveTQ $-$ Proteus & $+4.45$ & $[+2.37,\,+7.40]$ & 13/14 \\
Llama-3-8B & 2-bit & CurveTQ $-$ QTIP & $+4.42$ & $[+2.39,\,+7.11]$ & 13/14 \\
Llama-3-8B & 2-bit & CurveTQ $-$ Proteus$_{+\mathrm{g64}}$ & $+4.90$ & $[+3.03,\,+7.43]$ & 13/14 \\
Llama-3-8B & 2-bit & CurveTQ $-$ QTIP$_{+\mathrm{g64}}$ & $+4.81$ & $[+3.22,\,+7.17]$ & 14/14 \\
Llama-3-8B & 2-bit & CurveTQ $-$ Proteus$_{+\mathrm{start}}$ & $+1.17$ & $[+0.09,\,+2.39]$ & 9/14 \\
Llama-3-8B & 2-bit & CurveTQ $-$ QTIP$_{+\mathrm{start}}$ & $+2.17$ & $[+0.99,\,+3.68]$ & 13/14 \\
Llama-3-8B & 2-bit & CurveTQ $-$ Proteus$_{+\mathrm{start+g128}}$ & $+0.36$ & $[-0.54,\,+1.26]$ & 7/14 \\
Llama-3-8B & 2-bit & CurveTQ $-$ QTIP$_{+\mathrm{start+g128}}$ & $+2.72$ & $[+0.92,\,+5.10]$ & 11/14 \\
Llama-3-8B & 3-bit & CurveTQ $-$ Proteus & $+0.71$ & $[-0.11,\,+1.67]$ & 11/14 \\
Llama-3-8B & 3-bit & CurveTQ $-$ QTIP & $+1.28$ & $[+0.76,\,+1.88]$ & 13/14 \\
Llama-3-8B & 3-bit & CurveTQ $-$ Proteus$_{+\mathrm{g64}}$ & $+0.83$ & $[+0.14,\,+1.52]$ & 11/14 \\
Llama-3-8B & 3-bit & CurveTQ $-$ QTIP$_{+\mathrm{g64}}$ & $+0.44$ & $[-0.27,\,+1.05]$ & 11/14 \\
Llama-3-8B & 3-bit & CurveTQ $-$ QTIP$_{+\mathrm{start}}$ & $+0.67$ & $[+0.17,\,+1.18]$ & 12/14 \\
Llama-3-8B & 3-bit & CurveTQ $-$ Proteus$_{+\mathrm{start}}$ & $+0.17$ & $[-0.28,\,+0.68]$ & 8/14 \\
Llama-3-8B & 4-bit & CurveTQ $-$ Proteus & $+0.31$ & $[-0.20,\,+1.03]$ & 8/14 \\
Llama-3-8B & 4-bit & CurveTQ $-$ QTIP & $+0.37$ & $[+0.08,\,+0.67]$ & 12/14 \\
Llama-3-8B & 4-bit & CurveTQ $-$ Proteus$_{+\mathrm{g64}}$ & $+0.48$ & $[-0.00,\,+1.19]$ & 10/14 \\
Llama-3-8B & 4-bit & CurveTQ $-$ QTIP$_{+\mathrm{g64}}$ & $+0.65$ & $[+0.32,\,+0.95]$ & 11/14 \\
\midrule
Llama-3.1-8B & 2-bit & CurveTQ $-$ Proteus & $+2.49$ & $[+0.19,\,+5.25]$ & 10/14 \\
Llama-3.1-8B & 2-bit & CurveTQ $-$ QTIP & $+4.20$ & $[+2.05,\,+6.55]$ & 11/14 \\
Llama-3.1-8B & 2-bit & CurveTQ $-$ Proteus$_{+\mathrm{g64}}$ & $+3.02$ & $[+1.22,\,+5.00]$ & 11/14 \\
Llama-3.1-8B & 2-bit & CurveTQ $-$ QTIP$_{+\mathrm{g64}}$ & $+5.15$ & $[+2.75,\,+8.08]$ & 12/14 \\
Llama-3.1-8B & 2-bit & CurveTQ $-$ Proteus$_{+\mathrm{start}}$ & $+1.44$ & $[+0.30,\,+2.76]$ & 10/14 \\
Llama-3.1-8B & 2-bit & CurveTQ $-$ QTIP$_{+\mathrm{start}}$ & $+3.10$ & $[+1.39,\,+5.16]$ & 11/14 \\
Llama-3.1-8B & 2-bit & CurveTQ $-$ Proteus$_{+\mathrm{start+g128}}$ & $+1.02$ & $[+0.13,\,+1.94]$ & 10/14 \\
Llama-3.1-8B & 2-bit & CurveTQ $-$ QTIP$_{+\mathrm{start+g128}}$ & $+2.62$ & $[+0.67,\,+5.07]$ & 10/14 \\
Llama-3.1-8B & 3-bit & CurveTQ $-$ Proteus & $+0.95$ & $[+0.22,\,+1.73]$ & 11/14 \\
Llama-3.1-8B & 3-bit & CurveTQ $-$ QTIP & $+0.70$ & $[+0.00,\,+1.43]$ & 12/14 \\
Llama-3.1-8B & 3-bit & CurveTQ $-$ Proteus$_{+\mathrm{g64}}$ & $-0.38$ & $[-0.86,\,+0.12]$ & 5/14 \\
Llama-3.1-8B & 3-bit & CurveTQ $-$ QTIP$_{+\mathrm{g64}}$ & $+0.06$ & $[-0.66,\,+0.73]$ & 10/14 \\
Llama-3.1-8B & 3-bit & CurveTQ $-$ QTIP$_{+\mathrm{start}}$ & $-0.27$ & $[-1.05,\,+0.37]$ & 9/14 \\
Llama-3.1-8B & 3-bit & CurveTQ $-$ Proteus$_{+\mathrm{start}}$ & $-0.29$ & $[-0.98,\,+0.35]$ & 8/14 \\
Llama-3.1-8B & 4-bit & CurveTQ $-$ Proteus & $+0.28$ & $[-0.16,\,+0.85]$ & 9/14 \\
Llama-3.1-8B & 4-bit & CurveTQ $-$ QTIP & $+0.90$ & $[+0.30,\,+1.71]$ & 11/14 \\
Llama-3.1-8B & 4-bit & CurveTQ $-$ Proteus$_{+\mathrm{g64}}$ & $+0.34$ & $[-0.28,\,+1.03]$ & 10/14 \\
Llama-3.1-8B & 4-bit & CurveTQ $-$ QTIP$_{+\mathrm{g64}}$ & $+0.36$ & $[-0.21,\,+0.92]$ & 9/14 \\
\midrule
Qwen3-32B & 2-bit & CurveTQ $-$ Proteus & $+1.24$ & $[+0.35,\,+2.24]$ & 25/34 \\
Qwen3-32B & 2-bit & CurveTQ $-$ QTIP & $+0.88$ & $[+0.06,\,+1.71]$ & 25/34 \\
Qwen3-32B & 2-bit & CurveTQ $-$ Proteus$_{+\mathrm{g64}}$ & $+0.47$ & $[-0.20,\,+1.16]$ & 22/34 \\
Qwen3-32B & 2-bit & CurveTQ $-$ QTIP$_{+\mathrm{g64}}$ & $+1.27$ & $[+0.43,\,+2.24]$ & 25/34 \\
Qwen3-32B & 2-bit & CurveTQ $-$ Proteus$_{+\mathrm{start}}$ & $+0.23$ & $[-0.37,\,+0.90]$ & 13/34 \\
Qwen3-32B & 2-bit & CurveTQ $-$ QTIP$_{+\mathrm{start}}$ & $+0.23$ & $[-0.52,\,+0.96]$ & 22/34 \\
Qwen3-32B & 3-bit & CurveTQ $-$ Proteus & $+0.40$ & $[-0.12,\,+0.90]$ & 21/34 \\
Qwen3-32B & 3-bit & CurveTQ $-$ QTIP & $+0.49$ & $[+0.05,\,+0.96]$ & 22/34 \\
Qwen3-32B & 3-bit & CurveTQ $-$ Proteus$_{+\mathrm{g64}}$ & $+0.28$ & $[-0.16,\,+0.73]$ & 21/34 \\
Qwen3-32B & 3-bit & CurveTQ $-$ QTIP$_{+\mathrm{g64}}$ & $+0.28$ & $[-0.14,\,+0.71]$ & 19/34 \\
Qwen3-32B & 3-bit & CurveTQ $-$ QTIP$_{+\mathrm{start}}$ & $+0.43$ & $[+0.06,\,+0.82]$ & 20/34 \\
Qwen3-32B & 3-bit & CurveTQ $-$ Proteus$_{+\mathrm{start}}$ & $+0.35$ & $[-0.04,\,+0.74]$ & 21/34 \\
\midrule
Qwen3.6-35B-A3B & 2-bit & CurveTQ $-$ Proteus & $+2.13$ & $[+1.42,\,+2.99]$ & 28/30 \\
Qwen3.6-35B-A3B & 2-bit & CurveTQ $-$ QTIP & $+2.54$ & $[+1.85,\,+3.27]$ & 28/30 \\
Qwen3.6-35B-A3B & 2-bit & CurveTQ $-$ QTIP$_{+\mathrm{g64}}$ & $+2.89$ & $[+1.31,\,+4.92]$ & 26/30 \\
Qwen3.6-35B-A3B & 2-bit & CurveTQ $-$ Proteus$_{+\mathrm{g64}}$ & $+3.06$ & $[+1.59,\,+4.88]$ & 25/30 \\
Qwen3.6-35B-A3B & 2-bit & CurveTQ $-$ QTIP$_{+\mathrm{start}}$ & $+0.56$ & $[-0.10,\,+1.27]$ & 18/30 \\
Qwen3.6-35B-A3B & 2-bit & CurveTQ $-$ Proteus$_{+\mathrm{start}}$ & $+1.52$ & $[+0.45,\,+2.81]$ & 21/30 \\
\bottomrule
\end{longtable}

%% file: tab_nl_ci.tex
\begin{tabular}{@{}l l r r r r r r@{}}
\toprule
& & \multicolumn{2}{c}{vs.\ CurveTQ (with table)} & \multicolumn{2}{c}{vs.\ QTIP} & \multicolumn{2}{c}{vs.\ Proteus} \\
\cmidrule(lr){3-4}\cmidrule(lr){5-6}\cmidrule(lr){7-8}
Model & Tier & $\Delta$ [95\% CI] & wins/$n$ & $\Delta$ [95\% CI] & wins/$n$ & $\Delta$ [95\% CI] & wins/$n$ \\
\midrule
Qwen3-4B & 2-bit & $-1.18$ [$-2.26$, $-0.21$] & 10/34 & $+3.82$ [$+2.09$, $+5.86$] & 26/34 & $+2.71$ [$+1.41$, $+3.91$] & 27/34 \\
 & 3-bit & $+0.27$ [$-0.28$, $+0.85$] & 17/34 & $+1.33$ [$+0.57$, $+2.15$] & 26/34 & $+1.20$ [$+0.65$, $+1.75$] & 28/34 \\
 & 4-bit & $-0.20$ [$-0.70$, $+0.29$] & 17/34 & $+0.05$ [$-0.42$, $+0.51$] & 20/34 & $-0.16$ [$-0.71$, $+0.31$] & 20/34 \\
\midrule
Llama-3-8B & 2-bit & $-0.55$ [$-1.18$, $+0.15$] & 4/14 & $+3.86$ [$+2.17$, $+6.22$] & 13/14 & $+3.90$ [$+2.30$, $+6.37$] & 14/14 \\
 & 3-bit & $-0.52$ [$-0.83$, $-0.24$] & 2/14 & $+0.76$ [$+0.11$, $+1.44$] & 9/14 & $+0.19$ [$-0.62$, $+1.01$] & 10/14 \\
 & 4-bit & $-0.34$ [$-0.92$, $+0.08$] & 5/14 & $+0.03$ [$-0.38$, $+0.38$] & 8/14 & $-0.02$ [$-0.47$, $+0.43$] & 6/14 \\
\midrule
Llama-3.1-8B & 2-bit & $-2.59$ [$-4.89$, $-0.58$] & 3/14 & $+1.61$ [$+0.69$, $+2.44$] & 11/14 & $-0.10$ [$-1.09$, $+0.87$] & 8/14 \\
 & 3-bit & $+0.74$ [$+0.05$, $+1.57$] & 9/14 & $+1.44$ [$+0.23$, $+2.89$] & 11/14 & $+1.69$ [$+0.81$, $+2.99$] & 13/14 \\
 & 4-bit & $-0.40$ [$-0.92$, $+0.05$] & 4/14 & $+0.50$ [$+0.02$, $+1.06$] & 11/14 & $-0.12$ [$-0.35$, $+0.10$] & 6/14 \\
\bottomrule
\end{tabular}

%% file: tab_werr.tex
\begin{tabular}{@{}l c rrrr@{}}
\toprule
Method & bpw & Qwen3-4B & Llama-3-8B & Llama-3.1-8B & Qwen3-32B \\
\midrule
QTIP & 2.00$^{+}$ & 10.25 & 9.41 & 9.43 & 10.41 \\
Proteus & 2.00$^{+}$ & 10.01 & 9.17 & 9.19 & 10.15 \\
\textbf{CurveTQ} & 2.29 & 8.01 & \textbf{7.21} & \textbf{7.22} & \textbf{8.04} \\
\textbf{CurveTQ}$_{\mathrm{no\,LUT}}$ & 2.29 & 8.35 & 7.60 & 7.62 & -- \\
QTIP$_{+\mathrm{g64}}$ & 2.25$^{+}$ & 10.16 & 9.39 & 9.41 & 10.43 \\
Proteus$_{+\mathrm{g64}}$ & 2.25$^{+}$ & 9.90 & 9.14 & 9.16 & 10.16 \\
QTIP$_{+\mathrm{start}}$ & 2.19$^{+}$ & 8.74 & 8.00 & 8.02 & 8.86 \\
Proteus$_{+\mathrm{start}}$ & 2.22$^{+}$ & 8.10 & 7.39 & 7.41 & 8.19 \\
QTIP$_{+\mathrm{start+g128}}$ & 2.31$^{+}$ & 8.65 & 7.99 & 8.01 & 8.87 \\
Proteus$_{+\mathrm{start+g128}}$ & 2.34$^{+}$ & \textbf{8.00} & 7.38 & 7.39 & 8.19 \\
HIGGS & 2.25 & 17.34 & 15.60 & 15.63 & -- \\
\textbf{CurveTQ}$_{\mathrm{fp16}}$ & 2.47 & 7.97 & 7.18 & 7.19 & -- \\
\midrule
QTIP & 3.00$^{+}$ & 2.77 & 2.50 & 2.51 & 2.78 \\
Proteus & 3.00$^{+}$ & 2.58 & 2.34 & 2.34 & 2.60 \\
\textbf{CurveTQ} & 3.29 & \textbf{1.98} & \textbf{1.77} & \textbf{1.77} & \textbf{1.99} \\
\textbf{CurveTQ}$_{\mathrm{no\,LUT}}$ & 3.29 & 2.17 & 1.95 & 1.95 & -- \\
QTIP$_{+\mathrm{g64}}$ & 3.25$^{+}$ & 2.67 & 2.46 & 2.47 & 2.73 \\
Proteus$_{+\mathrm{g64}}$ & 3.25$^{+}$ & 2.51 & 2.31 & 2.31 & 2.56 \\
QTIP$_{+\mathrm{start}}$ & 3.16$^{+}$ & 2.46 & 2.22 & 2.22 & 2.46 \\
Proteus$_{+\mathrm{start}}$ & 3.20$^{+}$ & 2.15 & 1.94 & 1.94 & 2.15 \\
QTIP$_{+\mathrm{start+g128}}$ & 3.31$^{+}$ & 2.38 & 2.19 & 2.20 & -- \\
Proteus$_{+\mathrm{start+g128}}$ & 3.34$^{+}$ & 2.09 & 1.92 & 1.93 & -- \\
HIGGS & 3.25 & 4.67 & 4.19 & 4.20 & -- \\
NF (Hessian) & 3.25 & 5.98 & 5.52 & 5.52 & -- \\
\textbf{CurveTQ}$_{\mathrm{fp16}}$ & 3.47 & 1.95 & 1.74 & 1.75 & -- \\
\midrule
QTIP & 4.00$^{+}$ & 0.87 & 0.76 & 0.76 & -- \\
Proteus & 4.00$^{+}$ & 0.67 & 0.60 & 0.60 & -- \\
\textbf{CurveTQ} & 4.29 & \textbf{0.52} & \textbf{0.47} & \textbf{0.47} & -- \\
\textbf{CurveTQ}$_{\mathrm{no\,LUT}}$ & 4.29 & 0.55 & 0.49 & 0.49 & -- \\
QTIP$_{+\mathrm{g64}}$ & 4.25$^{+}$ & 0.79 & 0.73 & 0.73 & -- \\
Proteus$_{+\mathrm{g64}}$ & 4.25$^{+}$ & 0.64 & 0.59 & 0.59 & -- \\
QTIP$_{+\mathrm{start}}$ & 4.13$^{+}$ & 0.82 & 0.71 & 0.71 & -- \\
Proteus$_{+\mathrm{start}}$ & 4.19$^{+}$ & 0.57 & 0.51 & 0.51 & -- \\
QTIP$_{+\mathrm{start+g128}}$ & 4.31$^{+}$ & 0.75 & 0.69 & 0.70 & -- \\
Proteus$_{+\mathrm{start+g128}}$ & 4.34$^{+}$ & 0.55 & 0.50 & 0.51 & -- \\
HIGGS & 4.25 & 1.13 & 1.02 & 1.02 & -- \\
NF (Hessian) & 4.25 & 1.26 & 1.17 & 1.17 & -- \\
\textbf{CurveTQ}$_{\mathrm{fp16}}$ & 4.47 & 0.49 & 0.44 & 0.44 & -- \\
\bottomrule
\end{tabular}

%% file: tab_full_l8.tex
\begin{tabular}{@{}lrrrrrrrrrrrrrr@{\hspace{4pt}}r@{}}
\toprule
Method (bpw) & \rotatebox{60}{\,MMLU} & \rotatebox{60}{\,MMLU-Pro} & \rotatebox{60}{\,TruthfulQA} & \rotatebox{60}{\,GPQA} & \rotatebox{60}{\,GSM8K} & \rotatebox{60}{\,ARC-c} & \rotatebox{60}{\,ARC-e} & \rotatebox{60}{\,HellaSwag} & \rotatebox{60}{\,WinoGr.} & \rotatebox{60}{\,PIQA} & \rotatebox{60}{\,BoolQ} & \rotatebox{60}{\,IFEval} & \rotatebox{60}{\,MMLU-ProX} & \rotatebox{60}{\,INCLUDE} & \rotatebox{60}{\,\textbf{Mean}} \\
\midrule
fp16 & 66.7 & 40.8 & 52.4 & 36.4 & 75.4 & 62.5 & 78.8 & 78.8 & 75.5 & 78.9 & 84.5 & 44.7 & 41.0 & 49.3 & \textbf{61.8} \\
\midrule
\multicolumn{16}{@{}l}{\emph{2-bit tier}} \\
\textbf{CurveTQ} 2.29 & \textbf{57.2} & \textbf{27.3} & 47.0 & 32.8 & \textbf{47.5} & \textbf{54.5} & 72.6 & 72.0 & 72.7 & 76.3 & \textbf{82.3} & \textbf{37.5} & \textbf{27.4} & \textbf{38.3} & \textbf{53.2} \\
\textbf{CurveTQ}$_{\mathrm{no\,LUT}}$ 2.29 & 55.4 & 26.5 & 48.2 & \textbf{34.9} & 45.2 & 53.4 & 73.7 & 71.6 & 72.1 & \textbf{76.9} & 80.5 & 35.5 & 26.5 & 37.3 & 52.7 \\
\textbf{CurveTQ}$_{\mathrm{fp16}}$ 2.47 & 56.5 & 26.5 & 48.7 & 28.3 & 46.2 & 55.2 & 72.7 & 71.1 & 70.7 & 76.9 & 83.5 & 35.1 & 26.5 & 37.0 & 52.5 \\
QTIP 2.00$^{+}$ & 51.6 & 21.7 & 47.5 & 31.8 & 28.6 & 48.1 & 70.4 & 67.7 & 69.9 & 75.1 & 82.2 & 32.5 & 22.0 & 34.4 & 48.8 \\
Proteus 2.00$^{+}$ & 52.4 & 23.0 & 48.1 & 30.8 & 26.5 & 48.8 & 71.8 & 69.3 & 70.3 & 75.1 & 78.4 & 32.0 & 23.0 & 33.6 & 48.8 \\
QTIP$_{+\mathrm{g64}}$ 2.25$^{+}$ & 52.8 & 21.4 & 44.9 & 30.3 & 29.2 & 50.9 & 67.9 & 68.2 & 69.8 & 74.2 & 80.7 & 32.5 & 21.4 & 33.8 & 48.4 \\
Proteus$_{+\mathrm{g64}}$ 2.25$^{+}$ & 52.7 & 23.8 & 47.5 & 24.2 & 28.6 & 48.8 & 68.5 & 69.1 & 70.7 & 73.8 & 78.2 & 32.4 & 24.0 & 34.4 & 48.3 \\
HIGGS 2.25 & 41.8 & 14.9 & 47.5 & 23.2 & 8.1 & 39.5 & 60.6 & 55.4 & 64.2 & 71.3 & 73.8 & 19.4 & 14.9 & 26.4 & 40.1 \\
QTIP$_{+\mathrm{start}}$ 2.19$^{+}$ & 56.7 & 25.9 & 46.9 & 33.3 & 37.5 & 52.3 & 70.0 & 70.0 & 71.3 & 76.1 & 80.7 & 31.6 & 26.2 & 36.4 & 51.1 \\
Proteus$_{+\mathrm{start}}$ 2.22$^{+}$ & 56.0 & 26.8 & \textbf{48.8} & 27.8 & 41.8 & 53.6 & \textbf{74.3} & \textbf{72.1} & \textbf{73.7} & 76.5 & 80.5 & 36.0 & 26.6 & 34.5 & 52.1 \\
\midrule
\multicolumn{16}{@{}l}{\emph{3-bit tier}} \\
\textbf{CurveTQ} 3.29 & 64.2 & 37.2 & 51.4 & 33.8 & \textbf{71.5} & 61.1 & \textbf{78.1} & \textbf{77.3} & \textbf{75.8} & 77.5 & \textbf{84.8} & 41.8 & 37.4 & \textbf{46.7} & \textbf{59.9} \\
\textbf{CurveTQ}$_{\mathrm{no\,LUT}}$ 3.29 & 64.6 & 37.2 & 50.2 & 32.3 & 70.9 & 61.0 & 77.2 & 77.2 & 74.9 & 77.4 & 83.3 & 41.4 & 37.4 & 46.4 & 59.4 \\
\textbf{CurveTQ}$_{\mathrm{fp16}}$ 3.47 & 65.2 & 38.0 & 52.0 & 34.9 & 73.2 & 59.8 & 76.8 & 77.4 & 74.9 & 78.0 & 85.0 & 40.9 & 38.3 & 47.0 & 60.1 \\
QTIP 3.00$^{+}$ & 63.4 & 35.9 & 51.0 & 32.8 & 67.3 & 59.3 & 77.5 & 76.5 & 74.7 & \textbf{78.0} & 84.2 & 39.6 & 36.2 & 44.2 & 58.6 \\
Proteus 3.00$^{+}$ & 64.4 & 37.1 & 52.1 & 28.3 & 69.8 & 59.0 & 76.8 & 77.0 & 75.3 & 77.2 & 84.0 & 44.4 & 37.1 & 46.1 & 59.2 \\
QTIP$_{+\mathrm{g64}}$ 3.25$^{+}$ & 64.1 & \textbf{37.8} & 51.2 & 32.8 & 69.1 & 60.6 & 76.7 & 76.4 & 74.6 & 77.0 & 83.7 & \textbf{44.9} & \textbf{38.2} & 45.3 & 59.5 \\
Proteus$_{+\mathrm{g64}}$ 3.25$^{+}$ & 64.0 & 36.3 & 51.5 & 35.9 & 68.3 & 60.3 & 75.6 & 76.7 & 75.3 & \textbf{78.0} & 83.5 & 39.0 & 36.4 & 46.2 & 59.1 \\
HIGGS 3.25 & 61.6 & 34.0 & 50.3 & \textbf{37.4} & 61.3 & 58.1 & 75.0 & 75.4 & 74.3 & 76.9 & 83.2 & 42.0 & 34.2 & 42.9 & 57.6 \\
GPTQ 3.25 & 56.9 & 26.9 & 48.3 & 34.9 & 41.5 & 55.2 & 69.6 & 73.4 & 72.0 & 75.1 & 80.1 & 35.5 & 27.5 & 35.1 & 52.3 \\
QTIP$_{+\mathrm{start}}$ 3.16$^{+}$ & 64.0 & 36.6 & 50.6 & 31.3 & 68.8 & 60.8 & 77.6 & 77.0 & \textbf{75.8} & 77.3 & 83.8 & 43.1 & 36.5 & 46.1 & 59.2 \\
Proteus$_{+\mathrm{start}}$ 3.20$^{+}$ & \textbf{64.8} & 37.3 & \textbf{53.0} & 33.8 & 69.2 & \textbf{61.2} & 77.0 & 77.2 & 75.1 & 77.4 & 84.1 & 42.7 & 37.3 & 46.0 & 59.7 \\
\midrule
\multicolumn{16}{@{}l}{\emph{4-bit tier}} \\
\textbf{CurveTQ} 4.29 & 66.1 & 40.0 & 52.2 & 36.4 & 75.2 & 62.0 & 77.9 & \textbf{78.5} & 75.7 & \textbf{78.8} & 84.3 & 45.3 & 40.2 & \textbf{48.6} & \textbf{61.5} \\
\textbf{CurveTQ}$_{\mathrm{no\,LUT}}$ 4.29 & 65.9 & 40.0 & 52.4 & 32.8 & 75.3 & 62.4 & 78.9 & 78.2 & 75.2 & 78.3 & 84.3 & 44.5 & 39.7 & \textbf{48.6} & 61.2 \\
\textbf{CurveTQ}$_{\mathrm{fp16}}$ 4.47 & 66.1 & 40.6 & 52.1 & 35.9 & 74.7 & 62.0 & 79.0 & 78.7 & 75.6 & 78.9 & 84.8 & 43.4 & 40.7 & 48.7 & 61.5 \\
QTIP 4.00$^{+}$ & 66.0 & 39.8 & 51.6 & 34.9 & 75.1 & 62.3 & 77.7 & 78.3 & 74.7 & 78.4 & 84.9 & 44.0 & 39.9 & 48.4 & 61.1 \\
Proteus 4.00$^{+}$ & \textbf{66.5} & \textbf{40.1} & 51.6 & 32.3 & 73.6 & 62.0 & 77.9 & 78.3 & \textbf{77.0} & 78.5 & 84.6 & \textbf{45.7} & 40.1 & 48.5 & 61.2 \\
QTIP$_{+\mathrm{g64}}$ 4.25$^{+}$ & 65.5 & 39.5 & 52.4 & 34.3 & 73.3 & 62.1 & \textbf{79.5} & 77.9 & 75.2 & 77.8 & 84.0 & 43.2 & 39.3 & 48.0 & 60.9 \\
Proteus$_{+\mathrm{g64}}$ 4.25$^{+}$ & 66.1 & 40.0 & 51.8 & 31.3 & 74.0 & \textbf{62.6} & 78.1 & \textbf{78.5} & 75.8 & 78.5 & 84.7 & 44.2 & \textbf{40.4} & 48.5 & 61.0 \\
HIGGS 4.25 & 65.6 & 39.4 & 51.5 & 33.8 & 73.1 & 60.7 & 78.5 & 77.8 & 75.2 & 77.6 & 83.9 & 43.8 & 39.6 & 47.8 & 60.6 \\
NF 4.25 & 65.2 & 39.1 & 51.3 & 32.3 & 71.6 & 60.2 & 77.8 & 77.8 & 75.8 & 76.5 & 83.4 & 42.0 & 39.3 & 46.7 & 59.9 \\
GPTQ 4.25 & 64.0 & 37.0 & 51.0 & 31.3 & 66.4 & 61.9 & 77.9 & 76.7 & 75.4 & 78.3 & 84.0 & 42.9 & 37.3 & 45.6 & 59.3 \\
QTIP$_{+\mathrm{start}}$ 4.13$^{+}$ & 66.1 & 39.7 & \textbf{53.4} & 34.3 & \textbf{75.4} & 62.0 & 78.2 & 78.0 & 75.1 & 78.3 & \textbf{85.1} & 44.5 & 39.9 & 48.3 & 61.3 \\
Proteus$_{+\mathrm{start}}$ 4.19$^{+}$ & 66.0 & \textbf{40.1} & 51.9 & \textbf{37.9} & 74.2 & 61.8 & 76.8 & 78.4 & 76.4 & \textbf{78.8} & 84.5 & 44.4 & 40.2 & 48.5 & 61.4 \\
\bottomrule
\end{tabular}

%% file: tab_full_l31.tex
\begin{tabular}{@{}lrrrrrrrrrrrrrr@{\hspace{4pt}}r@{}}
\toprule
Method (bpw) & \rotatebox{60}{\,MMLU} & \rotatebox{60}{\,MMLU-Pro} & \rotatebox{60}{\,TruthfulQA} & \rotatebox{60}{\,GPQA} & \rotatebox{60}{\,GSM8K} & \rotatebox{60}{\,ARC-c} & \rotatebox{60}{\,ARC-e} & \rotatebox{60}{\,HellaSwag} & \rotatebox{60}{\,WinoGr.} & \rotatebox{60}{\,PIQA} & \rotatebox{60}{\,BoolQ} & \rotatebox{60}{\,IFEval} & \rotatebox{60}{\,MMLU-ProX} & \rotatebox{60}{\,INCLUDE} & \rotatebox{60}{\,\textbf{Mean}} \\
\midrule
fp16 & 68.3 & 43.3 & 54.5 & 33.3 & 77.9 & 61.4 & 79.6 & 80.4 & 78.6 & 81.3 & 85.3 & 45.8 & 43.1 & 52.4 & \textbf{63.2} \\
\midrule
\multicolumn{16}{@{}l}{\emph{2-bit tier}} \\
\textbf{CurveTQ} 2.29 & \textbf{59.9} & \textbf{26.7} & \textbf{50.9} & 27.3 & \textbf{49.7} & 50.9 & 73.2 & \textbf{73.5} & \textbf{74.9} & 77.9 & \textbf{82.7} & \textbf{41.8} & 26.8 & \textbf{39.4} & \textbf{54.0} \\
\textbf{CurveTQ}$_{\mathrm{no\,LUT}}$ 2.29 & 57.3 & 24.5 & 43.4 & 29.8 & 36.6 & \textbf{53.8} & 73.6 & 72.6 & 72.9 & 77.8 & 82.0 & 34.4 & 24.9 & 35.6 & 51.4 \\
\textbf{CurveTQ}$_{\mathrm{fp16}}$ 2.47 & 58.4 & 25.9 & 48.8 & 27.3 & 47.3 & 52.4 & 73.7 & 72.3 & 75.3 & 78.7 & 81.9 & 41.4 & 26.0 & 39.4 & 53.5 \\
QTIP 2.00$^{+}$ & 54.3 & 21.4 & 44.4 & 26.3 & 35.1 & 51.8 & \textbf{75.8} & 69.8 & 71.6 & 78.1 & 80.8 & 32.5 & 21.7 & 33.2 & 49.8 \\
Proteus 2.00$^{+}$ & 54.5 & 24.4 & 47.0 & \textbf{33.3} & 33.9 & 52.7 & 73.2 & 71.3 & 74.0 & \textbf{78.8} & 82.5 & 35.9 & 24.3 & 34.7 & 51.5 \\
QTIP$_{+\mathrm{g64}}$ 2.25$^{+}$ & 53.9 & 21.1 & 39.9 & 27.3 & 30.2 & 51.4 & 72.6 & 69.4 & 71.3 & 76.7 & 81.7 & 33.6 & 21.6 & 32.7 & 48.8 \\
Proteus$_{+\mathrm{g64}}$ 2.25$^{+}$ & 54.5 & 22.1 & 47.8 & 31.3 & 38.1 & 51.6 & 73.4 & 70.7 & 73.2 & 77.0 & 81.7 & 35.1 & 21.9 & 34.9 & 51.0 \\
HIGGS 2.25 & 39.6 & 11.9 & 41.9 & 27.3 & 11.1 & 37.2 & 61.1 & 56.4 & 64.2 & 72.5 & 74.9 & 20.7 & 12.0 & 27.1 & 39.8 \\
QTIP$_{+\mathrm{start}}$ 2.19$^{+}$ & 56.9 & 24.6 & 42.4 & 25.2 & 36.5 & 52.3 & 74.0 & 71.0 & 73.2 & 78.3 & 78.0 & 38.3 & 24.8 & 36.8 & 50.9 \\
Proteus$_{+\mathrm{start}}$ 2.22$^{+}$ & 58.7 & 26.4 & 47.4 & 28.8 & 47.4 & 53.2 & 70.8 & 72.9 & 73.2 & 77.9 & 81.6 & 33.8 & \textbf{27.1} & 36.3 & 52.5 \\
\midrule
\multicolumn{16}{@{}l}{\emph{3-bit tier}} \\
\textbf{CurveTQ} 3.29 & 66.4 & 38.6 & 51.5 & 34.9 & 71.7 & 58.6 & 76.5 & \textbf{78.9} & 77.7 & 80.4 & 85.1 & 44.2 & 38.7 & \textbf{49.8} & 60.9 \\
\textbf{CurveTQ}$_{\mathrm{no\,LUT}}$ 3.29 & \textbf{66.6} & 38.5 & \textbf{54.0} & \textbf{39.4} & \textbf{72.8} & 59.8 & 79.2 & \textbf{78.9} & \textbf{78.0} & 80.5 & 84.2 & 43.6 & 38.5 & 49.2 & \textbf{61.7} \\
\textbf{CurveTQ}$_{\mathrm{fp16}}$ 3.47 & 66.5 & 39.6 & 53.9 & 34.3 & 72.4 & 60.0 & 79.3 & 78.9 & 77.1 & 80.3 & 85.6 & 43.4 & 40.0 & 49.3 & 61.5 \\
QTIP 3.00$^{+}$ & 66.0 & 38.4 & 50.4 & 30.8 & 70.4 & 59.4 & 75.8 & 78.2 & 75.7 & 80.2 & 84.7 & 46.4 & 38.6 & 48.3 & 60.2 \\
Proteus 3.00$^{+}$ & 65.6 & 37.0 & 51.6 & 30.3 & 70.2 & 59.5 & 78.2 & 78.6 & 76.2 & 79.8 & 84.5 & 42.5 & 37.8 & 47.9 & 60.0 \\
QTIP$_{+\mathrm{g64}}$ 3.25$^{+}$ & 66.0 & 38.3 & 53.5 & 32.3 & 70.1 & 59.1 & 78.6 & 78.7 & 77.0 & 80.4 & 84.4 & 46.4 & 38.4 & 48.8 & 60.9 \\
Proteus$_{+\mathrm{g64}}$ 3.25$^{+}$ & 65.2 & 38.5 & 53.3 & 36.4 & 72.5 & 58.8 & 78.2 & 78.6 & 77.0 & \textbf{81.0} & \textbf{85.3} & 45.5 & \textbf{39.1} & 48.9 & 61.3 \\
HIGGS 3.25 & 64.3 & 34.4 & 52.0 & 28.8 & 67.5 & 58.5 & 76.5 & 76.5 & 76.0 & 79.0 & 83.6 & 45.7 & 34.2 & 44.9 & 58.7 \\
GPTQ 3.25 & 55.7 & 24.0 & 47.2 & 30.3 & 40.2 & 54.0 & 75.2 & 73.5 & 72.8 & 78.1 & 79.2 & 37.3 & 24.3 & 36.0 & 52.0 \\
QTIP$_{+\mathrm{start}}$ 3.16$^{+}$ & 66.2 & 38.4 & 52.0 & 33.8 & 71.8 & 59.4 & \textbf{79.5} & 78.5 & 77.6 & 80.0 & 84.7 & \textbf{47.9} & 38.4 & 48.7 & 61.2 \\
Proteus$_{+\mathrm{start}}$ 3.20$^{+}$ & 65.7 & \textbf{38.8} & 53.9 & 35.9 & 70.7 & \textbf{61.2} & 78.8 & 78.7 & 77.3 & 80.7 & 84.8 & 43.6 & 38.6 & 48.3 & 61.2 \\
\midrule
\multicolumn{16}{@{}l}{\emph{4-bit tier}} \\
\textbf{CurveTQ} 4.29 & \textbf{68.0} & \textbf{42.1} & 53.7 & \textbf{37.4} & \textbf{78.4} & 60.8 & 78.5 & \textbf{80.2} & 78.1 & \textbf{81.5} & 85.1 & 44.9 & \textbf{42.0} & 51.7 & \textbf{63.0} \\
\textbf{CurveTQ}$_{\mathrm{no\,LUT}}$ 4.29 & 67.8 & 41.4 & 54.3 & 34.9 & 76.1 & 60.5 & 79.4 & 80.1 & 78.5 & 80.6 & 84.9 & 45.3 & 41.5 & 51.5 & 62.6 \\
\textbf{CurveTQ}$_{\mathrm{fp16}}$ 4.47 & 67.6 & 41.7 & 54.2 & 31.3 & 76.7 & 60.8 & 79.1 & 80.3 & 77.5 & 81.3 & 85.0 & 45.3 & 41.8 & 51.8 & 62.5 \\
QTIP 4.00$^{+}$ & 67.7 & 41.1 & 54.2 & 32.3 & 77.2 & 60.0 & 78.7 & 79.7 & 77.7 & 80.5 & 85.3 & 42.5 & 41.8 & 51.0 & 62.1 \\
Proteus 4.00$^{+}$ & 67.9 & 41.9 & 54.0 & 34.3 & 77.0 & 61.3 & 79.2 & 79.9 & \textbf{78.7} & 81.1 & 84.5 & 45.7 & 41.8 & 51.3 & 62.7 \\
QTIP$_{+\mathrm{g64}}$ 4.25$^{+}$ & 67.4 & 41.3 & 53.7 & 34.9 & 76.3 & 61.1 & 77.7 & 79.7 & 78.0 & 81.2 & \textbf{85.8} & 46.8 & 41.6 & 51.7 & 62.7 \\
Proteus$_{+\mathrm{g64}}$ 4.25$^{+}$ & 67.8 & 41.9 & \textbf{54.5} & 33.8 & 76.7 & 60.3 & 78.6 & 79.9 & 77.7 & 81.0 & 85.0 & 47.1 & \textbf{42.0} & 51.4 & 62.7 \\
HIGGS 4.25 & 67.4 & 41.8 & 53.9 & 34.9 & 76.3 & 60.4 & \textbf{80.3} & 79.7 & 78.0 & 80.6 & 85.1 & 47.1 & 41.5 & 50.8 & 62.7 \\
NF 4.25 & 66.7 & 41.0 & 53.9 & 35.4 & 76.7 & 60.8 & 78.5 & 79.6 & 76.8 & 81.0 & 85.4 & 43.8 & 40.9 & 49.5 & 62.1 \\
GPTQ 4.25 & 65.4 & 37.6 & 53.0 & 36.9 & 71.2 & 59.3 & 79.8 & 79.0 & 77.9 & 80.2 & 85.1 & \textbf{50.1} & 37.9 & 48.0 & 61.5 \\
QTIP$_{+\mathrm{start}}$ 4.13$^{+}$ & \textbf{68.0} & 41.4 & 53.6 & 35.4 & 75.8 & \textbf{61.4} & 79.8 & 79.7 & 77.5 & 80.9 & 85.3 & 44.0 & 41.6 & 51.0 & 62.5 \\
Proteus$_{+\mathrm{start}}$ 4.19$^{+}$ & 67.8 & 41.3 & 53.9 & 35.4 & 77.0 & 61.3 & 79.7 & 80.1 & 78.4 & 81.0 & 85.4 & 47.1 & 41.9 & \textbf{51.9} & \textbf{63.0} \\
\bottomrule
\end{tabular}

%% file: tab_full_q4_text.tex
\begin{tabular}{@{}lrrrrrrrrrrrrrrrrrr@{\hspace{4pt}}r@{}}
\toprule
Method (bpw) & \rotatebox{60}{\,MMLU} & \rotatebox{60}{\,MMLU-Pro} & \rotatebox{60}{\,TruthfulQA} & \rotatebox{60}{\,GPQA} & \rotatebox{60}{\,LB-Lang} & \rotatebox{60}{\,GSM8K} & \rotatebox{60}{\,ARC-c} & \rotatebox{60}{\,ARC-e} & \rotatebox{60}{\,HellaSwag} & \rotatebox{60}{\,WinoGr.} & \rotatebox{60}{\,PIQA} & \rotatebox{60}{\,BoolQ} & \rotatebox{60}{\,LB-Reason} & \rotatebox{60}{\,LB-Math} & \rotatebox{60}{\,LB-DA} & \rotatebox{60}{\,IFEval} & \rotatebox{60}{\,MMLU-ProX} & \rotatebox{60}{\,INCLUDE} & \rotatebox{60}{\,\textbf{Mean}} \\
\midrule
fp16 & 64.8 & 56.6 & 59.8 & 35.9 & 34.1 & 87.4 & 65.3 & 78.5 & 69.6 & 66.2 & 77.0 & 86.1 & 55.9 & 41.0 & 52.8 & 54.3 & 56.9 & 54.6 & \textbf{60.9} \\
\midrule
\multicolumn{20}{@{}l}{\emph{2-bit tier}} \\
\textbf{CurveTQ} 2.29 & \textbf{63.9} & \textbf{35.0} & 55.6 & 25.8 & \textbf{23.6} & \textbf{68.7} & 56.3 & 70.6 & \textbf{62.6} & 63.4 & \textbf{74.5} & \textbf{86.3} & 12.6 & \textbf{9.9} & 28.3 & \textbf{43.1} & \textbf{35.3} & 42.1 & \textbf{47.6} \\
\textbf{CurveTQ}$_{\mathrm{no\,LUT}}$ 2.29 & 62.1 & 32.0 & 53.9 & 30.3 & 21.5 & 63.2 & 54.1 & 65.3 & 60.4 & 64.0 & 73.5 & 84.0 & 6.7 & 6.1 & \textbf{31.2} & 32.0 & 31.8 & 40.6 & 45.1 \\
\textbf{CurveTQ}$_{\mathrm{fp16}}$ 2.47 & 63.8 & 34.2 & 57.2 & 29.8 & 21.2 & 69.5 & 59.3 & 73.8 & 62.9 & 64.2 & 74.8 & 85.2 & 11.5 & 10.3 & 27.0 & 41.0 & 34.5 & 42.2 & 47.9 \\
QTIP 2.00$^{+}$ & 59.4 & 25.9 & 54.1 & 27.8 & 16.3 & 39.5 & 52.2 & 69.3 & 58.9 & 62.5 & 73.0 & 84.2 & 10.8 & 7.0 & 15.2 & 34.9 & 25.9 & 39.8 & 42.0 \\
Proteus 2.00$^{+}$ & 60.4 & 25.4 & \textbf{56.3} & 27.8 & 15.0 & 55.8 & 52.4 & 66.2 & 59.0 & 62.9 & 73.9 & 85.2 & \textbf{16.2} & 5.0 & 22.9 & 33.1 & 25.6 & 36.2 & 43.3 \\
QTIP$_{+\mathrm{g64}}$ 2.25$^{+}$ & 59.4 & 25.6 & 47.9 & 28.3 & 16.3 & 47.5 & 53.5 & 62.3 & 58.8 & 63.7 & 72.7 & 83.5 & 11.3 & 6.7 & 22.2 & 30.9 & 25.8 & 37.6 & 41.9 \\
Proteus$_{+\mathrm{g64}}$ 2.25$^{+}$ & 59.9 & 29.0 & 53.6 & 27.3 & 18.5 & 30.7 & 54.2 & 67.4 & 61.0 & 64.7 & 72.5 & 83.6 & 5.9 & 5.0 & 20.4 & 35.7 & 29.2 & 40.9 & 42.2 \\
HIGGS 2.25 & 46.6 & 15.8 & 47.2 & 25.8 & 9.1 & 21.2 & 42.7 & 59.1 & 49.0 & 59.3 & 69.1 & 81.8 & 10.1 & 3.7 & 8.7 & 24.2 & 16.1 & 27.5 & 34.3 \\
GPTQ 2.25 & 24.0 & 2.7 & 49.9 & 23.7 & 1.5 & 1.5 & 23.6 & 30.8 & 37.4 & 52.7 & 55.6 & 42.9 & 0.0 & 1.4 & 0.0 & 10.2 & 2.8 & 25.6 & 21.5 \\
QTIP$_{+\mathrm{start}}$ 2.19$^{+}$ & 62.8 & 28.3 & 54.9 & 28.3 & 17.3 & 58.5 & \textbf{57.2} & 69.8 & 60.4 & \textbf{66.1} & 73.3 & 86.1 & 9.3 & 8.2 & 24.4 & 40.5 & 28.8 & 40.7 & 45.3 \\
Proteus$_{+\mathrm{start}}$ 2.22$^{+}$ & 63.2 & 30.8 & 54.1 & \textbf{33.3} & 21.7 & 64.2 & 54.8 & \textbf{75.0} & 60.8 & 65.1 & 74.4 & 84.3 & 7.5 & 9.1 & 24.2 & 34.2 & 30.5 & \textbf{42.6} & 46.1 \\
\midrule
\multicolumn{20}{@{}l}{\emph{3-bit tier}} \\
\textbf{CurveTQ} 3.29 & \textbf{70.8} & \textbf{52.2} & 58.2 & 32.8 & 29.1 & 85.8 & 62.5 & 77.4 & 67.7 & 65.7 & 76.7 & 86.7 & 43.0 & \textbf{38.0} & 48.4 & 54.5 & \textbf{53.0} & \textbf{51.9} & \textbf{58.6} \\
\textbf{CurveTQ}$_{\mathrm{no\,LUT}}$ 3.29 & 70.5 & 52.0 & 57.5 & \textbf{34.9} & 30.9 & 85.9 & 62.6 & 77.6 & 68.0 & \textbf{67.6} & 76.0 & 86.6 & 42.0 & 35.1 & 46.1 & 51.8 & 52.2 & 51.7 & 58.3 \\
\textbf{CurveTQ}$_{\mathrm{fp16}}$ 3.47 & 71.2 & 51.9 & 57.9 & 35.4 & 31.5 & 84.6 & 63.5 & 77.6 & 67.6 & 65.9 & 76.2 & 86.5 & 48.2 & 37.9 & 45.9 & 53.2 & 52.0 & 52.4 & 58.9 \\
QTIP 3.00$^{+}$ & 70.0 & 48.8 & 56.9 & 30.3 & 29.0 & 83.1 & 62.5 & 79.0 & 67.1 & 67.0 & 76.0 & 86.2 & 40.7 & 30.6 & 41.4 & 55.5 & 48.7 & 50.1 & 56.8 \\
Proteus 3.00$^{+}$ & 69.5 & 48.9 & 57.6 & \textbf{34.9} & 29.2 & 85.1 & 61.7 & 77.2 & 66.8 & 65.2 & 75.9 & 86.1 & 41.3 & 35.0 & 50.2 & 50.1 & 49.7 & 50.9 & 57.5 \\
QTIP$_{+\mathrm{g64}}$ 3.25$^{+}$ & 69.7 & 48.9 & 58.7 & 33.8 & 28.2 & 84.0 & \textbf{63.5} & 78.0 & 67.0 & 65.8 & \textbf{77.0} & 83.6 & 32.9 & 30.4 & 47.7 & 53.4 & 49.3 & 51.0 & 56.8 \\
Proteus$_{+\mathrm{g64}}$ 3.25$^{+}$ & 70.0 & 51.1 & 59.1 & 31.3 & \textbf{32.9} & 78.7 & 62.9 & \textbf{79.7} & \textbf{68.6} & 66.1 & 76.4 & 85.7 & 41.4 & 37.4 & 46.1 & 53.2 & 51.1 & \textbf{51.9} & 58.0 \\
HIGGS 3.25 & 68.5 & 0.5 & 55.8 & 31.8 & 30.8 & 80.4 & 60.9 & 74.5 & 65.3 & 66.7 & 75.1 & 84.9 & 30.0 & 24.5 & 38.0 & 46.0 & 45.4 & 48.4 & 51.5 \\
NF 3.25 & 66.3 & 43.4 & 57.3 & 30.8 & 22.8 & 74.5 & 60.2 & 77.3 & 66.5 & 66.4 & \textbf{77.0} & 84.8 & 22.1 & 24.3 & 36.4 & 46.6 & 44.0 & 46.9 & 52.6 \\
GPTQ 3.25 & 63.1 & 32.2 & 55.1 & 27.8 & 22.6 & 61.1 & 56.7 & 70.0 & 64.3 & 66.5 & 74.6 & 84.7 & 7.3 & 5.2 & 27.2 & 40.9 & 32.5 & 40.9 & 46.2 \\
QTIP$_{+\mathrm{start}}$ 3.16$^{+}$ & 70.5 & 51.2 & 58.0 & 30.8 & 29.9 & 83.5 & 63.1 & 75.3 & 67.0 & 67.3 & 75.3 & \textbf{86.9} & 47.4 & 33.9 & \textbf{51.5} & 53.4 & 51.5 & 51.0 & 58.2 \\
Proteus$_{+\mathrm{start}}$ 3.20$^{+}$ & 69.8 & 51.5 & \textbf{60.0} & 29.8 & 32.1 & \textbf{86.5} & 63.0 & 74.5 & 67.8 & 65.8 & 76.3 & 85.4 & \textbf{49.3} & 33.2 & 47.1 & \textbf{56.4} & 52.0 & \textbf{51.9} & 58.5 \\
\midrule
\multicolumn{20}{@{}l}{\emph{4-bit tier}} \\
\textbf{CurveTQ} 4.29 & 72.1 & 55.3 & 59.5 & 32.8 & 30.8 & 87.6 & 65.6 & 78.9 & 69.3 & 67.1 & 77.0 & 86.6 & 51.6 & 40.1 & 53.2 & 53.8 & 55.3 & 53.6 & 60.6 \\
\textbf{CurveTQ}$_{\mathrm{no\,LUT}}$ 4.29 & 72.2 & \textbf{55.4} & 59.6 & 36.4 & 33.5 & 88.0 & 65.2 & 79.2 & 68.8 & \textbf{67.2} & 77.3 & 86.4 & 52.2 & 38.8 & 51.7 & 55.3 & \textbf{55.9} & 53.8 & 60.9 \\
\textbf{CurveTQ}$_{\mathrm{fp16}}$ 4.47 & 72.3 & 55.8 & 59.9 & 35.4 & 32.2 & 88.9 & 64.8 & 78.7 & 69.3 & 66.8 & 76.9 & 86.8 & 53.7 & 41.1 & 51.9 & 54.0 & 56.2 & 53.5 & 61.0 \\
QTIP 4.00$^{+}$ & 71.8 & 54.6 & 59.6 & 34.9 & 34.6 & 87.4 & 64.8 & 79.1 & 69.1 & 66.1 & 76.9 & 86.4 & 52.9 & 41.1 & 52.9 & 55.1 & 54.8 & 53.5 & 60.9 \\
Proteus 4.00$^{+}$ & 72.0 & 55.2 & 59.2 & \textbf{36.9} & 33.8 & 87.9 & 65.4 & 78.5 & 69.0 & 65.8 & 76.7 & 86.3 & 50.8 & 41.8 & 51.1 & 52.5 & 55.7 & 53.3 & 60.6 \\
QTIP$_{+\mathrm{g64}}$ 4.25$^{+}$ & 71.9 & 54.9 & 59.9 & 31.8 & 33.6 & 87.5 & 64.5 & 79.1 & 68.9 & 66.8 & 76.7 & 86.3 & 48.2 & 37.1 & 52.2 & 54.9 & 55.2 & 53.3 & 60.2 \\
Proteus$_{+\mathrm{g64}}$ 4.25$^{+}$ & 72.1 & 54.9 & 59.5 & 33.3 & 33.6 & 88.0 & 64.5 & 78.6 & \textbf{69.4} & 67.1 & 76.7 & 86.8 & 52.9 & 42.4 & 53.0 & 52.7 & 55.3 & 53.5 & 60.8 \\
HIGGS 4.25 & 71.5 & 54.6 & \textbf{60.1} & 35.9 & \textbf{36.0} & 78.5 & 64.1 & 76.7 & 68.4 & 67.0 & 76.8 & 86.2 & \textbf{53.6} & 37.4 & 53.2 & 53.8 & 55.1 & 53.3 & 60.1 \\
NF 4.25 & 71.4 & 53.9 & 59.7 & 36.4 & 31.0 & 83.0 & 65.6 & 78.9 & 69.2 & 66.1 & \textbf{77.5} & 85.6 & 52.0 & 37.2 & \textbf{54.2} & \textbf{55.8} & 54.2 & 52.7 & 60.2 \\
GPTQ 4.25 & 70.6 & 51.7 & 58.5 & 31.8 & 27.8 & 85.0 & 62.3 & 78.7 & 67.7 & 66.3 & 76.8 & \textbf{87.1} & 45.4 & 38.3 & 49.0 & 47.5 & 51.7 & 51.6 & 58.2 \\
QTIP$_{+\mathrm{start}}$ 4.13$^{+}$ & 71.6 & 53.9 & 58.8 & 33.3 & 33.8 & 87.0 & \textbf{65.7} & \textbf{80.2} & 68.6 & 66.7 & 77.1 & 86.5 & 52.2 & 37.3 & 53.5 & 52.9 & 54.1 & 53.3 & 60.4 \\
Proteus$_{+\mathrm{start}}$ 4.19$^{+}$ & \textbf{72.6} & 55.0 & 59.4 & 33.3 & 35.6 & \textbf{88.1} & 64.5 & 77.8 & 69.1 & 67.1 & 76.7 & 86.4 & 50.2 & \textbf{44.1} & 52.7 & 55.6 & 55.2 & \textbf{53.9} & \textbf{61.0} \\
\bottomrule
\end{tabular}

%% file: tab_full_q4_vl.tex
\begin{tabular}{@{}lrrrrrrrrrrrrrrrr@{\hspace{4pt}}r@{}}
\toprule
Method (bpw) & \rotatebox{60}{\,MMMU} & \rotatebox{60}{\,MMMU-Pro} & \rotatebox{60}{\,MMStar} & \rotatebox{60}{\,RealWorldQA} & \rotatebox{60}{\,SimpleVQA} & \rotatebox{60}{\,AI2D} & \rotatebox{60}{\,DocVQA} & \rotatebox{60}{\,InfoVQA} & \rotatebox{60}{\,MMLongB-Doc} & \rotatebox{60}{\,BLINK} & \rotatebox{60}{\,MuirBench} & \rotatebox{60}{\,VSI-Bench} & \rotatebox{60}{\,HRB-4K} & \rotatebox{60}{\,HRB-8K} & \rotatebox{60}{\,V$^*$-attr} & \rotatebox{60}{\,V$^*$-pos} & \rotatebox{60}{\,\textbf{Mean}} \\
\midrule
fp16 & 53.4 & 37.6 & 63.0 & 71.8 & 26.1 & 81.9 & 94.8 & 79.7 & 22.1 & 64.4 & 52.8 & 53.6 & 75.7 & 72.8 & 85.7 & 76.3 & \textbf{63.2} \\
\midrule
\multicolumn{18}{@{}l}{\emph{2-bit tier}} \\
\textbf{CurveTQ} 2.29 & 44.9 & 28.7 & 54.5 & 66.8 & 13.0 & \textbf{75.2} & \textbf{93.0} & \textbf{72.0} & 22.7 & \textbf{57.8} & \textbf{47.5} & 43.2 & 74.1 & 67.9 & 76.5 & 67.1 & 56.6 \\
\textbf{CurveTQ}$_{\mathrm{no\,LUT}}$ 2.29 & 44.0 & \textbf{32.3} & 55.1 & 66.9 & 12.8 & 74.4 & 92.2 & 71.1 & 21.1 & 56.0 & 46.6 & \textbf{44.9} & \textbf{74.9} & \textbf{69.5} & 80.9 & 67.1 & \textbf{56.9} \\
\textbf{CurveTQ}$_{\mathrm{fp16}}$ 2.47 & 44.7 & 32.7 & 55.2 & 66.7 & 11.7 & 74.4 & 92.9 & 71.5 & 20.5 & 60.3 & 50.1 & 39.0 & 74.0 & 70.1 & 82.6 & 65.8 & 57.0 \\
QTIP 2.00$^{+}$ & 41.8 & 26.4 & 51.8 & 62.9 & 10.3 & 70.9 & 88.8 & 64.5 & 22.8 & 55.6 & 37.9 & 32.7 & 68.4 & 62.8 & 71.3 & 67.1 & 52.2 \\
Proteus 2.00$^{+}$ & 42.2 & 25.8 & 51.6 & 64.0 & 10.7 & 70.7 & 89.3 & 65.0 & 22.9 & 55.2 & 42.4 & 34.4 & 71.1 & 66.5 & 77.4 & 61.8 & 53.2 \\
QTIP$_{+\mathrm{g64}}$ 2.25$^{+}$ & 44.7 & 28.3 & 52.8 & 64.8 & 13.3 & 72.5 & 90.6 & 68.7 & 22.9 & 57.1 & 37.7 & 38.7 & 72.0 & 66.6 & 75.7 & 67.1 & 54.6 \\
Proteus$_{+\mathrm{g64}}$ 2.25$^{+}$ & 43.6 & 28.8 & 51.8 & 63.5 & 12.9 & 70.1 & 90.5 & 69.0 & 18.4 & 49.2 & 38.4 & 36.7 & 69.2 & 63.9 & 75.7 & 67.1 & 53.0 \\
HIGGS 2.25 & 36.8 & 21.2 & 41.0 & 56.3 & 2.2 & 55.3 & 66.5 & 45.5 & 0.8 & 39.2 & 14.2 & 23.4 & 63.2 & 54.4 & 56.5 & 56.6 & 39.6 \\
GPTQ 2.25 & 25.7 & 14.2 & 21.8 & 14.2 & 0.0 & 17.6 & 10.9 & 15.9 & 1.0 & 6.9 & 1.3 & 0.8 & 7.9 & 10.2 & 17.4 & 28.9 & 12.2 \\
QTIP$_{+\mathrm{start}}$ 2.19$^{+}$ & \textbf{45.0} & 26.4 & 52.6 & 64.8 & 12.8 & 73.6 & 91.4 & 68.9 & \textbf{23.2} & 56.2 & 41.7 & 37.2 & 71.0 & 66.8 & 78.3 & 68.4 & 54.9 \\
Proteus$_{+\mathrm{start}}$ 2.22$^{+}$ & 43.4 & 28.8 & \textbf{55.6} & \textbf{68.0} & \textbf{13.5} & 74.3 & 91.9 & 71.3 & 22.1 & 55.8 & 42.0 & 42.2 & 73.6 & 67.9 & \textbf{81.7} & \textbf{71.0} & 56.5 \\
\midrule
\multicolumn{18}{@{}l}{\emph{3-bit tier}} \\
\textbf{CurveTQ} 3.29 & 51.6 & 33.7 & 61.6 & 71.5 & 22.1 & \textbf{81.1} & 94.4 & 78.3 & \textbf{22.9} & \textbf{62.4} & 47.1 & 51.1 & 74.8 & 71.6 & 84.3 & 76.3 & 61.6 \\
\textbf{CurveTQ}$_{\mathrm{no\,LUT}}$ 3.29 & 51.4 & 35.0 & 61.3 & \textbf{72.3} & 23.2 & 79.8 & 94.5 & \textbf{79.2} & 22.0 & 62.3 & 51.5 & \textbf{51.7} & \textbf{77.0} & 74.0 & 82.6 & 81.6 & \textbf{62.5} \\
\textbf{CurveTQ}$_{\mathrm{fp16}}$ 3.47 & 51.6 & 33.5 & 62.6 & 70.5 & 22.6 & 80.9 & 94.4 & 78.8 & 22.4 & 62.8 & 48.6 & 52.0 & 75.6 & 71.0 & 85.2 & 73.7 & 61.6 \\
QTIP 3.00$^{+}$ & 50.7 & 34.7 & \textbf{62.4} & 71.0 & 21.9 & 80.0 & 93.7 & 77.9 & 22.4 & 61.6 & 53.1 & 47.4 & 74.6 & 72.2 & 84.3 & 72.4 & 61.3 \\
Proteus 3.00$^{+}$ & 51.1 & 35.0 & 59.4 & 70.1 & 20.9 & 79.6 & \textbf{94.6} & 78.0 & 22.6 & \textbf{62.4} & 49.4 & 47.4 & 74.1 & 71.2 & 80.0 & 76.3 & 60.8 \\
QTIP$_{+\mathrm{g64}}$ 3.25$^{+}$ & 48.2 & 33.6 & 59.4 & 71.1 & 21.7 & 80.3 & 94.2 & 77.3 & 22.6 & 61.1 & \textbf{54.9} & 51.0 & 73.9 & 70.1 & 84.3 & 72.4 & 61.0 \\
Proteus$_{+\mathrm{g64}}$ 3.25$^{+}$ & 51.4 & 33.8 & 58.9 & 70.6 & \textbf{24.3} & 79.2 & 94.4 & 77.9 & 22.1 & 61.9 & 51.6 & 50.9 & 75.2 & 70.5 & 82.6 & \textbf{82.9} & 61.8 \\
HIGGS 3.25 & 48.7 & 31.7 & 58.8 & 69.7 & 19.3 & 79.1 & 93.8 & 76.5 & 20.5 & 59.5 & 47.0 & 45.8 & 73.1 & 68.4 & 81.7 & 75.0 & 59.3 \\
NF 3.25 & 47.1 & 30.5 & 57.7 & 66.8 & 20.2 & 75.5 & 93.4 & 73.9 & 22.2 & 57.2 & 50.1 & 46.5 & 70.5 & 66.0 & 78.3 & 72.4 & 58.0 \\
GPTQ 3.25 & 42.7 & 28.1 & 54.6 & 66.5 & 18.1 & 74.7 & 91.7 & 70.7 & 21.8 & 56.9 & 42.6 & 44.8 & 71.0 & 66.5 & 77.4 & 67.1 & 56.0 \\
QTIP$_{+\mathrm{start}}$ 3.16$^{+}$ & \textbf{51.8} & \textbf{36.4} & 59.9 & 69.9 & 20.7 & 79.2 & 94.1 & 77.0 & 22.3 & 61.1 & 49.3 & 51.1 & 76.0 & \textbf{74.5} & 82.6 & 73.7 & 61.2 \\
Proteus$_{+\mathrm{start}}$ 3.20$^{+}$ & 49.7 & 32.1 & 60.2 & 70.6 & 22.5 & 79.9 & \textbf{94.6} & 78.3 & 22.5 & 61.8 & 50.4 & 51.4 & 75.5 & 72.1 & \textbf{85.2} & 81.6 & 61.8 \\
\midrule
\multicolumn{18}{@{}l}{\emph{4-bit tier}} \\
\textbf{CurveTQ} 4.29 & 52.8 & 36.6 & 62.7 & 72.3 & \textbf{25.4} & \textbf{81.5} & 94.6 & 79.2 & 22.4 & 63.6 & 51.6 & 53.0 & 75.6 & 72.6 & 86.1 & 80.3 & 63.1 \\
\textbf{CurveTQ}$_{\mathrm{no\,LUT}}$ 4.29 & 48.4 & 34.0 & 62.7 & 71.5 & 24.1 & 81.4 & 94.5 & 78.7 & 22.5 & 63.9 & \textbf{52.9} & 52.7 & 75.8 & 72.5 & 85.2 & 76.3 & 62.3 \\
\textbf{CurveTQ}$_{\mathrm{fp16}}$ 4.47 & 52.1 & 36.2 & 62.8 & 71.8 & 24.6 & 81.3 & 94.8 & 79.5 & 22.2 & 64.1 & 50.9 & 52.8 & 75.9 & 72.1 & 82.6 & 75.0 & 62.4 \\
QTIP 4.00$^{+}$ & 52.6 & 36.5 & 63.0 & 71.4 & 24.6 & 81.1 & 94.8 & 79.2 & 22.3 & 62.9 & 52.6 & 52.6 & 76.1 & 72.5 & 80.9 & 73.7 & 62.3 \\
Proteus 4.00$^{+}$ & \textbf{54.1} & \textbf{38.6} & 62.3 & 72.4 & 24.7 & 81.3 & \textbf{94.9} & 79.3 & 22.4 & 63.3 & 52.0 & 53.0 & 76.2 & 72.2 & 84.3 & 76.3 & 63.0 \\
QTIP$_{+\mathrm{g64}}$ 4.25$^{+}$ & 52.4 & 37.1 & \textbf{63.5} & 70.6 & 25.0 & 81.1 & 94.5 & 79.3 & \textbf{22.7} & 63.7 & 51.9 & 52.8 & 76.5 & 72.0 & \textbf{87.0} & \textbf{81.6} & \textbf{63.2} \\
Proteus$_{+\mathrm{g64}}$ 4.25$^{+}$ & 52.2 & 37.5 & 62.2 & 70.8 & 24.5 & 81.3 & 94.8 & 79.1 & 22.0 & 63.4 & 51.7 & \textbf{54.0} & \textbf{76.6} & \textbf{73.2} & \textbf{87.0} & 80.3 & \textbf{63.2} \\
HIGGS 4.25 & 51.4 & 34.7 & 62.1 & 71.2 & 24.1 & 81.2 & 94.8 & 79.4 & 22.3 & 64.0 & 51.9 & 52.8 & 74.8 & 71.9 & 84.3 & 79.0 & 62.5 \\
NF 4.25 & 51.7 & 35.1 & 62.9 & 70.7 & 24.9 & 80.8 & 94.8 & 79.0 & 21.9 & 63.1 & 49.1 & 52.1 & 76.0 & \textbf{73.2} & 86.1 & \textbf{81.6} & 62.7 \\
GPTQ 4.25 & 48.7 & 33.2 & 60.3 & 70.2 & 23.3 & 80.5 & 94.2 & 77.6 & 21.7 & 62.1 & 48.2 & 51.2 & 73.8 & 70.2 & 85.2 & 75.0 & 61.0 \\
QTIP$_{+\mathrm{start}}$ 4.13$^{+}$ & 52.3 & 35.5 & 62.1 & \textbf{72.9} & 24.6 & 81.3 & 94.5 & 79.2 & 22.4 & \textbf{64.7} & 50.7 & 51.9 & 74.9 & 72.0 & 84.3 & \textbf{81.6} & 62.8 \\
Proteus$_{+\mathrm{start}}$ 4.19$^{+}$ & 51.6 & 37.7 & 62.9 & 70.3 & 24.9 & 81.3 & 94.8 & \textbf{79.7} & 22.0 & 64.3 & 51.2 & 53.8 & 75.1 & 71.8 & 86.1 & 75.0 & 62.7 \\
\bottomrule
\end{tabular}

%% file: tab_full_q32_vl.tex
\begin{tabular}{@{}lrrrrrrrrrrrrrrrr@{\hspace{4pt}}r@{}}
\toprule
Method (bpw) & \rotatebox{60}{\,MMMU} & \rotatebox{60}{\,MMMU-Pro} & \rotatebox{60}{\,MMStar} & \rotatebox{60}{\,RealWorldQA} & \rotatebox{60}{\,SimpleVQA} & \rotatebox{60}{\,AI2D} & \rotatebox{60}{\,DocVQA} & \rotatebox{60}{\,InfoVQA} & \rotatebox{60}{\,MMLongB-Doc} & \rotatebox{60}{\,BLINK} & \rotatebox{60}{\,MuirBench} & \rotatebox{60}{\,VSI-Bench} & \rotatebox{60}{\,HRB-4K} & \rotatebox{60}{\,HRB-8K} & \rotatebox{60}{\,V$^*$-attr} & \rotatebox{60}{\,V$^*$-pos} & \rotatebox{60}{\,\textbf{Mean}} \\
\midrule
fp16 & 61.8 & 47.3 & 73.6 & 77.5 & 26.6 & 86.9 & 96.0 & 87.6 & 22.3 & 67.5 & 59.5 & 60.9 & 79.4 & 77.0 & 86.1 & 89.5 & \textbf{68.7} \\
\midrule
\multicolumn{18}{@{}l}{\emph{2-bit tier}} \\
\textbf{CurveTQ} 2.29 & 56.9 & 41.9 & 68.4 & 73.7 & 20.8 & 84.5 & 95.5 & 84.0 & \textbf{22.4} & 66.2 & 57.6 & 58.6 & 77.4 & 75.1 & 83.5 & 80.3 & 65.4 \\
QTIP 2.00$^{+}$ & \textbf{57.6} & \textbf{43.2} & 66.2 & 72.9 & 19.6 & 84.3 & 95.3 & 83.6 & \textbf{22.4} & 63.2 & 51.1 & 57.9 & 77.1 & 74.0 & 84.3 & \textbf{85.5} & 64.9 \\
Proteus 2.00$^{+}$ & 55.8 & 42.3 & 67.5 & 73.5 & 21.0 & 83.9 & \textbf{96.0} & 83.4 & \textbf{22.4} & 64.0 & 52.4 & 54.1 & \textbf{78.1} & 74.1 & \textbf{87.8} & 84.2 & 65.0 \\
QTIP$_{+\mathrm{g64}}$ 2.25$^{+}$ & \textbf{57.6} & 41.3 & 68.0 & 74.2 & 19.8 & \textbf{84.8} & 95.6 & \textbf{84.1} & 22.2 & 63.2 & 54.9 & 58.1 & 75.4 & 73.1 & 86.1 & 82.9 & 65.1 \\
Proteus$_{+\mathrm{g64}}$ 2.25$^{+}$ & 55.7 & 42.7 & 67.5 & \textbf{75.0} & 20.3 & 84.1 & 95.7 & \textbf{84.1} & 22.3 & 65.2 & 54.8 & \textbf{59.5} & 77.1 & \textbf{75.2} & 87.0 & 81.6 & 65.5 \\
QTIP$_{+\mathrm{start}}$ 2.19$^{+}$ & 56.4 & 40.5 & 67.5 & 73.1 & 20.2 & 84.5 & 95.7 & 83.5 & \textbf{22.4} & \textbf{66.8} & \textbf{57.8} & 58.9 & 77.0 & 73.9 & 87.0 & 84.2 & \textbf{65.6} \\
Proteus$_{+\mathrm{start}}$ 2.22$^{+}$ & 56.7 & 39.9 & \textbf{68.8} & 73.9 & \textbf{21.9} & \textbf{84.8} & 95.8 & 84.0 & \textbf{22.4} & 63.5 & 54.9 & 58.6 & 76.2 & 73.4 & 82.6 & 84.2 & 65.1 \\
\midrule
\multicolumn{18}{@{}l}{\emph{3-bit tier}} \\
\textbf{CurveTQ} 3.29 & 59.9 & 44.2 & 72.1 & \textbf{79.0} & 25.8 & 86.7 & 96.2 & 87.2 & 22.3 & \textbf{68.2} & \textbf{61.0} & \textbf{61.2} & 78.5 & 76.8 & \textbf{88.7} & 86.8 & \textbf{68.4} \\
QTIP 3.00$^{+}$ & 59.4 & 44.6 & \textbf{73.0} & 76.5 & 24.9 & \textbf{87.0} & \textbf{96.3} & 86.8 & 22.3 & 66.3 & 60.3 & 60.1 & 78.4 & 76.1 & 86.1 & 86.8 & 67.8 \\
Proteus 3.00$^{+}$ & 60.3 & 45.8 & 72.3 & 76.5 & 25.0 & 86.1 & 95.8 & 86.4 & 22.3 & 65.9 & 57.1 & 60.5 & 78.8 & 77.5 & 87.0 & 85.5 & 67.7 \\
QTIP$_{+\mathrm{g64}}$ 3.25$^{+}$ & 59.2 & 46.7 & 72.3 & 75.6 & 24.7 & 86.4 & 96.2 & 86.6 & 22.5 & 66.9 & 57.5 & 60.9 & 79.2 & 76.9 & 86.1 & 86.8 & 67.8 \\
Proteus$_{+\mathrm{g64}}$ 3.25$^{+}$ & \textbf{60.9} & \textbf{46.8} & 71.4 & 78.7 & 26.1 & 86.3 & 95.9 & 86.3 & \textbf{22.6} & 68.1 & 56.6 & 59.6 & 77.0 & 76.5 & 86.1 & \textbf{88.2} & 67.9 \\
QTIP$_{+\mathrm{start}}$ 3.16$^{+}$ & 59.4 & 44.4 & 72.4 & 76.5 & 25.4 & 86.1 & 96.1 & 87.0 & 22.4 & 66.5 & 58.8 & 60.3 & 78.9 & 77.4 & 87.0 & 86.8 & 67.8 \\
Proteus$_{+\mathrm{start}}$ 3.20$^{+}$ & 60.6 & 44.9 & 71.5 & 77.7 & \textbf{26.4} & 86.2 & 96.0 & \textbf{87.3} & 22.4 & 67.6 & 58.2 & 59.5 & \textbf{79.6} & \textbf{77.8} & 86.1 & 86.8 & 68.0 \\
\bottomrule
\end{tabular}

%% file: tab_full_q32_text.tex
\begin{tabular}{@{}lrrrrrrrrrrrrrrrrrr@{\hspace{4pt}}r@{}}
\toprule
Method (bpw) & \rotatebox{60}{\,MMLU} & \rotatebox{60}{\,MMLU-Pro} & \rotatebox{60}{\,TruthfulQA} & \rotatebox{60}{\,GPQA} & \rotatebox{60}{\,LB-Lang} & \rotatebox{60}{\,GSM8K} & \rotatebox{60}{\,ARC-c} & \rotatebox{60}{\,ARC-e} & \rotatebox{60}{\,HellaSwag} & \rotatebox{60}{\,WinoGr.} & \rotatebox{60}{\,PIQA} & \rotatebox{60}{\,BoolQ} & \rotatebox{60}{\,LB-Reason} & \rotatebox{60}{\,LB-Math} & \rotatebox{60}{\,LB-DA} & \rotatebox{60}{\,IFEval} & \rotatebox{60}{\,MMLU-ProX} & \rotatebox{60}{\,INCLUDE} & \rotatebox{60}{\,\textbf{Mean}} \\
\midrule
fp16 & 83.7 & 74.0 & 60.6 & 40.9 & 49.8 & 88.5 & 74.7 & 84.8 & 84.0 & 78.2 & 82.0 & 90.7 & 72.0 & 52.7 & 64.2 & 66.4 & 74.1 & 66.7 & \textbf{71.5} \\
\midrule
\multicolumn{20}{@{}l}{\emph{2-bit tier}} \\
\textbf{CurveTQ} 2.29 & 80.3 & 66.9 & 59.1 & 37.9 & \textbf{47.6} & 86.3 & 70.8 & 83.8 & \textbf{80.5} & \textbf{77.3} & 81.0 & 90.4 & 58.8 & 40.4 & \textbf{58.9} & 56.4 & 66.9 & 63.7 & \textbf{67.1} \\
QTIP 2.00$^{+}$ & 79.5 & 64.2 & 56.2 & 41.9 & 40.9 & 88.7 & 69.1 & \textbf{85.8} & 79.6 & 77.1 & \textbf{81.6} & 90.0 & 57.6 & 37.9 & 56.3 & 52.1 & 64.1 & 62.8 & 65.9 \\
Proteus 2.00$^{+}$ & 79.8 & 64.4 & 58.4 & 39.9 & 36.9 & 87.7 & 69.6 & 83.2 & 79.4 & 75.8 & 80.8 & 90.1 & 57.0 & 36.2 & 52.0 & 52.5 & 64.3 & 62.7 & 65.1 \\
QTIP$_{+\mathrm{g64}}$ 2.25$^{+}$ & 79.3 & 65.8 & 58.5 & 38.4 & 40.8 & 77.4 & 69.9 & 83.7 & 79.7 & 76.9 & 80.5 & 89.8 & 54.7 & 41.1 & 48.8 & 54.9 & 65.9 & 63.0 & 64.9 \\
Proteus$_{+\mathrm{g64}}$ 2.25$^{+}$ & 80.2 & 64.1 & 58.7 & 41.4 & 40.5 & \textbf{89.7} & 70.1 & 85.4 & 80.2 & 76.1 & 80.6 & \textbf{90.6} & 56.7 & 38.2 & 56.9 & 53.0 & 64.3 & 63.1 & 66.1 \\
QTIP$_{+\mathrm{start}}$ 2.19$^{+}$ & 80.2 & 66.4 & \textbf{61.3} & \textbf{43.9} & 45.0 & 82.0 & 69.0 & 82.6 & 80.3 & 77.0 & 80.8 & 90.0 & \textbf{61.6} & 37.5 & 52.3 & 56.2 & 66.6 & 63.7 & 66.5 \\
Proteus$_{+\mathrm{start}}$ 2.22$^{+}$ & \textbf{80.7} & \textbf{67.2} & 59.5 & 38.4 & 42.1 & 86.8 & \textbf{71.0} & 85.3 & 80.4 & 76.3 & 81.3 & 90.2 & 59.3 & \textbf{42.9} & 53.0 & \textbf{58.4} & \textbf{67.0} & \textbf{64.7} & 66.9 \\
\midrule
\multicolumn{20}{@{}l}{\emph{3-bit tier}} \\
\textbf{CurveTQ} 3.29 & \textbf{83.1} & 72.8 & 60.8 & 41.9 & 50.8 & 87.2 & 73.4 & \textbf{85.9} & \textbf{83.5} & 77.2 & 81.7 & \textbf{90.9} & 68.8 & 51.4 & 63.4 & 63.6 & \textbf{73.2} & 66.2 & \textbf{70.9} \\
QTIP 3.00$^{+}$ & 83.0 & 72.8 & 60.6 & \textbf{42.9} & 48.9 & 88.9 & 72.5 & 81.7 & 83.3 & \textbf{78.8} & 81.9 & 90.8 & 71.1 & 48.3 & 61.6 & 62.5 & 72.7 & 66.6 & 70.5 \\
Proteus 3.00$^{+}$ & 82.3 & 72.4 & 60.6 & 39.9 & 48.6 & \textbf{89.9} & 71.7 & 84.6 & 83.0 & 78.5 & \textbf{82.2} & 90.7 & 70.7 & \textbf{53.7} & 60.3 & \textbf{65.6} & 72.1 & 67.0 & 70.8 \\
QTIP$_{+\mathrm{g64}}$ 3.25$^{+}$ & 82.7 & 72.6 & \textbf{61.8} & 41.9 & 49.7 & 88.6 & 73.4 & 84.8 & \textbf{83.5} & 78.1 & 81.3 & 90.5 & \textbf{71.3} & 51.9 & 61.6 & 62.7 & 72.6 & \textbf{67.2} & \textbf{70.9} \\
Proteus$_{+\mathrm{g64}}$ 3.25$^{+}$ & 82.8 & \textbf{73.0} & 60.5 & 39.4 & 50.1 & 89.2 & 73.8 & 85.7 & 83.3 & 78.4 & 82.0 & 90.7 & 70.2 & 49.4 & 62.1 & 63.6 & 73.0 & 66.6 & 70.8 \\
QTIP$_{+\mathrm{start}}$ 3.16$^{+}$ & 82.8 & 72.4 & 59.1 & \textbf{42.9} & \textbf{51.1} & 88.2 & \textbf{74.1} & 82.6 & 83.1 & 77.5 & 81.0 & 90.6 & 67.3 & 52.8 & 60.5 & 64.5 & 72.6 & 66.9 & 70.6 \\
Proteus$_{+\mathrm{start}}$ 3.20$^{+}$ & 83.0 & 72.9 & 60.3 & 38.9 & 48.9 & 86.9 & 73.2 & 83.7 & 83.3 & 78.5 & 81.5 & 90.8 & 70.8 & 50.2 & \textbf{64.8} & 63.6 & 72.7 & 65.8 & 70.5 \\
\bottomrule
\end{tabular}

%% file: tab_full_q35_vl.tex
\begin{tabular}{@{}lrrrrrrrrrrrrrr@{\hspace{4pt}}r@{}}
\toprule
Method (bpw) & \rotatebox{60}{\,MMMU-Pro} & \rotatebox{60}{\,MMStar} & \rotatebox{60}{\,RealWorldQA} & \rotatebox{60}{\,SimpleVQA} & \rotatebox{60}{\,AI2D} & \rotatebox{60}{\,DocVQA} & \rotatebox{60}{\,InfoVQA} & \rotatebox{60}{\,MMLongB-Doc} & \rotatebox{60}{\,BLINK} & \rotatebox{60}{\,VSI-Bench} & \rotatebox{60}{\,HRB-4K} & \rotatebox{60}{\,HRB-8K} & \rotatebox{60}{\,V$^*$-attr} & \rotatebox{60}{\,V$^*$-pos} & \rotatebox{60}{\,\textbf{Mean}} \\
\midrule
fp16 16 & 45.3 & 71.3 & 75.7 & 31.6 & 88.0 & 95.8 & 86.7 & 22.3 & 69.6 & 60.8 & 79.5 & 75.9 & 86.1 & 84.2 & \textbf{69.5} \\
\midrule
\textbf{CurveTQ} 2.29 & \textbf{35.3} & \textbf{69.6} & \textbf{76.1} & \textbf{28.9} & \textbf{85.9} & \textbf{94.8} & \textbf{82.1} & 21.8 & \textbf{66.2} & \textbf{57.4} & \textbf{77.6} & \textbf{74.5} & 84.3 & \textbf{82.9} & \textbf{67.0} \\
QTIP 2.00$^{+}$ & 32.6 & 66.5 & 75.0 & 26.5 & 81.5 & 93.7 & 80.0 & 22.1 & 60.3 & 55.9 & 73.0 & 70.2 & \textbf{85.2} & 79.0 & \textbf{64.4} \\
Proteus 2.00$^{+}$ & 33.1 & 66.9 & 74.4 & 27.3 & 85.8 & 94.2 & 79.3 & \textbf{22.7} & 60.1 & 56.2 & 74.4 & 70.5 & 79.1 & 80.3 & \textbf{64.6} \\
QTIP$_{+\mathrm{g64}}$ 2.25$^{+}$ & 34.8 & 63.9 & 75.3 & 26.0 & 69.3 & 94.5 & 80.5 & 21.5 & 42.2 & 57.0 & 76.9 & 72.4 & 81.7 & \textbf{82.9} & \textbf{62.8} \\
Proteus$_{+\mathrm{g64}}$ 2.25$^{+}$ & 35.9 & 66.0 & 74.0 & 29.1 & 84.1 & 94.6 & 81.8 & 21.7 & 55.2 & 52.1 & 76.1 & 73.8 & 84.3 & 64.5 & 63.8 \\
QTIP$_{+\mathrm{start}}$ 2.19$^{+}$ & 38.8 & 71.9 & 73.9 & 30.1 & 85.5 & 93.9 & 80.5 & 22.3 & 65.6 & 57.8 & 75.5 & 73.8 & 86.1 & 77.6 & 66.7 \\
Proteus$_{+\mathrm{start}}$ 2.22$^{+}$ & 31.0 & 69.1 & 75.0 & 28.5 & 71.1 & 93.6 & 81.3 & 22.5 & 63.2 & 57.0 & 75.1 & 71.2 & 86.1 & 73.7 & 64.2 \\
\bottomrule
\end{tabular}

%% file: tab_full_q35_text.tex
\begin{tabular}{@{}lrrrrrrrrrrrrrrrr@{\hspace{4pt}}r@{}}
\toprule
Method (bpw) & \rotatebox{60}{\,MMLU} & \rotatebox{60}{\,MMLU-Pro} & \rotatebox{60}{\,GPQA} & \rotatebox{60}{\,LB-Lang} & \rotatebox{60}{\,GSM8K} & \rotatebox{60}{\,ARC-c} & \rotatebox{60}{\,HellaSwag} & \rotatebox{60}{\,WinoGr.} & \rotatebox{60}{\,PIQA} & \rotatebox{60}{\,BoolQ} & \rotatebox{60}{\,LB-Reason} & \rotatebox{60}{\,LB-Math} & \rotatebox{60}{\,LB-DA} & \rotatebox{60}{\,IFEval} & \rotatebox{60}{\,MMLU-ProX} & \rotatebox{60}{\,INCLUDE} & \rotatebox{60}{\,\textbf{Mean}} \\
\midrule
fp16 16 & 82.4 & 78.1 & 46.0 & 52.5 & 88.8 & 72.0 & 82.8 & 76.2 & 82.8 & 85.0 & 73.3 & 54.8 & 64.3 & 82.4 & 77.7 & 68.0 & \textbf{72.9} \\
\midrule
\textbf{CurveTQ} 2.29 & \textbf{81.9} & \textbf{73.5} & \textbf{45.0} & \textbf{41.6} & 88.9 & \textbf{70.4} & \textbf{81.2} & \textbf{75.2} & \textbf{82.4} & 88.0 & \textbf{65.6} & 42.3 & \textbf{53.8} & 80.0 & \textbf{72.9} & \textbf{64.3} & \textbf{69.2} \\
QTIP 2.00$^{+}$ & 80.8 & 72.4 & 39.4 & 36.0 & 87.4 & 69.8 & 79.8 & 73.9 & 81.2 & 87.5 & 57.8 & 39.9 & 49.8 & 76.7 & 71.8 & 62.5 & \textbf{66.7} \\
Proteus 2.00$^{+}$ & 81.0 & 72.9 & 40.4 & 37.1 & 87.5 & 69.5 & 80.2 & 74.5 & 81.9 & 88.6 & 55.6 & 40.9 & 53.3 & 77.1 & 72.8 & 63.2 & \textbf{67.3} \\
QTIP$_{+\mathrm{g64}}$ 2.25$^{+}$ & 81.0 & 72.5 & 41.4 & 35.9 & \textbf{89.5} & 70.0 & 80.1 & 73.3 & 81.0 & \textbf{89.9} & 57.3 & \textbf{43.1} & 50.4 & 79.7 & 71.3 & 62.6 & \textbf{67.4} \\
Proteus$_{+\mathrm{g64}}$ 2.25$^{+}$ & 80.7 & 71.1 & 41.9 & 33.5 & 87.8 & 69.1 & 80.7 & 75.6 & 81.1 & 87.2 & 49.7 & 36.5 & 49.5 & \textbf{80.4} & 70.9 & 63.6 & 66.2 \\
QTIP$_{+\mathrm{start}}$ 2.19$^{+}$ & 81.6 & 73.8 & 42.9 & 36.4 & 87.2 & 71.2 & 80.7 & 76.1 & 81.5 & 88.8 & 62.0 & 40.0 & 55.2 & 79.5 & 73.9 & 63.6 & 68.4 \\
Proteus$_{+\mathrm{start}}$ 2.22$^{+}$ & 81.2 & 73.8 & 47.5 & 40.7 & 88.4 & 68.5 & 81.3 & 74.6 & 81.2 & 83.1 & 65.6 & 43.8 & 54.9 & 78.2 & 73.5 & 64.1 & 68.8 \\
\bottomrule
\end{tabular}

%% file: tab_abl_full_ci.tex
\begin{tabular}{@{}l r r r r r l@{}}
\toprule
& & & \multicolumn{2}{c}{WikiText-2 ppl} & \multicolumn{2}{c}{Downstream (34 tasks)} \\
\cmidrule(lr){4-5}\cmidrule(lr){6-7}
Configuration & bpw & $\Delta$bpw & ppl & $\Delta$ & mean & $\Delta$ [95\% CI] \\
\midrule
\multicolumn{7}{@{}l}{\emph{Building up the codec (each row adds one component to the row above)}} \\
128-entry table, EM fit (start) & 2.469 &  & 13.32 &  & 50.4 &  \\
+ closed-form quantile fit & 2.469 & $+0.000$ & 13.05 & $-0.28$ & 50.8 & $+0.3$ [$-0.6$, $+1.3$] \\
+ int8 second-level scales & 2.345 & $-0.124$ & 13.46 & $+0.41$ & 50.9 & $+0.1$ [$-0.6$, $+1.0$] \\
+ edge-indexed emission & 2.345 & $+0.000$ & 12.19 & $-1.27$ & 51.9 & $+1.0$ [$-0.6$, $+2.4$] \\
+ int4 scales (flat grid) & 2.282 & $-0.062$ & 12.77 & $+0.58$ & 51.0 & $-0.8$ [$-1.9$, $+0.0$] \\
\rowcolor{oursbg}\textbf{+ factored int4 scales, per-module table (= \textsc{CurveTQ})} & 2.289 & $+0.006$ & 11.50 & $-1.26$ & 51.8 & $+0.8$ [$-0.0$, $+1.7$] \\
\midrule
\multicolumn{7}{@{}l}{\emph{Removing one component from \textsc{CurveTQ}}} \\
$-$ factored scales (fp16 scales) & 2.469 & $+0.180$ & 11.86 & $+0.35$ & 52.2 & $+0.4$ [$-0.3$, $+1.1$] \\
$-$ emission table (computed code) & 2.289 & $-0.000$ & 13.63 & $+2.13$ & 50.7 & $-1.2$ [$-2.2$, $-0.2$] \\
\midrule
\multicolumn{7}{@{}l}{\emph{Ablating the search (from the quantile-fit row above)}} \\
$-$ error feedback & 2.469 & $+0.000$ & 13.38 & $+0.34$ & 48.2 & $-2.5$ [$-4.3$, $-1.4$] \\
Euclidean branch metric instead of $1/d_i^2$ & 2.469 & $+0.000$ & 12.81 & $-0.24$ & 47.4 & $-3.3$ [$-4.5$, $-2.3$] \\
\midrule
\multicolumn{7}{@{}l}{\emph{Block termination (from \textsc{CurveTQ}; Section~\ref{sec:search})}} \\
chained trellis, one start state per row & 2.073 & $-0.216$ & 13.59 & $+2.09$ & 46.4 & $-5.5$ [$-7.0$, $-4.2$] \\
tail-biting block, no start state & 2.069 & $-0.220$ & 15.60 & $+4.10$ & 46.8 & $-5.1$ [$-6.9$, $-3.5$] \\
\quad same, two rows per sequence ($T{=}128$) & 2.069 & $-0.220$ & 14.62 & $+3.12$ & 46.8 & $-5.0$ [$-6.5$, $-3.5$] \\
\quad same, four rows per sequence ($T{=}256$) & 2.069 & $-0.220$ & 13.38 & $+1.88$ & 47.7 & $-4.2$ [$-6.4$, $-2.5$] \\
\bottomrule
\end{tabular}

%% file: tab_tput_full.tex
\begin{tabular}{@{}l c r rr>{\columncolor{oursbg}}r rr rr>{\columncolor{oursbg}}r rr>{\columncolor{oursbg}}r@{}}
\toprule
& & \multicolumn{6}{c}{batch 1} & \multicolumn{3}{c}{batch 64} & \multicolumn{3}{c}{batch 128} \\
\cmidrule(lr){3-8}\cmidrule(lr){9-11}\cmidrule(lr){12-14}
Model & bits & fp16 & QTIP & Prot. & \textsc{CTQ} & $+$rot. & $-$start & QTIP & Prot. & \textsc{CTQ} & QTIP & Prot. & \textsc{CTQ} \\
\midrule
L2-7B & 2 & 52.5 & 127.0 & 105.9 & \textbf{139.0} & 130.0 & 141.3 & 1011 & 1038 & \textbf{1082} & 1715 & 1830 & \textbf{1927} \\
 & 3 & & 119.0 & 103.5 & \textbf{125.7} & 117.6 & 126.8 & 999 & 1025 & \textbf{1071} & 1702 & 1805 & \textbf{1916} \\
 & 4 & & 109.0 & 102.6 & \textbf{114.9} & 109.9 & 117.3 & 985 & 1011 & \textbf{1063} & 1690 & 1782 & \textbf{1907} \\
\midrule
Q3-4B & 2 & 74.3 & 153.6 & 117.9 & \textbf{160.8} & 143.5 & 165.3 & 1007 & 1012 & \textbf{1021} & 2186 & 2227 & \textbf{2268} \\
 & 3 & & 144.2 & 115.4 & \textbf{151.7} & 138.9 & 155.3 & 1002 & 1003 & \textbf{1016} & 2178 & 2212 & \textbf{2246} \\
 & 4 & & 137.2 & 113.2 & \textbf{146.2} & 133.0 & 148.3 & 997 & 1000 & \textbf{1012} & 2163 & 2194 & \textbf{2236} \\
\bottomrule
\end{tabular}

%% file: tab_rate.tex
\begin{tabular}{@{}llrrrr@{}}
\toprule
Model & $k$ & Coded weights & Code $+$ scales $+$ table & $+$ 1 state/row & \textbf{$+$ 1 state/block (reported)} \\
\midrule
Qwen3-4B & 2 & 3.63\,B & 2.0685 & 2.0734 & \textbf{2.2872} \\
Qwen3-4B & 3 & 3.63\,B & 3.0685 & 3.0734 & \textbf{3.2872} \\
Qwen3-4B & 4 & 3.63\,B & 4.0685 & 4.0734 & \textbf{4.2872} \\
Qwen3-4B, no table & 2 & 3.63\,B & 2.0685 & 2.0734 & \textbf{2.2872} \\
Qwen3-4B, fp16 scales & 2 & 3.63\,B & 2.2500 & 2.2549 & \textbf{2.4688} \\
Llama-2-7B & 2 & 6.48\,B & 2.0669 & 2.0703 & \textbf{2.2857} \\
Llama-3-8B & 2 & 6.98\,B & 2.0667 & 2.0699 & \textbf{2.2855} \\
Llama-3-8B & 3 & 6.98\,B & 3.0667 & 3.0699 & \textbf{3.2855} \\
Llama-3-8B & 4 & 6.98\,B & 4.0667 & 4.0699 & \textbf{4.2855} \\
Llama-3.1-8B & 2 & 6.98\,B & 2.0667 & 2.0699 & \textbf{2.2855} \\
Qwen3-32B & 2 & 31.2\,B & 2.0659 & 2.0682 & \textbf{2.2846} \\
Qwen3-32B & 3 & 31.2\,B & 3.0659 & 3.0682 & \textbf{3.2846} \\
Qwen3.6-35B-A3B & 2 & 30.3\,B & 2.0720 & -- & \textbf{2.2908} \\
\bottomrule
\end{tabular}

%% file: tab_stream.tex
\begin{tabular}{@{}lrr@{}}
\toprule
Component (Qwen3-4B, layer 0, 7 tensors, 100.9\,M weights) & Bytes & bpw \\
\midrule
Trellis payload ($k=2$) & 25{,}231{,}360 & 2.000000 \\
Start state, one 14-bit state per (row, 64-weight block) & 2{,}759{,}680 & 0.218750 \\
Scale field: int4 residual & 788{,}480 & 0.062500 \\
Scale field: fp16 block absmax (1 per 256 residuals) & 12{,}320 & 0.000977 \\
Scale field: fp16 row factor & 61{,}440 & 0.004870 \\
Scale field: fp16 column factor & 832 & 0.000066 \\
Emission table (128 int8 knots $+$ fp32 step, per tensor) & 924 & 0.000073 \\
\midrule
Total & 28{,}855{,}036 & 2.287236 \\
\midrule
Decodes bit-exactly from the stream above & \multicolumn{2}{r}{yes (0 of 100.9\,M differ)} \\
Decodes from one start state per row instead & \multicolumn{2}{r}{no (10.17\% of weights differ)} \\
\bottomrule
\end{tabular}

%% file: tab_audit_full.tex
\begin{tabular}{@{}lrrrr@{}}
\toprule
Block termination & Packed bpw & Start bpw & Rel.\ W-err & Wiki2 ppl \\
\midrule
Block-local, stored start state (\textsc{CurveTQ}) & 2.2872 & 0.2188 & 0.0801 & 11.54 \\
Chained, one stored state per row & 2.0734 & 0.0049 & 0.1148 & 13.59 \\
Per-block tail-biting, no start bits & 2.0685 & 0 & 0.1093 & 15.60 \\
Tail-biting, two rows ($T{=}128$), no start bits & 2.0685 & 0 & 0.1026 & 14.62 \\
Tail-biting, four rows ($T{=}256$), no start bits & 2.0685 & 0 & 0.1005 & 13.38 \\
\bottomrule
\end{tabular}